%% file: pNMF-main-arxiv.tex
\documentclass[hidelinks,onefignum,onetabnum]{siamart251216}

\input{pNMF-shared-simax}

\begin{document}

\maketitle

\begin{abstract}
Nonnegative matrix factorization (NMF) and its graph-based extensions typically produce a single embedding, leaving the evolution of representations across connectivity resolutions unexplored. We propose persistent nonnegative matrix factorization (pNMF), which learns an ordered embedding sequence at intrinsic connectivity transitions. Using zero-dimensional persistent homology, pNMF constructs a canonical minimal sufficient scale set that represents each distinct connected-component count once. Graph regularization captures scale-specific geometry, while cross-scale coupling promotes continuity between adjacent embeddings. We characterize cross-scale graph and embedding variation and, for fixed bases, exploit the uniformly coercive block tridiagonal embedding operator to establish perturbation bounds and conditional exponential localization of contributions from scales where graph support changes. A two-block alternating algorithm updates the basis matrices in parallel and the embedding sequence jointly. For exact block updates, the objective values converge and every accumulation point satisfies the Karush--Kuhn--Tucker (KKT) conditions. Experiments on three synthetic geometries illustrate embedding evolution, with further tests examining coupling and conditional stability. Computational tests quantify sparse-graph storage savings and the cost of retaining the full sequence. Across five single-cell RNA sequencing datasets, the prespecified first-scale embedding yields competitive clustering performance. On GSE75748time, the best-matching partition of sampling times changes across scales at a fixed cluster count.
\end{abstract}



\section{Introduction}
Nonnegative matrix factorization (NMF) approximates a nonnegative data matrix by the product of two low-rank nonnegative factors, yielding additive representations that are often compact and interpretable~\cite{lee1999learning,gillis2020nonnegative}.
Because standard NMF does not explicitly preserve sample relations, graph-regularized NMF (GNMF) introduces a Laplacian penalty that encourages samples connected in a prescribed affinity graph to remain close in the embedding~\cite{cai2010graph,chung1997spectral}.

The representation learned by GNMF nevertheless depends on the graph resolution.
Fine-scale graphs emphasize local neighborhoods, whereas coarser graphs connect previously separated groups.
Multiple-graph NMF methods address this dependence by combining graph regularizers~\cite{wang2013multiple}.
A related approach uses persistent Laplacians to describe topological and spectral information across filtrations~\cite{wang2020persistent,memoli2022persistent}.
In NMF, topological NMF (TNMF) and its $k$-TNMF variant construct a single Laplacian regularizer from graphs at multiple distance and nearest-neighbor scales, respectively~\cite{hozumi2024analyzing}.
These graph-aggregation NMF models incorporate multiscale information into one embedding, but do not explicitly learn and couple an ordered sequence of scale-specific representations.

Learning such a sequence raises two related questions: which scales should index the embeddings, and how should their variation be controlled?
Prespecified distance grids or neighborhood sizes can repeat component counts or omit intermediate connectivity levels, while independently fitted embeddings can vary abruptly.
This motivates selecting scales from the connectivity transitions of the data and jointly learning their embeddings, with a balance between continuity across adjacent scales and differences across resolutions.

Persistent homology~\cite{zomorodian2005computing,edelsbrunner2008persistent} provides a natural basis for this scale selection.
Along a Vietoris--Rips filtration~\cite{zomorodian2010fast}, $\mathcal{H}_0$ tracks connected components, and positive finite death values mark their mergers.
Together with scale zero, the distinct positive finite death values form the canonical minimal sufficient scale set, which selects the earliest scale attaining each distinct connected-component count.
The model uses the positive elements of this set, which can be computed from the distinct positive edge weights of a minimum spanning tree~\cite{edelsbrunner2010computational}.

We propose persistent nonnegative matrix factorization (pNMF).
The model fits a graph-regularized factorization at each selected scale, couples adjacent embeddings through a quadratic smoothness penalty, and regularizes both factor sequences.
This construction retains an embedding at each selected connectivity resolution and allows us to analyze how the representations vary across scales and respond to input perturbations.
Our focus is the coupled factorization model, its operators, and embedding evolution, using the classical $\mathcal{H}_0$ scale construction as its index.

The main contributions of this work are summarized as follows.

\textbullet\enspace We characterize the canonical minimal sufficient scale set as zero together with the distinct positive finite $\mathcal{H}_0$ death values and use the positive scales to index the graph Laplacians that regularize the embeddings in pNMF.
This set is uniquely determined by the first occurrence of each distinct connected component count along the filtration.
For the associated Laplacian family, we establish piecewise Lipschitz continuity between changes in graph support.
Along the selected positive scales, we derive adjacent-scale Laplacian and spectral variation bounds and spectral monotonicity.

\textbullet\enspace We show that coupling consecutive embeddings yields a global symmetric block-tridiagonal operator that is uniformly coercive and has exponentially decaying inverse blocks.
By comparing the Karush--Kuhn--Tucker (KKT) conditions at consecutive scales, we bound the corresponding embedding increments in terms of changes in the basis matrices and graph Laplacians.
With the bases fixed, uniform coercivity guarantees a unique embedding minimizer and allows us to bound its sensitivity to input perturbations.
Under the localization assumptions, inverse-block decay gives an exponentially decaying bound on the embedding response attributable to a local edge-set change as the distance between scales increases.

\textbullet\enspace We exploit the block-tridiagonal structure of the embedding operator in a two-block alternating algorithm that uses active-set projected Newton steps with preconditioned conjugate gradients (PCG) for both blocks.
For the embedding update, PCG uses Hessian--vector products involving only each scale and its neighbors to compute a joint search direction for the entire embedding sequence without assembling the Hessian.
For the basis update, the subproblems separate across scales and feature rows and can be solved in parallel, with all rows at a given scale sharing one Gram matrix.
We establish coercivity of the objective and, under exact block minimization, prove sufficient descent, convergence of the objective values, and KKT stationarity of every accumulation point.

\textbullet\enspace Experiments on three synthetic geometries illustrate how the embeddings evolve across scales and show that coupling reduces the average geometric change between adjacent embeddings.
On the noncoplanar circles, we also examine conditional stability and measure how runtime and graph storage grow with sample size.
Across five scRNA-seq datasets, the prespecified first-scale embedding yields competitive clustering results, while the effect of coupling on clustering varies with the dataset and scale.
On GSE75748time, the best-matching partition of sampling times changes across scales even when the number of clusters is held fixed.

This paper is organized as follows.
\Cref{sec:preliminaries} reviews the required NMF, GNMF, and persistent-homology concepts, and \Cref{sec:multi_scale_graph} develops the scale construction, graph sequence, and pNMF model.
\Cref{sec:block} derives the conditional block operator used to analyze embedding evolution and stability in \Cref{sec:evolution_stability} and to update the embeddings jointly in \Cref{sec:optimization}, where we also analyze convergence under exact block updates.
The synthetic and scRNA-seq experiments appear in \Cref{sec:numerical_experiments,sec:scrna_clustering}, followed by the conclusion in \Cref{sec:conclusion}.

The code for pNMF and representative experiments is available at \url{https://github.com/LiminLi-xjtu/pNMF}.

\section{Preliminaries and Related Work}
\label{sec:preliminaries}

This section reviews the NMF formulations underlying pNMF and the persistent homology concepts used to construct its scale sequence.

\subsection{NMF and Graph-Regularized Extensions}
Let $X=[x_1,x_2,\ldots,x_n] \in \mathbb{R}_{+}^{p \times n}$ be a nonnegative data matrix with samples as columns.
For a target rank $d$, nonnegative matrix factorization (NMF)~\cite{lee1999learning,lee2000algorithms,wang2012nonnegative,gillis2020nonnegative} seeks a basis matrix $W \in \mathbb{R}_{+}^{p \times d}$ and a coefficient matrix $H \in \mathbb{R}_{+}^{d \times n}$ by solving
\[
\min_{W,H \geq 0} \ \|X-WH\|_F^2,
\]
where $\|\! \cdot \!\|_F$ denotes the Frobenius norm.
Since the columns of $H$ provide $d$-dimensional nonnegative representations of the samples, we refer to $H$ as a \textit{nonnegative low-rank embedding} of $X$.

To incorporate sample relationships, graph-regularized NMF (GNMF)~\cite{cai2010graph} augments the reconstruction objective with a Laplacian penalty,
\[
\min_{W,H \geq 0} \ \|X-WH\|_F^2 + \lambda \operatorname{Tr}(HLH^\top),
\]
where $\operatorname{Tr}(\cdot)$ denotes the matrix trace and $(\cdot)^\top$ denotes the matrix transpose.
Here, $L \in \mathbb{R}^{n \times n}$ is the unnormalized Laplacian~\cite{chung1997spectral} of an undirected affinity graph with nonnegative edge weights, and $\lambda>0$ controls the geometric regularization.
The penalty $\operatorname{Tr}(HLH^\top)$ encourages samples joined by an edge to have similar representations, thereby incorporating the graph's local geometry into the embedding.

Several extensions incorporate multiple graph constructions or resolutions into one factorization.
For example, MultiGrNMF jointly learns the NMF factors and the weights in a linear combination of candidate graph Laplacians~\cite{wang2013multiple}.
TNMF and $k$-TNMF likewise combine the graphs from their respective filtrations into one Laplacian regularizer~\cite{hozumi2024analyzing}.
Although these NMF formulations incorporate multiple graphs or scales, they produce a single embedding.
pNMF instead retains the ordered Laplacian sequence and couples embeddings at adjacent selected positive scales to study representation evolution across connectivity resolutions.

\subsection{Persistent Homology}
Persistent homology~\cite{zomorodian2005computing,edelsbrunner2008persistent,edelsbrunner2010computational} tracks topological features across scales.
As the scale parameter $\varepsilon$ increases, the data are represented by a nested family of finite simplicial complexes, called a filtration.
We take homology with coefficients in a fixed field.
At each scale, homology groups describe the topological features of the complex: degree zero $\mathcal{H}_0$ captures connected components, while degree one $\mathcal{H}_1$ captures loops.
Persistent homology follows these features through the filtration and summarizes their evolution using a barcode, where each interval records the scales over which a homology class persists~\cite{crawley2015decomposition}.

For scale selection, pNMF uses the degree-zero barcode.
In a Vietoris--Rips filtration~\cite{zomorodian2010fast}, increasing $\varepsilon$ adds edges between data points, which can cause previously separate components to merge.
Each positive finite death value marks a scale at which one or more such mergers occur and thus identifies a change in connectivity.

\section{Multi-Scale Persistent NMF Modeling}
\label{sec:multi_scale_graph}

We construct the persistence-selected scale set and the corresponding ordered Laplacian sequence, characterize the sequence's scale-wise and spectral evolution, and formulate the pNMF model using these graph regularizers.

\subsection{Persistence-Guided Scale Set}
Assume $n \geq 2$, set $[n] := \{1,2,\ldots,n\}$ and let $r_{ij} := \|x_i-x_j\|_2$ be the Euclidean distance between samples $i$ and $j$, where $\|\cdot\|_2$ denotes the Euclidean norm for vectors and the spectral norm for matrices.
At scale $\varepsilon \geq 0$, the Vietoris--Rips (VR) complex~\cite{zomorodian2010fast} is
\[
\mathcal{X}_{\varepsilon} := \{\sigma \subseteq [n] : r_{ij} \leq \varepsilon \text{ for all distinct } i,j \in \sigma\}.
\]
These complexes form a filtration because $\mathcal{X}_{\varepsilon} \subseteq \mathcal{X}_{\varepsilon'}$ whenever $\varepsilon \leq \varepsilon'$.

Over the fixed coefficient field, the zeroth Betti number $\beta_0(\mathcal{X}_{\varepsilon}) := \dim\mathcal{H}_0(\mathcal{X}_{\varepsilon})$ equals the number of connected components.
As the scale increases, components can merge but cannot split, and all edges of a given length enter at that scale.
Thus, $\beta_0(\mathcal{X}_{\varepsilon})$ is a nonincreasing, right-continuous step function of $\varepsilon$.
Along this filtration, two scales have the same component count exactly when they induce the same partition of $[n]$ into connected components.
Our goal is to select a minimal set of scales that captures all distinct values of $\beta_0$ attained along the filtration.

\begin{definition}[Sufficient Scale Set]
A finite set $\Lambda \subset [0,\infty)$ is called \textbf{sufficient} for the connectivity evolution if, for every $\varepsilon \geq 0$, there exists $\varepsilon' \in \Lambda$ such that $\beta_0(\mathcal{X}_{\varepsilon}) = \beta_0(\mathcal{X}_{\varepsilon'})$.
\end{definition}

Since connectivity can change only at pairwise distances, the \textbf{distance-scale set}
\[
\Lambda_{\mathrm{dist}} := \{0\} \cup \Lambda_{\mathrm{dist}}^+, \qquad
\Lambda_{\mathrm{dist}}^+ := \{r_{ij}: 1 \leq i < j \leq n, \ r_{ij}>0\}
\]
is sufficient.
It may nevertheless contain several scales with the same value of $\beta_0$, since adding edges within an existing connected component does not change the number of connected components.
We remove this redundancy by selecting the earliest scale at which each distinct value of $\beta_0$ is attained.

\begin{definition}[Canonical Minimal Sufficient Scale Set]
A sufficient scale set $\Lambda$ is called \textbf{minimal} if no proper subset remains sufficient.
Furthermore, it is called \textbf{canonical} if for every $\varepsilon \in \Lambda$, 
\[
\varepsilon = \inf\{s \geq 0: \beta_0(\mathcal{X}_s) = \beta_0(\mathcal{X}_{\varepsilon})\}.
\]
\end{definition}

To construct this set, let $\varepsilon_0:=0$ and denote by $0<\varepsilon_1<\cdots<\varepsilon_m$ the distinct positive distances at which the number of connected components decreases, that is,
\[
\lim_{\varepsilon \to \varepsilon_t^-}\beta_0(\mathcal{X}_{\varepsilon}) > \beta_0(\mathcal{X}_{\varepsilon_t}), \qquad
1 \leq t \leq m.
\]
The values $\varepsilon_1,\ldots,\varepsilon_m$ are precisely the distinct positive finite $\mathcal{H}_0$ death values~\cite{edelsbrunner2010computational}.
We define the \textbf{persistence-guided scale set}
\[
\Lambda_{\mathrm{PH}} := \{0\} \cup \Lambda_{\mathrm{PH}}^+, \qquad
\Lambda_{\mathrm{PH}}^+ := \{\varepsilon_1,\ldots,\varepsilon_m\}, 
\]
with $m=0$ and $\Lambda_{\mathrm{PH}}^+=\varnothing$ if there is no positive transition.

\begin{theorem}
\label{thm:critical_scale_set}
$\Lambda_{\mathrm{PH}}$ is the unique canonical minimal sufficient scale set and contains exactly one representative of each attained value of $\beta_0(\mathcal{X}_{\varepsilon})$.
\end{theorem}

The characterization follows from the intervals on which $\beta_0(\mathcal{X}_{\varepsilon})$ is constant, as shown in Section~\ref{Supp-sec:thm_critical_scale_set} of the supplementary material.
The filtration starts at $\varepsilon=0$, where the graph consists of $n$ isolated vertices if all sample locations are distinct.
Each positive selected scale marks one or more mergers of connected components.
When $m \geq 1$, the last selected scale $\varepsilon_m$ is the first scale at which the graph is connected.
Repeated death values contribute only one selected scale.
Since the number of connected components decreases by at least one at each positive selected scale, we have $m \leq \beta_0(\mathcal{X}_0)-1 \leq n-1$.

In practice, Kruskal's algorithm yields a minimum spanning tree whose distinct positive edge weights form $\Lambda_{\mathrm{PH}}^+$~\cite{edelsbrunner2010computational}.
This gives a union--find construction of the selected scales, with the treatment of tied distances described in Section~\ref{Supp-sec:scale_computation} of the supplementary material.

\subsection{Multi-Scale Graph Construction}
We use the positive selected scales $\Lambda_{\mathrm{PH}}^+$ to index the graph sequence and assume from here on that $X$ contains at least two distinct columns, so $m \geq 1$.
For $\varepsilon>0$, define a weighted affinity graph $G(\varepsilon)$ on $[n]$ by its symmetric adjacency matrix $A(\varepsilon) \in \mathbb{R}_+^{n \times n}$:
\begin{equation}
\label{eq:adjacency_matrix}
A_{ij}(\varepsilon) = 
\begin{cases}
\exp \! \left(-\dfrac{r_{ij}^2}{\varepsilon^\alpha}\right), &i \neq j, \ r_{ij} \leq \varepsilon, \\
0, &\text{otherwise},
\end{cases}
\end{equation}
where $\alpha>0$ is a fixed graph-decay exponent.
Let $D(\varepsilon)$ be the diagonal degree matrix with $D_{ii}(\varepsilon):=\sum_{j=1}^n A_{ij}(\varepsilon)$ and define the unnormalized graph Laplacian by $L(\varepsilon):=D(\varepsilon)-A(\varepsilon)$.
Restricting this family to $\Lambda_{\mathrm{PH}}^+$ yields the persistence-guided graph sequence $\{G(\varepsilon_t)\}_{t=1}^{m}$ and Laplacian sequence $\{L(\varepsilon_t)\}_{t=1}^{m}$.

To describe how these regularizers vary with scale, write the distinct positive pairwise distances as $\Lambda_{\mathrm{dist}}^+ = \{\delta_1<\cdots<\delta_M\}$. 
The edge set is constant on each interval
\[
I_0:=(0,\delta_1), \qquad
I_k:=[\delta_k,\delta_{k+1}) \quad (1 \leq k < M), \qquad
I_M:=[\delta_M,\infty).
\]
Within each interval, the Laplacian varies only through the weights of the active edges, which gives the following bound.

\begin{proposition}[Piecewise Lipschitz Continuity of the Laplacian]
\label{prop:piecewise_lipschitz_continuity}
There exists a constant $C_{\mathrm{fix}}>0$, depending only on $X$ and the fixed exponent $\alpha$, such that
\[
\|L(\varepsilon)-L(\varepsilon')\|_F \leq C_{\mathrm{fix}}|\varepsilon-\varepsilon'| 
\]
whenever $\varepsilon,\varepsilon'$ belong to the same interval $I_k$.
\end{proposition}

The proof uses a uniform bound on the derivatives of the active edge weights, as detailed in Section~\ref{Supp-sec:pro_piecewise_lipschitz_continuity} of the supplementary material.
For the spectral conclusions of the following theorem, let $0 = \gamma_1(\varepsilon) \leq \cdots \leq \gamma_n(\varepsilon)$ denote the eigenvalues of $L(\varepsilon)$.
The following theorem follows from \Cref{prop:piecewise_lipschitz_continuity}.


\begin{theorem}[Spectral Variation and Monotonicity Along Persistence-Selected Scales]
\label{thm:spectral_stability_monotonicity}
The selected positive scales have the following properties for every $1 \leq t < m$.

\emph{(1) Laplacian and spectral variation.}
There exists a constant $C_{\mathrm{adj}}>0$, depending on $X$ and $\alpha$ but independent of $t$, such that
\[
\begin{aligned}
\|L(\varepsilon_{t+1})-L(\varepsilon_t)\|_F &\leq C_{\mathrm{adj}}(\varepsilon_{t+1}-\varepsilon_t), \\
|\gamma_\ell(\varepsilon_{t+1})-\gamma_\ell(\varepsilon_t)| &\leq C_{\mathrm{adj}}(\varepsilon_{t+1}-\varepsilon_t), \qquad
1 \leq \ell \leq n.
\end{aligned}
\]

\emph{(2) Spectral monotonicity.}
\[
\gamma_\ell(\varepsilon_t) \leq \gamma_\ell(\varepsilon_{t+1}), \qquad
1 \leq \ell \leq n.
\]

\end{theorem}

The proof is given in Section~\ref{Supp-sec:thm_lipschitz_bound} of the supplementary material.
The bound in part~(1) applies to the fixed finite sequence.
The constant $C_{\mathrm{adj}}$ includes jump contributions divided by the selected scale gaps and need not remain bounded across datasets as distinct selected scales approach one another.

\subsection{Multi-Scale Persistent NMF Model}
\label{sec:pNMF}
At each selected scale $\varepsilon_t$, pNMF approximates $X$ by $W_tH_t$, where $W_t \in \mathbb{R}_+^{p \times d}$ is the basis matrix and $H_t \in \mathbb{R}_+^{d \times n}$ is the embedding.
To relate these representations across scales, we learn the \textit{multi-scale low-rank embedding sequence} $\{H_t\}_{t=1}^m$ jointly and penalize differences between adjacent embeddings, while retaining a reconstruction term and a graph penalty at each scale.
A quadratic penalty on both factor sequences additionally controls their norms.
For notational convenience, set $A_t:=A(\varepsilon_t)$, $D_t:=D(\varepsilon_t)$, and $L_t:=L(\varepsilon_t)$.
With regularization parameters $\lambda_1 > 0$, $\lambda_2 \geq 0$, and $\lambda_3 > 0$, this gives the objective
\begin{equation}
\label{eq:pnmf_model}
\begin{aligned}
\min_{\substack{W_t,H_t \geq 0 \\ 1 \leq t \leq m}} \ \mathcal{O} = {}&\underbrace{\sum_{t=1}^{m}\|X-W_tH_t\|_F^2}_{\text{data reconstruction}} + \underbrace{\lambda_1\sum_{t=1}^{m}\operatorname{Tr}(H_tL_tH_t^\top)}_{\text{geometric regularization}} \\
&+ \underbrace{\lambda_2\sum_{t=2}^{m}\|H_t-H_{t-1}\|_F^2}_{\text{cross-scale smoothness}} + \underbrace{\lambda_3\sum_{t=1}^{m}\bigl(\|W_t\|_F^2+\|H_t\|_F^2\bigr)}_{\text{factor regularization}}.
\end{aligned}
\end{equation}
The learned sequence gives a nonnegative low-rank embedding at each selected scale, with the cross-scale term comparing the representations of the same sample.
Setting $\lambda_2=0$ separates the objective into independent factor-regularized GNMF problems, while for $m=1$ the cross-scale sum is empty and interpreted as zero.
Because the reconstruction, geometric, and cross-scale terms are nonnegative, the factor penalty also ensures that the joint objective is coercive.

\begin{proposition}[Global Objective Coercivity]
\label{prop:factor_coercivity}
For $\lambda_1 > 0$, $\lambda_2 \geq 0$, and $\lambda_3 > 0$, the pNMF objective satisfies
\[
\mathcal{O} \geq \lambda_3\sum_{t=1}^m\bigl(\|W_t\|_F^2+\|H_t\|_F^2\bigr). 
\]
Consequently, $\mathcal{O}$ is coercive on the nonnegative feasible set, every sublevel set is compact, and the pNMF problem admits a global minimizer.
\end{proposition}

Section~\ref{Supp-sec:prop_factor_coercivity} of the supplementary material establishes compactness and existence from the coercive bound and continuity of $\mathcal{O}$.
To analyze how the embeddings evolve across scales and to derive their joint update, we fix the basis sequence and treat the embeddings at all scales as a single block of variables.
\Cref{sec:block} derives the resulting block tridiagonal operator and establishes the coercivity and inverse-block decay bounds used in the stability analysis.

\section{Block Tridiagonal Structure of the Coupled System}
\label{sec:block}

With the basis sequence $\{W_t\}_{t=1}^m$ fixed, the embedding variables enter the pNMF objective quadratically.
Because the cross-scale penalty couples only adjacent embeddings, the Hessian is block tridiagonal.
We derive this operator and show how factor regularization controls its invertibility and the decay of its inverse blocks.
Throughout this section, $m \geq 1$ and the parameters satisfy $\lambda_1 > 0$, $\lambda_2 \geq 0$, and $\lambda_3 > 0$, as in \Cref{eq:pnmf_model}.

At an interior stationary point of this conditional problem, differentiating \Cref{eq:pnmf_model} with respect to $H_t$ and dividing by two gives, for every $1 \leq t \leq m$,
\begin{equation}
\label{eq:interior_stationarity}
\bigl(W_t^\top W_t+(\nu_t\lambda_2+\lambda_3)I_d\bigr)H_t + \lambda_1H_tL_t - \lambda_2\sum_{j\in\mathcal{N}_t}H_j = W_t^\top X,
\end{equation}
where $I_d$ denotes the identity matrix of order $d$, $\mathcal{N}_t := \{t-1,t+1\} \cap \{1,\ldots,m\}$, and $\nu_t := |\mathcal{N}_t|$.
For $m \geq 2$, $\nu_t=2$ at interior scales and $\nu_t=1$ at the boundaries.

To collect these equations, let $\operatorname{vec}(\cdot)$ denote column stacking and $\otimes$ the Kronecker product, and set
\[
\mathbf{h} := 
\begin{bmatrix}
\operatorname{vec}(H_1) \\
\vdots \\
\operatorname{vec}(H_m)
\end{bmatrix}, \qquad
\mathbf{b} := 
\begin{bmatrix}
\operatorname{vec}(W_1^\top X) \\
\vdots \\
\operatorname{vec}(W_m^\top X)
\end{bmatrix} 
\in \mathbb{R}^{dmn}.
\]
Using the identity $\operatorname{vec}(ABC)=(C^\top \otimes A)\operatorname{vec}(B)$, together with $L_t^\top=L_t$, we can write \Cref{eq:interior_stationarity} as
\begin{equation}
\label{eq:global_system}
\mathcal{A} \mathbf{h} = \mathbf{b},
\end{equation}
where $\mathcal{A} \in \mathbb{R}^{dmn \times dmn}$ has the block tridiagonal form
\begin{equation}
\label{eq:global_block_operator}
\mathcal{A}_{t,s} = 
\begin{cases}
\mathcal{M}_t, &t=s, \\
-\lambda_2I_{dn}, &|t-s|=1, \\
0, &|t-s|>1, 
\end{cases}
\qquad
1 \leq t,s \leq m,
\end{equation}
with $dn \times dn$ diagonal blocks
\begin{equation}
\label{eq:diagonal_block_M}
\mathcal{M}_t = I_n \otimes(W_t^\top W_t) + \lambda_1L_t \otimes I_d + (\nu_t\lambda_2+\lambda_3)I_{dn}.
\end{equation}
This blockwise definition also covers $m=1$, when $\nu_1=0$ and $\mathcal{A} = \mathcal{M}_1$.

The same operator describes the conditional objective at every feasible point.
Indeed, with $\mathcal{O}_H$ denoting the objective for fixed basis matrices, expansion gives
\begin{equation}
\label{eq:conditional_embedding_quadratic}
\mathcal{O}_H(\mathbf{h}) = \mathbf{h}^\top\mathcal{A}\mathbf{h}-2\mathbf{b}^\top\mathbf{h}+c,
\end{equation}
where $c=m\|X\|_F^2+\lambda_3\sum_{t=1}^m\|W_t\|_F^2$ is independent of $\mathbf{h}$.
Thus, the Hessian is $2\mathcal{A}$ whether or not nonnegativity constraints are active.
The expansion is detailed in Section~\ref{Supp-sec:block_quadratic_form} of the supplementary material.

For a symmetric matrix $S$, write $\gamma_{\min}(S)$ and $\gamma_{\max}(S)$ for its smallest and largest eigenvalues.
We use $\|\cdot\|_2$ for the spectral norm of matrices and the Euclidean norm of vectors.
The positive semidefinite graph and reconstruction terms then give the following lower bound.

\begin{proposition}[Uniform Block Coercivity]
\label{prop:diagonal_dominance}
The operator $\mathcal{A}$ is symmetric and block tridiagonal, and its block-row coercivity margins satisfy
\begin{equation}
\label{eq:block_coercivity}
\gamma_{\min}(\mathcal{M}_t) - \sum_{\substack{1 \leq s \leq m \\ s \neq t}}\|\mathcal{A}_{t,s}\|_2 \geq \lambda_3, \qquad
1 \leq t \leq m.
\end{equation}
Consequently, $\mathcal{A}$ is strictly block diagonally dominant in this spectral-norm sense, and
\begin{equation}
\label{eq:global_coercivity}
\mathcal{A} \succeq \lambda_3I_{dmn}.
\end{equation}
In particular, it is invertible and $\|\mathcal{A}^{-1}\|_2 \leq \lambda_3^{-1}$.
\end{proposition}

The proof uses blockwise estimates and a decomposition of $\mathcal{A}-\lambda_3I_{dmn}$ into positive semidefinite terms, as detailed in Section~\ref{Supp-sec:block_quadratic_form} of the supplementary material.

For fixed basis matrices, the proposition implies that the conditional embedding objective in \Cref{eq:conditional_embedding_quadratic} is $2\lambda_3$-strongly convex and has a unique minimizer on the nonnegative orthant.
The uniform positive-definiteness bound and block tridiagonal structure also underpin the subsequent analysis of cross-scale influence and stability.
The minimizer is characterized by the Karush--Kuhn--Tucker (KKT) conditions
\begin{equation}
\label{eq:embedding_kkt}
\mathbf{h} \geq 0, \qquad
\mathbf{r} := \mathcal{A}\mathbf{h} - \mathbf{b} \geq 0, \qquad
\mathbf{h}^\top\mathbf{r} = 0,
\end{equation}
where the vector inequalities are componentwise.
If $\mathbf{h}>0$, complementarity implies $\mathbf{r}=0$, yielding \Cref{eq:global_system}.
If some entries of $\mathbf{h}$ are zero, the corresponding entries of $\mathbf{r}$ may be positive, so \Cref{eq:global_system} need not hold.

For the linear system \Cref{eq:global_system} with $\mathcal{A}$ fixed, the inverse-norm bound controls the response to changes in $\mathbf{b}$, while the block structure determines how that response spreads across scales.
A conformably partitioned matrix has block bandwidth at most $k$ if its $(t,s)$ block vanishes whenever $|t-s|>k$. 
For $\mathcal{A}$, this locality yields exponential decay of the inverse blocks.

\begin{proposition}[Block Locality and Exponential Inverse Decay]
\label{prop:structural_bandwidth}
The operator $\mathcal{A}$ has block bandwidth at most one.
Define
\[
\kappa_2(\mathcal{A}) := \frac{\gamma_{\max}(\mathcal{A})}{\gamma_{\min}(\mathcal{A})}, \qquad
\rho_{\mathcal{A}} := \frac{\sqrt{\kappa_2(\mathcal{A})}-1}{\sqrt{\kappa_2(\mathcal{A})}+1} \in [0,1).
\]
Then, for $1 \leq t,k \leq m$,
\begin{equation}
\label{eq:demko_bound}
\bigl\|(\mathcal{A}^{-1})_{t,k}\bigr\|_2 \leq \frac{2}{\gamma_{\min}(\mathcal{A})}\rho_{\mathcal{A}}^{|t-k|} \leq \frac{2}{\lambda_3}\rho_{\mathcal{A}}^{|t-k|},
\end{equation}
with the convention $\rho_{\mathcal{A}}^0=1$, including when $\rho_{\mathcal{A}}=0$.
\end{proposition}

The proof applies the polynomial approximation argument of Demko, Moss, and Smith~\cite{demko1984decay} to the scale blocks.
A polynomial of degree $k$ in $\mathcal{A}$ has block bandwidth at most $k$, so the approximation error bounds every inverse block outside that band.
Section~\ref{Supp-sec:prop_structural_bandwidth} of the supplementary material gives the complete argument, including the diagonal blocks and $\kappa_2(\mathcal{A})=1$.
The prefactor is controlled uniformly by $\lambda_3$, while the decay rate depends on the condition number.

By the spectral monotonicity in \Cref{thm:spectral_stability_monotonicity}(2), the maximum Laplacian eigenvalue over the selected scales is attained at the last scale, so $\max_t\gamma_{\max}(L_t)=\gamma_{\max}(L_m)$.
Combining the lower bound $\gamma_{\min}(\mathcal{A})\geq\lambda_3$ from \Cref{prop:diagonal_dominance} with the spectral upper bound $4\lambda_2$ for the coupling term gives
\begin{equation}
\label{eq:condition_number}
\kappa_2(\mathcal{A}) \leq K:= \frac{\max_t\|W_t\|_2^2+\lambda_1\gamma_{\max}(L_m)+4\lambda_2+\lambda_3}{\lambda_3}.
\end{equation}
The inverse-block decay factor in \Cref{prop:structural_bandwidth} therefore satisfies $\rho_{\mathcal{A}}\leq(\sqrt{K}-1)/(\sqrt{K}+1)$.

With $\mathcal{A}$ fixed, this proposition bounds the response at scale $t$ to a change confined to the right-hand-side block at scale $k$ in \Cref{eq:global_system}.
The bound decays exponentially in the scale-index distance $|t-k|$, with a faster rate when $\mathcal{A}$ is better conditioned.
For a fixed active set, analogous bounds hold for the inverse of the reduced operator on the free coordinates, as detailed in Section~\ref{Supp-sec:block_active_set} of the supplementary material.
\Cref{subsec:adjacent_scale_evolution} uses the adjacent-scale Laplacian bound in \Cref{thm:spectral_stability_monotonicity} to control the graph contribution to embedding increments.
\Cref{subsec:perturbation_stability} uses the coercivity and inverse-block decay of the embedding operator established in \Cref{prop:diagonal_dominance,prop:structural_bandwidth} to study how the embeddings respond to input perturbations.

\section{Evolution and Stability Analysis}
\label{sec:evolution_stability}

The block operator allows us to analyze both variation across scales and sensitivity to input perturbations.
We first bound adjacent-scale increments at a KKT point, then study perturbation responses with the basis sequence fixed.

\subsection{Adjacent-Scale Embedding Evolution}
\label{subsec:adjacent_scale_evolution}
We study variation between adjacent scales using only the embedding KKT conditions at a KKT point $\{(W_t,H_t)\}_{t=1}^m$ of \Cref{eq:pnmf_model}, with $m \geq 2$.
For $2 \leq t \leq m$, define the adjacent-scale increments
\[
\Delta H_t:=H_t-H_{t-1}, \qquad
\Delta W_t:=W_t-W_{t-1}, \qquad
\Delta L_t:=L_t-L_{t-1},
\]
and set $\Delta H_1=\Delta H_{m+1}=0$ for the boundary terms.
When all embedding entries are positive, subtracting \Cref{eq:interior_stationarity} at scale $t-1$ from that at scale $t$ and collecting the embedding increments yields
\[
W_t^\top W_t\Delta H_t + \lambda_1(\Delta H_t)L_t + (2\lambda_2+\lambda_3)\Delta H_t = \mathcal{R}_t + \lambda_2\bigl(\Delta H_{t-1}+\Delta H_{t+1}\bigr), 2 \leq t \leq m,
\]
where the local forcing is
\begin{equation}
\label{eq:local_driving_term}
\mathcal{R}_t := (\Delta W_t)^\top X - \bigl(W_t^\top W_t-W_{t-1}^\top W_{t-1}\bigr)H_{t-1} - \lambda_1H_{t-1}\Delta L_t.
\end{equation}
The three terms capture changes in the reconstruction source $W_t^\top X$, the Gram matrix $W_t^\top W_t$, and the graph Laplacian $L_t$, respectively. By the adjacent-scale Laplacian estimate in \Cref{thm:spectral_stability_monotonicity}(1),
the graph contribution to the local forcing satisfies
\[
\lambda_1\|H_{t-1}\Delta L_t\|_F \leq \lambda_1\|H_{t-1}\|_2\|\Delta L_t\|_F \leq \lambda_1 C_{\mathrm{adj}}\|H_{t-1}\|_2(\varepsilon_t-\varepsilon_{t-1}).
\]
Thus, the gap between consecutive selected scales bounds the
graph-dependent contribution to the embedding increment estimates.

The following bounds also hold at general KKT points because complementary slackness makes the constraint contribution nonpositive after taking the inner product with $\Delta H_t$.
The proofs of \Cref{lem:local_embedding_increment,cor:uniform_embedding_increment} are given in Section~\ref{Supp-sec:thm_embedding_increment_bounds} of the supplementary material.

\begin{lemma}[Cross-Scale Recursive Inequality]
\label{lem:local_embedding_increment}
At the KKT point above,
\begin{equation}
\label{eq:recursive_inequality}
\mu_t\|\Delta H_t\|_F \leq \|\mathcal{R}_t\|_F + \lambda_2\bigl(\|\Delta H_{t-1}\|_F + \|\Delta H_{t+1}\|_F\bigr), \qquad
2 \leq t \leq m,
\end{equation}
where $\mu_t=\gamma_{\min}(W_t^\top W_t)+2\lambda_2+\lambda_3$.
\end{lemma}

Set $\mu := \min_{2 \leq t \leq m}\mu_t$.
Since the Gram matrices are positive semidefinite, $\mu-2\lambda_2 \geq \lambda_3 > 0$, even when the basis matrices are rank deficient.
Applying the recurrence at an index of maximal increment gives the following bound.

\begin{corollary}[Uniform Bound on Embedding Increments]
\label{cor:uniform_embedding_increment}
Under the assumptions of \Cref{lem:local_embedding_increment},
\begin{equation}
\label{eq:uniform_embedding_increment}
\max_{2 \leq t \leq m}\|\Delta H_t\|_F \leq \frac{\max_{2 \leq t \leq m}\|\mathcal{R}_t\|_F}{\mu-2\lambda_2} \leq \frac{1}{\lambda_3}\max_{2 \leq t \leq m}\|\mathcal{R}_t\|_F.
\end{equation}
\end{corollary}

This bound shows that small local forcing yields uniformly small embedding increments across scales.
The recurrence (\Cref{eq:recursive_inequality}) also captures the scale locations of the local forcing.
Comparison with a scalar tridiagonal system gives the following pointwise bound.

\begin{theorem}[Pointwise Bound on Embedding Increments]
\label{thm:exponential_localization}
Under the assumptions of \Cref{lem:local_embedding_increment},
\begin{equation}
\label{eq:exponential_localization}
\|\Delta H_t\|_F \leq \frac{1}{\mu\sqrt{1-4b_{\mathrm{sc}}^2}}\sum_{k=2}^m\rho_{\mathrm{sc}}^{|t-k|}\|\mathcal{R}_k\|_F, \qquad
2 \leq t \leq m,
\end{equation}
where
\[
b_{\mathrm{sc}}:=\frac{\lambda_2}{\mu} \in \left[0,\frac{1}{2}\right), \qquad
\rho_{\mathrm{sc}} := \frac{2b_{\mathrm{sc}}}{1+\sqrt{1-4b_{\mathrm{sc}}^2}} \in [0,1).
\]
The convention $\rho_{\mathrm{sc}}^0=1$ applies also when $\lambda_2=0$.
\end{theorem}

\begin{proof}
Set $u_t:=\|\Delta H_t\|_F$ and $f_t:=\|\mathcal{R}_t\|_F/\mu$, with $u_1=u_{m+1}=0$.
By \Cref{eq:recursive_inequality} and $\mu_t \geq \mu$,
\[
u_t-b_{\mathrm{sc}}(u_{t-1}+u_{t+1}) \leq f_t, \qquad
2 \leq t \leq m.
\]
If $b_{\mathrm{sc}}=0$, the conclusion follows immediately.
Otherwise, the operator
\[
(\mathcal Lv)_t:=v_t-b_{\mathrm{sc}}(v_{t-1}+v_{t+1}), \qquad
t \in \mathbb{Z},
\]
has Green function
\[
G_{t,k} = \frac{\rho_{\mathrm{sc}}^{|t-k|}}{\sqrt{1-4b_{\mathrm{sc}}^2}},
\]
as follows by substitution using $b_{\mathrm{sc}}\rho_{\mathrm{sc}}^2-\rho_{\mathrm{sc}}+b_{\mathrm{sc}}=0$.

Extend $f$ by zero outside $\{2,\ldots,m\}$ and set $v_t:=\sum_{k=2}^mG_{t,k}f_k$.
Then $\mathcal{L}v=f$ and $v \geq 0$.
Hence $w:=u-v$ satisfies $\mathcal{L}w \leq 0$ on $\{2,\ldots,m\}$ and $w_1,w_{m+1} \leq 0$.
A positive interior maximum of $w$ would imply
\[
(\mathcal{L}w)_t \geq (1-2b_{\mathrm{sc}})w_t > 0,
\]
a contradiction.
Thus $u \leq v$, which yields
\[
\|\Delta H_t\|_F \leq \frac{1}{\mu\sqrt{1-4b_{\mathrm{sc}}^2}}\sum_{k=2}^m\rho_{\mathrm{sc}}^{|t-k|}\|\mathcal{R}_k\|_F.
\]
\end{proof}

The theorem bounds each embedding increment by the local forcing across scales, with exponentially decaying weights in the scale-index distance.

The cross-scale Laplacian estimates in \Cref{thm:spectral_stability_monotonicity} control the graph component of the local forcing, accounting for both changes in edge weights and edge activations.
By \Cref{prop:factor_coercivity}, the basis-dependent component is bounded by a constant times $\|\Delta W_t\|_F$, uniformly across scales on each objective sublevel set.
The increment bounds therefore depend on both graph and basis variation.
These estimates describe variation along the fixed scales used in the model.
The next subsection studies embedding stability under input perturbations.

\subsection{Perturbation Stability and Scale Localization}
\label{subsec:perturbation_stability}
Let $\widetilde{X} = X + \delta X \geq 0$ denote the perturbed data matrix.
We fix the basis sequence $\{W_t\}_{t=1}^m$ and all regularization parameters in both conditional embedding problems and study how input perturbations affect their minimizers through changes in the reconstruction source and graph Laplacians.

Assume that $X$ and $\widetilde{X}$ each have $m \geq 1$ positive selected scales, matched in increasing order.
Write $\varepsilon_t:=\varepsilon_t(X)$ and $\widetilde{\varepsilon_t}:=\varepsilon_t(\widetilde{X})$, and use the same convention for the corresponding graph matrices $A_t$, $D_t$, and $L_t$.

Let $\mathbf{h}:=\mathbf{h}(X)$ and $\widetilde{\mathbf{h}}:=\mathbf{h}(\widetilde{X})$ denote the unique conditional embedding minimizers.
Their quadratic operators and sources are $\mathcal{A}:=\mathcal{A}(X)$, $\mathbf{b}:=\mathbf{b}(X)$ and $\widetilde{\mathcal{A}}:=\mathcal{A}(\widetilde{X})$, $\widetilde{\mathbf{b}}:=\mathbf{b}(\widetilde{X})$, respectively, as defined in \Cref{sec:block}.
For any input-dependent quantity $Z$, set $\delta Z:=\widetilde{Z}-Z$.
With the basis sequence and cross-scale coupling held fixed, the operator perturbation is confined to the diagonal blocks:
\[
\delta\mathcal{A} = \operatorname{blkdiag}(\delta\mathcal{M}_1,\ldots,\delta\mathcal{M}_m), \qquad
\delta\mathcal{M}_t = \lambda_1\delta L_t \otimes I_d,
\]
where $\operatorname{blkdiag}$ denotes block-diagonal assembly.
The lower bound in \Cref{prop:diagonal_dominance} allows us to control embedding perturbations even when the active set changes.

\begin{proposition}[Conditional Embedding Perturbation Bound]
\label{prop:global_embedding_perturbation}
Under the fixed-basis and scale-correspondence assumptions, the conditional minimizers satisfy
\begin{equation}
\label{eq:global_operator_perturbation_bound}
\|\delta \mathbf{h}\|_2 \leq \frac{1}{\lambda_3}\bigl(\|\delta\mathbf{b}\|_2+\|\delta\mathcal{A}\|_2\|\mathbf{h}\|_2\bigr).
\end{equation}
If both minimizers are interior, they also satisfy the exact identity
\begin{equation}
\label{eq:perturbation_identity}
\delta\mathbf{h} = \widetilde{\mathcal{A}}^{-1}\bigl(\delta \mathbf{b}-\delta\mathcal{A}\,\mathbf{h}\bigr).
\end{equation}
\end{proposition}

Section~\ref{Supp-sec:global_embedding_perturbation} of the supplementary material derives the bound from variational inequalities and the identity by subtracting the interior linear systems.

To relate the operator change to the input, define the graph support by
\[
\mathcal{E}_t(X) := \bigl\{\{i,j\} : 1 \leq i < j \leq n, \ \|x_i-x_j\|_2 \leq \varepsilon_t(X)\bigr\}.
\]
We first consider perturbations that preserve the graph supports and the vertex pairs representing the selected scales.

\begin{lemma}[Local Laplacian Stability]
\label{lem:input_laplacian_stability}
Suppose that the distances $\|x_i-x_j\|_2$, $1 \leq i < j \leq n$, are positive and pairwise distinct.
Then there exist $r,C_L>0$ such that every nonnegative input $\widetilde{X}$ with $\|\delta X\|_F<r$ has the same number of positive selected scales and the same graph supports at corresponding scales, and satisfies
\begin{equation}
\label{eq:input_laplacian_stability}
\max_{1 \leq t \leq m}\|\delta L_t\|_2 \leq C_L\|\delta X\|_F.
\end{equation}
\end{lemma}

Distinct positive distances ensure that sufficiently small perturbations preserve their ordering, and hence the merger edges and graph supports.
The proof in Section~\ref{Supp-sec:lem_input_laplacian_stability} of the supplementary material combines Lipschitz bounds for distances with derivatives of the smooth edge weights at positive scales.

\begin{theorem}[Stability Under Support-Preserving Perturbations]
\label{thm:support_preserving_stability}
Under the fixed-basis assumptions and the assumptions of \Cref{lem:input_laplacian_stability}, the conditional minimizers satisfy the following bound whenever $\|\delta X\|_F<r$:
\[
\|\delta\mathbf{h}\|_2 \leq \frac{C_W+\lambda_1C_L\|\mathbf{h}\|_2}{\lambda_3}\|\delta X\|_F,
\]
where $C_W =\bigl(\sum_{t=1}^m\|W_t\|_2^2\bigr)^{1/2}$.
\end{theorem}

\begin{proof}
With the bases fixed, the source change satisfies
\[
\|\delta\mathbf{b}\|_2^2 = \sum_{t=1}^m\|W_t^\top\delta X\|_F^2 \leq C_W^2\|\delta X\|_F^2.
\]
Moreover, the diagonal block structure and symmetry of the Laplacians give
\[
\|\delta \mathcal{A}\|_2 = \lambda_1\max_{1 \leq t \leq m}\|\delta L_t\|_2 \leq \lambda_1C_L\|\delta X\|_F.
\]
Applying \Cref{eq:global_operator_perturbation_bound} proves the claim.
\end{proof}

This is a local Lipschitz estimate at $X$ within the stated support-preserving class, even when the embedding active sets differ.
Its reconstruction and graph contributions are controlled by $C_W$ and $\lambda_1C_L\|\mathbf{h}\|_2$, respectively.
On an objective sublevel set $\{\mathcal{O} \leq \tau\}$, \Cref{prop:factor_coercivity} further gives $\max\{C_W,\|\mathbf{h}\|_2\} \leq (\tau/\lambda_3)^{1/2}$.

We next allow graph supports to change while retaining the correspondence between the two scale sequences.
Define the support-changing scale indices and the effective perturbation source by
\[
\mathcal{I}_{\mathrm{chg}} := \{t \in \{1,\ldots,m\}:\mathcal{E}_t(\widetilde{X}) \neq \mathcal{E}_t(X)\}, \qquad
\mathbf{g} := \delta\mathbf{b}-\delta\mathcal{A}\,\mathbf{h}.
\]
Writing $\mathbf{h}_t:=\operatorname{vec}(H_t)$, its $t$-th $dn$-dimensional block is
\[
\mathbf{g}_t=(\delta\mathbf{b})_t-\delta\mathcal{M}_t\mathbf{h}_t, \qquad
1 \leq t \leq m.
\]
At interior solutions, \Cref{eq:perturbation_identity} becomes $\delta \mathbf{h}=\widetilde{\mathcal{A}}^{-1}\mathbf{g}$.
The inverse-decay estimate then separates the contributions from support-preserving and support-changing scales.

\begin{theorem}[Localization Under Support-Changing Perturbations]
\label{thm:support_changing_localization}
Under the fixed-basis and scale-correspondence assumptions, suppose the conditional solutions are interior.
Assume, uniformly over the perturbation class,
\[
\|\mathbf{g}_t\|_2 \leq C_{\mathrm{g}}\|\delta X\|_F \quad (t \notin \mathcal{I}_{\mathrm{chg}}), \qquad
\kappa_2(\widetilde{\mathcal{A}}) \leq \overline{\kappa},
\]
where $C_{\mathrm{g}} \geq 0$ and $\overline{\kappa}>1$.
Then
\begin{equation}
\label{eq:support_changing_localization}
\|(\delta\mathbf{h})_t\|_2 \leq \frac{2C_{\mathrm{g}}}{\lambda_3}\frac{1+\overline{\rho}}{1-\overline{\rho}}\|\delta X\|_F+\frac{2}{\lambda_3}\sum_{k \in \mathcal{I}_{\mathrm{chg}}}\overline{\rho}^{|t-k|}\|\mathbf{g}_k\|_2, \qquad
1 \leq t \leq m,
\end{equation}
where $\overline{\rho}:=(\sqrt{\overline{\kappa}}-1)/(\sqrt{\overline{\kappa}}+1) \in (0,1)$.
\end{theorem}

\begin{proof}
Since both conditional minimizers are interior, \Cref{eq:perturbation_identity} gives
\[
(\delta\mathbf{h})_t = \sum_{k=1}^m(\widetilde{\mathcal{A}}^{-1})_{t,k}\mathbf{g}_k.
\]
The condition-number bound and \Cref{prop:diagonal_dominance,prop:structural_bandwidth} imply
\[
\|(\widetilde{\mathcal{A}}^{-1})_{t,k}\|_2 \leq \frac{2}{\lambda_3}\overline{\rho}^{|t-k|},
\]
since $(\sqrt\kappa-1)/(\sqrt\kappa+1)$ is increasing in $\kappa$.
Splitting the sum according to whether the graph support changes and applying the assumed source bound yields
\[
\|(\delta\mathbf{h})_t\|_2 \leq \frac{2C_{\mathrm{g}}}{\lambda_3}\|\delta X\|_F\sum_{k\notin\mathcal{I}_{\mathrm{chg}}}\overline{\rho}^{|t-k|}+\frac{2}{\lambda_3}\sum_{k\in\mathcal{I}_{\mathrm{chg}}}\overline{\rho}^{|t-k|}\|\mathbf{g}_k\|_2.
\]
Finally,
\[
\sum_{k \notin \mathcal{I}_{\mathrm{chg}}}\overline{\rho}^{|t-k|} \leq 1 + 2\sum_{j=1}^{\infty}\overline{\rho}^j = \frac{1+\overline{\rho}}{1-\overline{\rho}},
\]
which proves \Cref{eq:support_changing_localization}.
\end{proof}

The condition number bound in \Cref{eq:condition_number} provides a uniform choice of $\overline{\kappa}$ and thus determines the decay factor $\overline{\rho}$ in \Cref{eq:support_changing_localization}.
The estimate separates a contribution proportional to $\|\delta X\|_F$ from contributions at support-changing scales, whose weights decay exponentially with distance from the target scale.
The background term need not vanish away from the support-changing scales, so the full response need not decay to zero with this distance.
Section~\ref{Supp-sec:common_free_input_bounds} of the supplementary material extends this estimate to a common active set, and Section~\ref{Supp-sec:thm_support_changing_localization} gives sufficient conditions for uniform constants.

\section{Optimization Algorithm}
\label{sec:optimization}

Although the pNMF objective in \Cref{eq:pnmf_model} is jointly nonconvex, fixing either factor sequence gives a strongly convex nonnegative quadratic problem.
We therefore alternate between parallel updates of the basis matrices and a joint update of the embedding sequence, using the block tridiagonal structure for the latter.
Throughout this section, the data and Laplacians are fixed, $m \geq 1$, and $\lambda_1,\lambda_3>0$, $\lambda_2 \geq 0$.

\subsection{Two-Block Alternating Optimization}
Write $\boldsymbol{W}:=(W_1,\ldots,W_m)$ and $\boldsymbol{H}:=(H_1,\ldots,H_m)$, with product-space norms $\|\boldsymbol{W}\|_F^2:=\sum_{t=1}^m\|W_t\|_F^2$, $\|\boldsymbol{H}\|_F^2:=\sum_{t=1}^m\|H_t\|_F^2$.
Following the alternating nonnegative least-squares approach~\cite{kim2008nonnegative,kim2011fast}, one exact outer iteration updates the basis and embedding blocks through
\begin{equation}
\label{eq:block_update}
\boldsymbol{W}^{(l+1)} = \operatorname*{argmin}_{\boldsymbol{W} \geq 0}\mathcal{O}(\boldsymbol{W},\boldsymbol{H}^{(l)}), \qquad
\boldsymbol{H}^{(l+1)} = \operatorname*{argmin}_{\boldsymbol{H} \geq 0}\mathcal{O}(\boldsymbol{W}^{(l+1)},\boldsymbol{H}),
\end{equation}
respectively.
Each block objective is $2\lambda_3$-strongly convex and coercive, so both minimizers exist and are unique.
These exact updates define the iteration analyzed below, while the numerical algorithm targets prescribed residual tolerances subject to finite inner budgets.

\paragraph{$W$-block update}
Since the graph and cross-scale terms are independent of $\boldsymbol{W}$, the basis update separates across scales as
\begin{equation}
\label{eq:W_scale_subproblem}
W_t^{(l+1)} = \operatorname*{argmin}_{W_t \geq 0}\left\{\|X-W_tH_t^{(l)}\|_F^2+\lambda_3\|W_t\|_F^2\right\}, \qquad
1 \leq t \leq m.
\end{equation}
Each scale further separates across feature rows.
Writing $w_{t,i}:=(W_t)_{i,:}^\top$, the row update is
\begin{equation}
\label{eq:W_row_quadratic}
w_{t,i}^{(l+1)} = \operatorname*{argmin}_{w \geq 0}\left\{w^\top G_t^{(l)}w-2(c_{t,i}^{(l)})^\top w\right\}, \qquad
1 \leq i \leq p,
\end{equation}
where $G_t^{(l)}:=H_t^{(l)}(H_t^{(l)})^\top+\lambda_3I_d$ and $c_{t,i}^{(l)}:=H_t^{(l)}X_{i,:}^\top$.
Since $G_t^{(l)} \succeq \lambda_3I_d$, each row subproblem is strongly convex and has a unique minimizer.
The $mp$ row subproblems can be solved independently, with the same $d \times d$ matrix $G_t^{(l)}$ reused across all feature rows at scale $t$.

\paragraph{$H$-block update}
With $\boldsymbol{W}^{(l+1)}$ fixed, the quadratic form from \Cref{sec:block} gives
\begin{equation}
\label{eq:H_joint_qp}
\min_{\mathbf{h} \geq 0}q_H(\mathbf{h}) := \mathbf{h}^\top\mathcal{A}^{(l+1)}\mathbf{h}-2\bigl(\mathbf{b}^{(l+1)}\bigr)^\top\mathbf{h},
\end{equation}
where $\mathbf{h} \in \mathbb{R}^{dmn}$ stacks the embeddings and the operator and right-hand side are evaluated at the updated basis sequence.
Its Hessian is $2\mathcal{A}^{(l+1)} \succeq 2\lambda_3I_{dmn}$.
For $m \geq 2$ and $\lambda_2>0$, the off-diagonal blocks couple adjacent embeddings, so the entire sequence is updated jointly.
When $\lambda_2=0$, the problem separates across scales, and when $m=1$, it is a single factor-regularized GNMF embedding problem.

\paragraph{A unified active-set projected Newton solver}
Both the row-wise basis subproblems and the joint embedding
subproblem share the following quadratic form, allowing them
to be solved within a unified algorithmic framework:
\begin{equation}
\label{eq:commonsolver}
\min_{z \geq 0}q(z):=z^\top Cz-2c^\top z, \qquad
C \succeq \lambda_3I,
\end{equation}
where $(C,c)=(G_t^{(l)},c_{t,i}^{(l)})$ for a basis row and $(C,c)=(\mathcal{A}^{(l+1)},\mathbf{b}^{(l+1)})$ for the embedding block.
The gradient is $g(z):=\nabla q(z)=2(Cz-c)$, and the Hessian is $\nabla^2q(z)=2C$.
At inner iteration $r$, let $g^{(r)}:=g(z^{(r)})$.
Following the active-set projected Newton framework~\cite{bertsekas1982projected}, the exact-zero working-set description uses $\mathcal{J}_r:=\{i:z_i^{(r)}=0,\ g_i^{(r)} \geq 0\}$, with complementary free set $\mathcal{F}_r:=\mathcal{J}_r^c$.
Coordinates in $\mathcal{J}_r$ are held fixed during the current inner step, since increasing them gives no first-order decrease.
The implementation uses the numerical near-zero test specified in Section~\ref{Supp-sec:projected_newton_details} of the supplementary material.

Next, let $p$ denote a candidate step.
Fixing the active coordinates gives $p_{\mathcal{J}_r}=0$, while minimizing the quadratic model over the free coordinates yields $2C_{\mathcal{F}_r,\mathcal{F}_r}p_{\mathcal{F}_r}=-g_{\mathcal{F}_r}^{(r)}$.
We solve this system approximately by preconditioned conjugate gradients (PCG)~\cite{hestenes1952methods,saad2003iterative}.

For the joint embedding subproblem, PCG exploits the block-tridiagonal structure of $\mathcal{A}^{(l+1)}$.
Writing a vector $v$ in matrix blocks $V_t \in \mathbb{R}^{d \times n}$, the $t$-th block of $\mathcal{A}^{(l+1)}v$ is the vectorization of
\begin{equation}
\label{eq:H_hessian_action}
W_t^\top W_tV_t + \lambda_1V_tL_t + (\nu_t\lambda_2+\lambda_3)V_t - \lambda_2V_{t-1} - \lambda_2V_{t+1},
\end{equation}
where $W_t=W_t^{(l+1)}$, $\nu_t$ is the number of adjacent scales, and nonexistent boundary terms are omitted.
Thus, interscale coupling involves only neighboring blocks, and the product can be evaluated directly in matrix form.
Setting $v_{\mathcal{J}_r}=0$, applying this operator, and doubling the output on $\mathcal{F}_r$ yields the reduced Hessian--vector product required by PCG, without assembling either the full or reduced embedding Hessian.

The computed direction is safeguarded by a preconditioned-gradient fallback and a projected Armijo line search~\cite{armijo1966minimization,bertsekas1982projected}.
Independent basis-row problems are processed in batches using their block-diagonal quadratic form.
At outer iteration $l$, the inner solver targets
\[
\frac{\|z-[z-g(z)]_+\|_2}{\max\{1,\|z\|_2\}} \leq \tau_{\mathrm{in}}^{(l)},
\]
where $z$ stacks the variables of a basis batch or the complete embedding block, and $[\cdot]_+$ denotes projection onto the nonnegative orthant.
Meeting this target is distinguished from exhausting an inner budget or a stalled line search.
In joint-test mode, outer stopping requires both small relative objective change and a small normalized joint projected KKT residual evaluated at the same updated factor pair.
Sections~\ref{Supp-sec:projected_newton_details} and~\ref{Supp-sec:outer_iteration_stopping} of the supplementary material specify the numerical rules and termination statuses, including the outer numerical monotonicity safeguard.

The basis updates can be parallelized across scales and feature rows, with each scale's Gram matrix reused across its row subproblems.
For the joint embedding update, PCG exploits the block-tridiagonal structure through Hessian--vector products, avoiding explicit construction and storage of the full or reduced embedding Hessian.
With the Gram matrices precomputed, each reduced embedding Hessian--vector product costs $O(md^2n+dm_L)$, where $m_L:=\sum_{t=1}^m\operatorname{nnz}(L_t)$.
Let $R_W,R_H \geq 1$ bound the Newton iterations per basis-batch solve and joint embedding solve, respectively, and let $J_W,J_H \geq 1$ bound the corresponding PCG iterations per Newton step.
With quadratic objective differences evaluated from cached products, bounded backtracking, and preconditioner reuse, the arithmetic cost over at most $l_{\max}$ outer iterations is
\[
O\!\left(l_{\max}\left[mpnd+md^2(p+n)+md^3+R_WJ_Wmpd^2+R_HJ_H(md^2n+dm_L)\right]\right).
\]
Section~\ref{Supp-sec:computational_complexity} of the supplementary material derives this bound, the storage requirement, and the idealized parallel cost of the basis solves.

\subsection{Convergence Analysis}
We analyze the sequence generated by the exact alternating updates in \Cref{eq:block_update} from a nonnegative initial pair.
Each block objective is strongly convex, yielding sufficient descent, while \Cref{prop:factor_coercivity} ensures that the iterates remain in a compact sublevel set.
Write $\mathcal{O}^{(l)}:=\mathcal{O}(\boldsymbol{W}^{(l)},\boldsymbol{H}^{(l)})$.

\begin{theorem}
\label{thm:subproblem_nonincreasing}
The exact alternating updates generate a nonincreasing sequence of objective values $\{\mathcal{O}^{(l)}\}$ converging to a finite limit.
Moreover,
\begin{equation}
\label{eq:solution-limit}
\|\boldsymbol{W}^{(l+1)}-\boldsymbol{W}^{(l)}\|_F+\|\boldsymbol{H}^{(l+1)}-\boldsymbol{H}^{(l)}\|_F \longrightarrow 0 \qquad
\text{as } l \to \infty.
\end{equation}
\end{theorem}

\begin{proof}
For either block objective $q$, its exact minimizer $z^*$ satisfies the constrained first-order condition.
Combining that condition with $2\lambda_3$-strong convexity gives, for every feasible $z$,
\[
q(z) \geq q(z^*) + \langle\nabla q(z^*),z-z^*\rangle + \lambda_3\|z-z^*\|_2^2 \geq q(z^*) + \lambda_3\|z-z^*\|_2^2.
\]
Applying this estimate to the two updates yields
\begin{align*}
\mathcal{O}(\boldsymbol{W}^{(l)},\boldsymbol{H}^{(l)}) - \mathcal{O}(\boldsymbol{W}^{(l+1)},\boldsymbol{H}^{(l)}) &\geq \lambda_3\|\boldsymbol{W}^{(l+1)}-\boldsymbol{W}^{(l)}\|_F^2, \\
\mathcal{O}(\boldsymbol{W}^{(l+1)},\boldsymbol{H}^{(l)}) -\mathcal{O}(\boldsymbol{W}^{(l+1)},\boldsymbol{H}^{(l+1)}) &\geq \lambda_3\|\boldsymbol{H}^{(l+1)}-\boldsymbol{H}^{(l)}\|_F^2.
\end{align*}
Adding these inequalities yields
\begin{equation}
\label{eq:full_sufficient_descent}
\mathcal{O}^{(l)}-\mathcal{O}^{(l+1)} \geq \lambda_3\left(\|\boldsymbol{W}^{(l+1)}-\boldsymbol{W}^{(l)}\|_F^2+\|\boldsymbol{H}^{(l+1)}-\boldsymbol{H}^{(l)}\|_F^2\right).
\end{equation}
Thus, the objective values are nonincreasing and bounded below by zero, hence convergent.
Since $\mathcal{O}^{(l)}-\mathcal{O}^{(l+1)} \to 0$ and $\lambda_3>0$, \Cref{eq:full_sufficient_descent} implies \Cref{eq:solution-limit}.
\end{proof}

We next establish KKT stationarity of accumulation points.
The KKT conditions for the full pNMF problem are
\begin{equation}
\label{eq:pNMF_KKT}
\begin{cases}
W_t \geq 0, \quad H_t \geq 0, \\
\nabla_{W_t}\mathcal{O}(\boldsymbol{W},\boldsymbol{H}) \geq 0, \quad \nabla_{H_t}\mathcal{O}(\boldsymbol{W},\boldsymbol{H}) \geq 0, \\
W_t\odot\nabla_{W_t}\mathcal{O}(\boldsymbol{W},\boldsymbol{H})=0, \quad H_t\odot\nabla_{H_t}\mathcal{O}(\boldsymbol{W},\boldsymbol{H})=0,
\end{cases}
\qquad 1 \leq t \leq m,
\end{equation}
where the gradients are evaluated at $(\boldsymbol{W},\boldsymbol{H})$, all inequalities are entrywise, and $\odot$ denotes elementwise multiplication.

\begin{theorem}
\label{thm:subproblem_kkt}
Let $\lambda_1>0$, $\lambda_2 \geq 0$, and $\lambda_3>0$.
Starting from any nonnegative initial pair, the sequence generated by the exact alternating updates in \Cref{eq:block_update} is bounded, and every accumulation point satisfies \Cref{eq:pNMF_KKT}.
\end{theorem}

The proof in Section~\ref{Supp-sec:thm_subproblem_kkt} of the supplementary material uses boundedness, vanishing successive differences, and gradient continuity to pass both block optimality conditions to the same limit pair.

\section{Numerical Experiments}
\label{sec:numerical_experiments}

We use synthetic data to investigate the effects of graph connectivity and cross-scale coupling on the embedding sequence.
Three geometric constructions illustrate the embedding evolution, and a paired ablation on the noncoplanar-circle construction isolates the effect of coupling.
We then examine conditional stability under input perturbations and evaluate computational scaling with the sample size and number of selected scales.

\subsection{Embedding Evolution Across Scales}
The first dataset contains $80$ points on four noncoplanar circles in $\mathbb{R}^3$, with radii $5$, $10$, $20$, and $50$ and $20$ points per circle.
The second contains $270$ points in nine filled planar disks, arranged as three groups of three disks.
The third contains $480$ points sampled around a planar branching structure with a shared trunk, two main branches, and four terminal branches, with unequal sibling lengths and a common linear sampling density.
Each construction is augmented with Gaussian nuisance coordinates and translated to obtain a nonnegative input in $100$ dimensions, from which pNMF learns embeddings with $d=2$.
The constructions, their distinct noise models, and the optimization settings are specified in Section~\ref{Supp-sec:synthetic_experiment_settings} of the supplementary material.

\begin{figure}[ht]
\centering
\includegraphics[width=\textwidth]{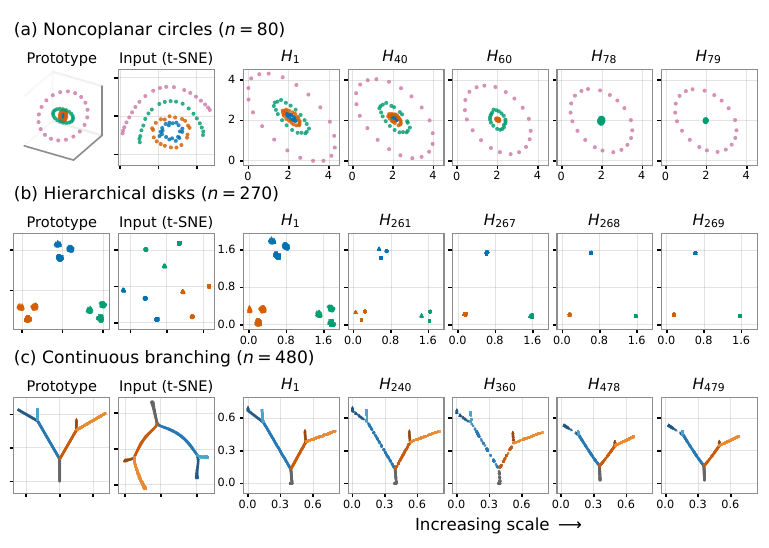}
\caption{Embedding evolution across scales for three synthetic datasets.
Rows show (a) noncoplanar circles ($n=80$), (b) hierarchical disks ($n=270$), and (c) continuous branching ($n=480$).
The first two columns show each dataset's geometric prototype and a two-dimensional t-SNE view of the full $100$-dimensional input.
The last five columns show representative pNMF embeddings $H_t$ in increasing scale order.
Within each row, the embedding panels share axis limits and display the coordinates without alignment or rescaling.
Colors identify generating circles, disk groups, or branch segments, and marker shapes distinguish disks within each group.}
\label{fig:multi_scale_vis}
\end{figure}

The selected sequences contain $79$, $269$, and $479$ positive scales, respectively.
Each row of \Cref{fig:multi_scale_vis} shows the geometric prototype and a t-SNE view~\cite{van2008visualizing} of the full input, followed by five representative embeddings plotted in increasing scale order with common axis limits to show changes in spread.
Section~\ref{Supp-sec:persistence_scale_selection} of the supplementary material explains how the displayed scales were chosen, shows additional embeddings, and reports how many connectivity states each of four scale-selection rules captures on the circle data.

For the noncoplanar circles, smaller-radius groups become concentrated at earlier scales, whereas the largest circle retains a ring-like configuration even at the last selected scale.
The hierarchical disks show a corresponding distinction between within-disk and between-disk variation: the nine disks first become compact, and the three disks within each group subsequently approach one another while the three groups remain separated.
Reading this sequence in reverse reveals the finer grouping within the coarse configuration.
In the branching example, the short terminal segments are more distinct at fine scales and become concentrated near their junctions at the largest selected scales, while longer segments retain more visible extent.
Intermediate embeddings also show local aggregates and gaps along the segments.
At the last selected scale, each input graph is connected, while the embeddings retain geometric organization.

\subsection{Effect of Cross-Scale Coupling}
\label{subsec:cross_scale_coupling_ablation}
The visual evolution alone does not identify the contribution of coupling, so we compare Full pNMF with Decoupled pNMF, which sets only $\lambda_2=0$.
The latter retains geometric and factor regularization at each scale, and its objective is separable across scales.
We use the noncoplanar-circle construction for this controlled comparison, with ten paired data realizations sharing the scales, initialization, and stopping protocol within each pair.

To distinguish geometric changes from changes in size or orientation, we form the matrix $\operatorname{Dist}(H_t)$ of Euclidean distances between the columns of $H_t$ and normalize it by its Frobenius norm.
We assign zero when all distances vanish and denote the normalized matrix by $\widehat{\operatorname{Dist}}(H_t)$, which gives the adjacent embedding change
\[
\operatorname{AEC}_t := \bigl\|\widehat{\operatorname{Dist}}(H_t)-\widehat{\operatorname{Dist}}(H_{t-1})\bigr\|_F, \qquad
2 \leq t \leq m.
\]
This quantity is invariant to translation, orthogonal transformation, and positive uniform rescaling, and it measures change between successive selected scales without division by their threshold spacing.

\begin{figure}[ht]
\centering
\includegraphics[width=\textwidth]{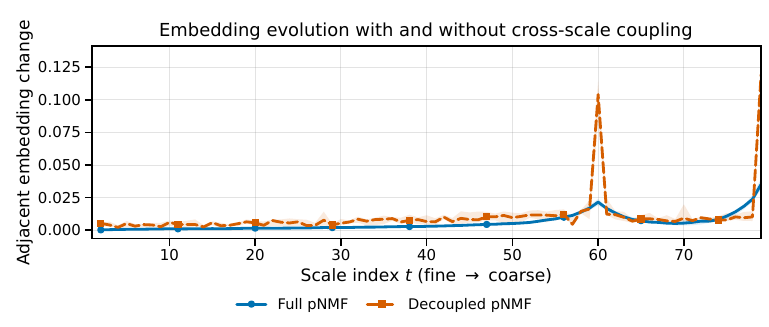}
\caption{Embedding evolution in Full pNMF and Decoupled pNMF on the noncoplanar-circle data.
Curves show the pointwise median of adjacent embedding change $\operatorname{AEC}_t$ across ten paired realizations, and shaded bands show the interquartile range.}
\label{fig:cross_scale_coupling_ablation_body}
\end{figure}

\Cref{fig:cross_scale_coupling_ablation_body} shows smaller median AEC for Full pNMF at $67$ of the $78$ transitions.
After averaging over transitions, Full pNMF has lower AEC in every pair, and its mean across the ten realizations is $51.2\%$ lower than that of Decoupled pNMF.
To assess the associated fitting cost, we compare the relative reconstruction error $E_{\mathrm{rec}} := \frac{1}{m}\sum_{t=1}^m\frac{\|X-W_tH_t\|_F}{\|X\|_F}$.

Full pNMF gives $0.1750\pm0.0179$, compared with $0.1747\pm0.0179$ for Decoupled pNMF (mean $\pm$ sample standard deviation).
The reduction in adjacent geometric change is accompanied by a small mean increase of $2.59\times10^{-4}$ in relative reconstruction error.
Because AEC does not measure embedding concentration, Section~\ref{Supp-sec:additional_coupling_diagnostics} of the supplementary material reports a dispersion diagnostic alongside the paired reconstruction comparison.
Full pNMF also has lower scale-averaged AEC in every disk and branching pair, with a small increase in mean reconstruction error for disks and a small decrease for branching data (Section~\ref{Supp-sec:additional_geometry_coupling} of the supplementary material).

\subsection{Conditional Stability and Localization}
\label{subsec:numerical_validation_scale_evolution_stability}
To examine how perturbations propagate across scales, we fix the basis sequence from the reference fit of the noncoplanar circles and study the inverse blocks of the coupled operator.

\begin{figure}[ht]
\centering
\includegraphics[width=\textwidth]{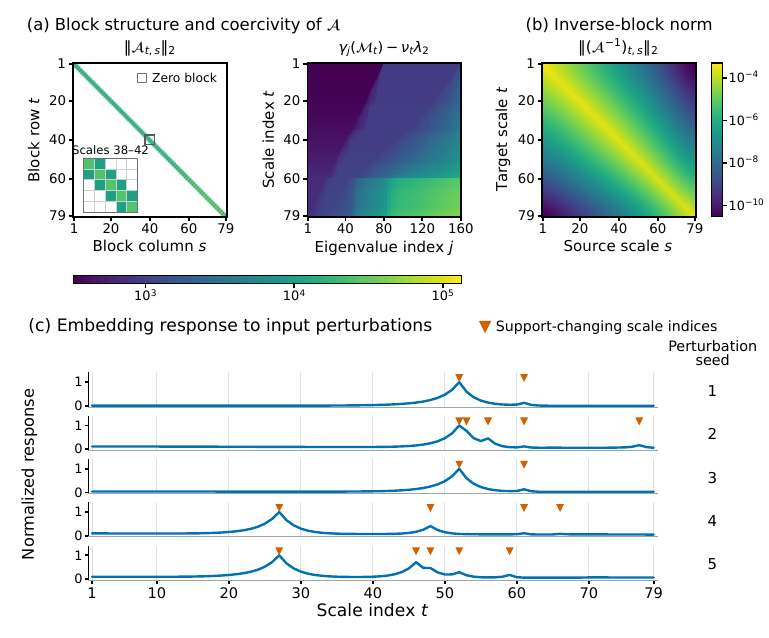}
\caption{Conditional stability and localization for the noncoplanar-circle example with fixed basis matrices.
(a) Block spectral norms $\|\mathcal{A}_{t,s}\|_2$ (left) and ordered values $\gamma_j(\mathcal{M}_t)-\nu_t\lambda_2$ for $j=1,\ldots,160$ (right), on a shared logarithmic color scale.
White cells in the left heatmap denote zero blocks, and the inset enlarges scales $38$--$42$.
The first column of the spectral heatmap gives the block-row coercivity margins.
(b) Inverse-block spectral norms $\|(\mathcal{A}^{-1})_{t,s}\|_2$ over all $79$ source and target scales, on a separate logarithmic color scale.
(c) Normalized embedding responses $\|(\delta\mathbf{h})_t\|_2/\max_j\|(\delta\mathbf{h})_j\|_2$ to five input perturbations, with orange triangles marking the support-changing scale indices $\mathcal{I}_{\mathrm{chg}}$.}
\label{fig:theoretical_validation_scale_evolution_stability}
\end{figure}

In panel~(a) of \Cref{fig:theoretical_validation_scale_evolution_stability}, the smallest block-row coercivity margin, read from the first column of the spectral heatmap, is $325.073$, above the lower bound $\lambda_3=1$ in \Cref{prop:diagonal_dominance}.
Panel~(b) shows larger inverse-block norms near the diagonal, consistent with the attenuation in \Cref{prop:structural_bandwidth}.
The quantitative comparison with the exponential bound for five prescribed source scales is reported in Figure~\ref{Supp-fig:block_operator_coercivity_inverse_decay} of the supplementary material.

Panel~(c) shows the normalized responses to five perturbations that change graph support at some scales and satisfy the numerical criteria for scale correspondence and a common coefficient active set.
For each perturbation, the largest response occurs at a support-changing scale, with smaller responses also visible where the graph support is preserved.
This behavior is consistent with the localized contributions and background term in \Cref{thm:support_changing_localization}.
Since the conditional solutions have active constraints, the relevant comparison uses the common-active-set extension in Section~\ref{Supp-sec:common_free_input_bounds} of the supplementary material.

Section~\ref{Supp-sec:detailed_theoretical_validation} of the supplementary material reports graph diagnostics and tests all three embedding-increment bounds.
It also describes perturbation selection and compares observed responses with the support-preserving and localization bounds.

\subsection{Computational Cost and Scalability}
\label{sec:computational_performance}
The preceding examples use modest sample sizes to keep individual points visible, so we assess computational scaling separately by increasing the number of points in the noncoplanar-circle construction.
We use $n \in \{80,400,800,1200,1600\}$ with $m=n-1$, keeping the ambient dimension $100$ and embedding dimension $d=2$ fixed.
At each size, we generate three data realizations and use the same scales, graphs, initialization, and outer-iteration budget for Full and Decoupled pNMF within each realization.
All runs use double-precision CuPy on one NVIDIA GeForce RTX~4090.

\begin{figure}[ht]
\centering
\includegraphics[width=\textwidth]{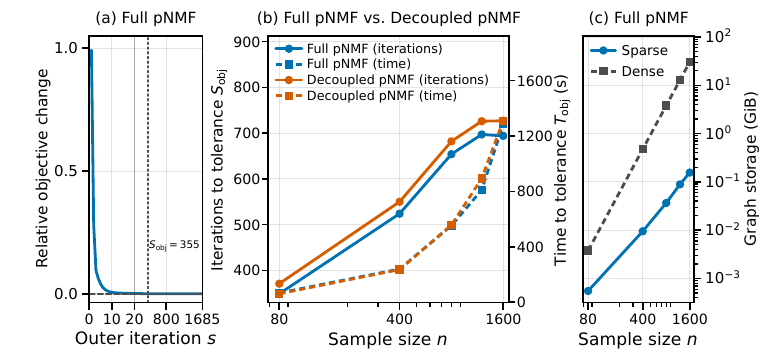}
\caption{Objective progress and graph storage for the noncoplanar-circle datasets.
(a) Relative objective change $|\mathcal{O}^{(s)}-\mathcal{O}^{(s-1)}|/\max\{1,|\mathcal{O}^{(s-1)}|\}$ for the complete Full pNMF run with $n=80$, seed~0, and $1685$ outer iterations.
The horizontal dashed line marks the relative objective-change threshold of $10^{-5}$, and the vertical dotted line marks the first iteration at which it is met ($S_{\mathrm{obj}}=355$), with the iteration axis compressed after iteration~20.
(b) Median iteration $S_{\mathrm{obj}}$ (solid curves with circles, left axis) and time $T_{\mathrm{obj}}$ (dashed curves with squares, right axis) to the objective-change threshold, over three seeds for each variant.
(c) Median storage of the per-scale sparse graph matrices for Full pNMF and the dense $8mn^2$-byte reference.}
\label{fig:computational_performance_scalability_compact}
\end{figure}

To compare optimization progress, we record the first iteration at which the relative objective change is at most $10^{-5}$ as $S_{\mathrm{obj}}$, and its elapsed optimization time as $T_{\mathrm{obj}}$.
For $n=80$, the representative Full pNMF run in \Cref{fig:computational_performance_scalability_compact}(a) first reaches this threshold at iteration $355$, close to the three-seed median of $349$.
Panel~(b) shows that Full pNMF has a lower median iteration count at every tested sample size, while the two variants have comparable median times to the threshold.
Across problem sizes, elapsed time grows much more strongly than the iteration count, which is consistent with the increasing work within each outer iteration as both $n$ and $m$ grow.

Only the three Full runs at $n=80$ meet both the objective-change threshold and the joint projected-KKT tolerance of $10^{-4}$, and all runs at larger sizes exhaust the iteration budget.
For example, at $n=1600$, the Full runs end with joint residuals of $0.331$--$0.369$ after a median optimization time of $3094.3$ seconds.
Section~\ref{Supp-sec:detailed_computational_performance} of the supplementary material gives the detailed stopping diagnostics and actual runtimes, together with a single-scale GNMF reference.

The growing scale sequence also increases storage requirements, since a dense representation uses $8mn^2=8(n-1)n^2$ bytes in double precision and therefore grows cubically with $n$ in this experiment.
Panel~(c) shows how retaining the graph matrices in compressed sparse row (CSR) format reduces this cost for the tested sequences.
At $n=1600$, the per-scale CSR matrices together occupy a median of $0.1560$ GiB, compared with $30.4985$ GiB for the dense representation, a factor of $195.5$.
These values cover the graph matrices themselves, while additional resident graph structures, factor matrices, and solver workspaces contribute to the overall memory requirement.

\section{Application to scRNA-seq Data Clustering}
\label{sec:scrna_clustering}

We assess how well pNMF representations distinguish annotated cell populations and developmental stages in scRNA-seq data.
Because pNMF produces a sequence of embeddings, we first compare the prespecified first-scale output with the baselines on five datasets and then examine temporal grouping across scales on GSE75748time.

\subsection{Datasets and Experimental Setup}
\label{subsec:scrna_setup}
We use five processed scRNA-seq datasets with cell-population labels or developmental annotations.
GSE75748time, the time-course subset of GSE75748~\cite{chu2016single}, contains $758$ cells from six time points during the differentiation of human embryonic stem cells toward definitive endoderm.
GSE94820~\cite{villani2017single} contains $1{,}140$ human blood dendritic cells and monocytes assigned to five annotated groups.
Dendritic\_batch1 and Dendritic\_batch2, from the same study, share no cells and each contain $384$ dendritic cells assigned to four groups.
GSE75140~\cite{camp2015human} contains $734$ cells from fetal neocortex and cerebral organoids, with nine evaluation categories defined by developmental age and sample source.
The matrices contain $15{,}426$--$26{,}593$ genes, and the evaluation protocol and label conventions are described in Section~\ref{Supp-sec:scrna_evaluation_protocol} of the supplementary material.

The comparison includes NMF~\cite{lee1999learning}, GNMF~\cite{cai2010graph}, TNMF, and $k$-TNMF~\cite{hozumi2024analyzing}, so that pNMF is evaluated alongside unregularized factorization, single-graph regularization, and two multigraph constructions.
All methods use NNDSVDA initialization~\cite{boutsidis2008svd} and the label-independent dimension $d=\lceil\sqrt{n}\rceil$.
The parameters and iteration budgets for pNMF and each baseline are given in Section~\ref{Supp-sec:scrna_evaluation_protocol} of the supplementary material.

For pNMF, we evaluate the first-scale embedding $H_1$ obtained by jointly fitting the full sequence.
For evaluation at each scale, we rescale the coefficient rows by the corresponding basis-column norms:
\[
Z_t:=\operatorname{diag}\bigl(\|(W_t)_{:,1}\|_2,\ldots,\|(W_t)_{:,d}\|_2\bigr)H_t, \qquad
1 \leq t \leq m.
\]
For the comparison with the baselines, we use $Z_1$ for pNMF and retain the raw factors for the embedding-path diagnostics.
Each baseline supplies a single coefficient matrix whose rows already incorporate the corresponding basis-column norms.
We apply Lloyd's $k$-means algorithm~\cite{lloyd1982least} to the columns of each evaluation embedding.
Labels specify the cluster count and are used to evaluate the resulting assignments.

We assess the resulting clusterings using four measures that capture complementary aspects of agreement with the annotations.
The adjusted Rand index (ARI) measures pairwise agreement after correction for chance, and normalized mutual information (NMI) measures the association between cluster and annotation labels.
Purity assigns each cluster its most frequent label, whereas Accuracy requires a one-to-one matching of cluster and annotation labels.
Definitions and implementation settings for all four measures are given in Section~\ref{Supp-sec:scrna_evaluation_protocol} of the supplementary material.

\subsection{Clustering Results}
\label{subsec:scrna_performance}

\begin{table}[ht]
\centering
\caption{Clustering results on five scRNA-seq datasets. The pNMF result uses $Z_1$, obtained from the finest positive-scale factor $H_1$, on every dataset. The best and second-best values are bold and underlined, respectively, with ranks determined from unrounded scores.}
\label{tab:clustering_results}
\begin{tabular}{ccccccc}
\hline
Dataset & Metric & NMF & GNMF & TNMF & $k$-TNMF & pNMF\\
\hline
\multirow{4}{*}{GSE75748time} & ARI$\uparrow$ & 0.611 & 0.601 & \underline{0.648} & 0.592 & \textbf{0.747}\\
 & NMI$\uparrow$ & 0.735 & 0.730 & \underline{0.764} & 0.711 & \textbf{0.857}\\
 & Purity$\uparrow$ & 0.745 & 0.720 & 0.715 & \underline{0.751} & \textbf{0.881}\\
 & Accuracy$\uparrow$ & 0.693 & 0.679 & 0.710 & \underline{0.736} & \textbf{0.881}\\
\hline
\multirow{4}{*}{GSE94820} & ARI$\uparrow$ & 0.538 & \underline{0.615} & 0.492 & 0.607 & \textbf{0.754}\\
 & NMI$\uparrow$ & 0.658 & \underline{0.710} & 0.627 & 0.699 & \textbf{0.771}\\
 & Purity$\uparrow$ & 0.742 & \underline{0.788} & 0.709 & 0.787 & \textbf{0.896}\\
 & Accuracy$\uparrow$ & 0.707 & \underline{0.757} & 0.673 & 0.755 & \textbf{0.896}\\
\hline
\multirow{4}{*}{GSE75140} & ARI$\uparrow$ & 0.268 & \textbf{0.308} & \underline{0.307} & 0.203 & 0.277\\
 & NMI$\uparrow$ & 0.436 & 0.433 & \underline{0.442} & 0.342 & \textbf{0.481}\\
 & Purity$\uparrow$ & \underline{0.503} & 0.497 & 0.490 & 0.432 & \textbf{0.548}\\
 & Accuracy$\uparrow$ & 0.402 & 0.409 & \underline{0.428} & 0.362 & \textbf{0.455}\\
\hline
\multirow{4}{*}{Dendritic\_batch1} & ARI$\uparrow$ & 0.567 & 0.591 & 0.295 & \underline{0.618} & \textbf{0.760}\\
 & NMI$\uparrow$ & 0.673 & 0.686 & 0.485 & \underline{0.715} & \textbf{0.793}\\
 & Purity$\uparrow$ & 0.701 & 0.706 & 0.492 & \underline{0.714} & \textbf{0.909}\\
 & Accuracy$\uparrow$ & 0.701 & 0.703 & 0.490 & \underline{0.714} & \textbf{0.909}\\
\hline
\multirow{4}{*}{Dendritic\_batch2} & ARI$\uparrow$ & 0.587 & \textbf{0.838} & 0.583 & 0.820 & \underline{0.832}\\
 & NMI$\uparrow$ & 0.696 & \underline{0.810} & 0.697 & 0.809 & \textbf{0.820}\\
 & Purity$\uparrow$ & 0.711 & \textbf{0.938} & 0.711 & 0.930 & \underline{0.935}\\
 & Accuracy$\uparrow$ & 0.703 & \textbf{0.938} & 0.706 & 0.930 & \underline{0.935}\\
\hline
\end{tabular}
\end{table}

\Cref{tab:clustering_results} shows that the prescribed pNMF output performs favorably across the five datasets, ranking first in $16$ of the $20$ dataset--metric comparisons, second in three, and third in one.
It leads all four metrics on GSE75748time, GSE94820, and Dendritic\_batch1, and achieves the highest NMI on every dataset.
On GSE75140, pNMF ranks first in NMI, Purity, and Accuracy, whereas GNMF and TNMF achieve higher ARI, indicating closer pairwise agreement with the reference labels.
On Dendritic\_batch2, GNMF leads ARI, Purity, and Accuracy, with pNMF second in each.

\begin{figure}[ht]
\centering
\includegraphics[width=\textwidth]{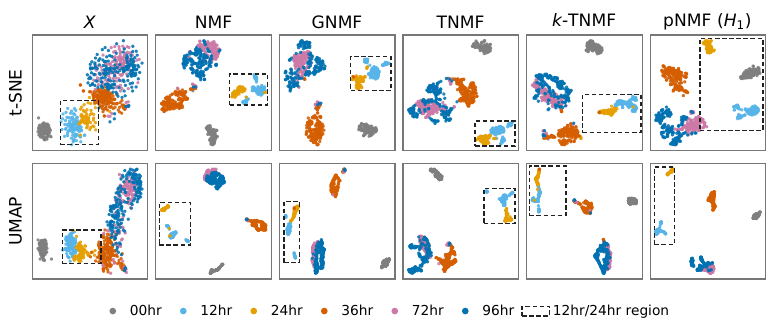}
\caption{Embedding visualization on GSE75748time. Rows show t-SNE and UMAP, while columns show the preprocessed input, NMF, GNMF, TNMF, $k$-TNMF, and pNMF using $Z_1$ (labeled $H_1$ in the panels). Factorized representations use the common evaluation scaling. Colors indicate the six temporal labels, and dashed rectangles highlight the 12hr and 24hr populations using the same panel-wise rule. Colors and rectangles are added after projection, while clustering scores are computed in the evaluation coordinates before projection.}
\label{fig:scrna_embedding_visualization}
\end{figure}

To examine how these differences appear in the learned representations, \Cref{fig:scrna_embedding_visualization} compares the GSE75748time input and the five embeddings using t-SNE~\cite{van2008visualizing} and UMAP~\cite{mcinnes2020umap} under common settings.
In the displayed two-dimensional projections, the pNMF representation shows clearer separation of the 12hr and 24hr groups than the input and baseline representations, consistent with its higher clustering scores evaluated before projection.
Projection settings and visualizations for the remaining datasets are provided in Sections~\ref{Supp-sec:embedding_diagnostics} and~\ref{Supp-sec:additional_h1_visualizations} of the supplementary material.

To examine how coupling affects the embedding sequence, we compare Full pNMF with Decoupled pNMF ($\lambda_2=0$) under a common fixed training budget in Section~\ref{Supp-sec:scrna_cross_scale_coupling} of the supplementary material.
Across the five datasets, coupling reduces the mean adjacent change in the normalized distance geometry of the raw $H_t$ by $45.1\%$--$69.6\%$.
The clustering scores averaged over the full sequence differ by less than $0.0034$ in absolute value, although positive and negative differences at individual scales can cancel in these averages.
The ablation therefore supports smoother geometric evolution under coupling, while its clustering benefit depends on the dataset and the scale at which the sequence is evaluated.

\subsection{Scale-Dependent Clustering Analysis}
\label{subsec:scrna_scale_analysis}
The benchmark in \Cref{tab:clustering_results} uses the evaluation coordinates $Z_1$ at the prespecified first scale.
On GSE75748time, we use the full embedding sequence to ask whether the expression-based clusters favor the same boundaries between sampling times across graph resolutions when the cluster count is held fixed.
For $k=3,4,5$, we enumerate all partitions of the six sampling times into $k$ contiguous groups, giving $10$, $10$, and $5$ candidates, respectively.
At each of the $m=757$ positive scales and for each $k$, we compare the same saved $k$-means assignments on $Z_t$ with every candidate partition of that size.
We measure agreement by Accuracy, the fraction of cells matched under optimal one-to-one assignment.

\begin{figure}[ht]
\centering
\includegraphics[width=\textwidth]{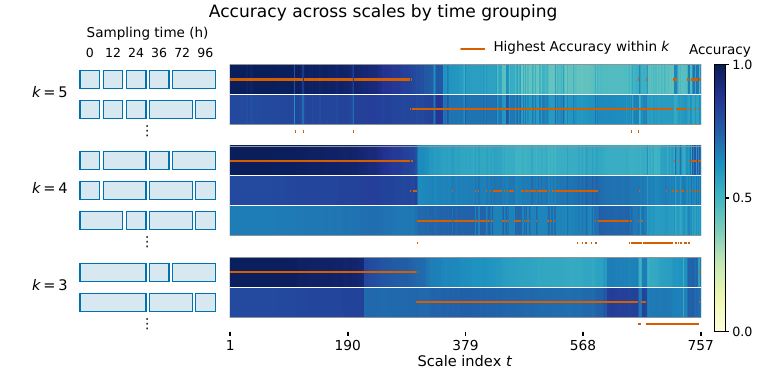}
\caption{Accuracy of time-group partitions across scales on GSE75748time.
Each bar on the left spans the sampling times in one group.
Orange lines mark the highest Accuracy among all candidate partitions at each scale within the same $k$, including ties.
An orange line in an ellipsis row indicates that at least one omitted partition attains that maximum.
For each $k$, we select the partitions that attain the highest Accuracy at the most scales, retaining two each for $k=5$ and $k=3$ and three for $k=4$.}
\label{fig:scale_dependent_clustering}
\end{figure}

The clearest change concerns the grouping of 72~h (\Cref{fig:scale_dependent_clustering}).
For $1 \leq t \leq 104$, the highest-Accuracy partitions at all three cluster counts group 72~h with 96~h and keep 36~h separate.
For $304 \leq t \leq 558$, they group 72~h with 36~h and keep 96~h separate.
The two leading four-group partitions in this interval differ in how they group the earlier times, yet agree on this boundary.
At $t=734$, a three-group partition in the ellipsis row attains the highest Accuracy and groups all three times together.
One cluster contains $171$ of $172$ cells from 36~h and $313$ of $326$ from 72 or 96~h.

For each fixed cluster count, the best-matching partitions move the boundary from before 72~h at early scales to after 72~h at intermediate scales.
By comparing the ordered embeddings learned by pNMF, researchers can identify which samples consistently group together and assess how the choice of connectivity scale shapes their interpretation of group structure.

\section{Conclusion}
\label{sec:conclusion}

We developed pNMF, a factorization framework that learns an ordered embedding sequence indexed by intrinsic connectivity transitions.
A canonical minimal sufficient scale set represents every distinct connected-component count, and its positive scales index graph-regularized embeddings linked by cross-scale coupling.
The analysis bounds adjacent-scale embedding increments and, with the basis sequence fixed, establishes conditional sensitivity and localization estimates.
Under the stated assumptions, exact alternating updates produce convergent objective values, and every accumulation point satisfies the KKT conditions.

Experiments on three synthetic geometries illustrate how the embeddings evolve across scales, and tests on the noncoplanar circles examine the conditional stability estimates.
Coupling reduces the mean geometric change between adjacent embeddings in all three geometries and in fixed-budget comparisons on five scRNA-seq datasets.
On these datasets, the prespecified first-scale embedding yields competitive clustering results.
The ordered embeddings let us track which samples group together across connectivity scales while holding the number of clusters fixed.
On GSE75748time, the best-matching time partitions group 72~h with 96~h at early scales and with 36~h at intermediate scales.

pNMF currently selects scales at which connected components merge.
We will investigate how to extend this $\mathcal{H}_0$ construction to incorporate higher-dimensional topological features.
On the computational side, we will study scale compression to reduce the cost of fitting and storing long embedding sequences and seek convergence guarantees for inexact block updates.
The changes in sample groupings observed across scales motivate learning from the full embedding sequence.
We plan to train sequence models to capture how each sample's embedding changes along this sequence and to use the resulting features for downstream tasks such as clustering and classification.
In single-cell studies of differentiation, relating these multiscale features to gene expression could help identify candidate transitional cell states and the gene programs associated with them.

\section*{Acknowledgments}
The authors used DeepSeek and WorkBuddy to assist with polishing the manuscript, checking mathematical derivations, and debugging code.
The authors assume responsibility for all content.

\bibliographystyle{siamplain}
\bibliography{pNMF-references-simax}
\clearpage
\input{pNMF-supplement-arxiv}
\end{document}

%% file: pNMF-shared-simax.tex
\usepackage{amsfonts}
\usepackage{graphicx}
\usepackage{epstopdf}
\usepackage{algorithmic}
\usepackage{amssymb}
\usepackage{multirow}
\ifpdf
  \DeclareGraphicsExtensions{.eps,.pdf,.png,.jpg}
\else
  \DeclareGraphicsExtensions{.eps}
\fi

\newsiamremark{remark}{Remark}
\newsiamremark{hypothesis}{Hypothesis}
\crefname{hypothesis}{Hypothesis}{Hypotheses}
\newsiamthm{claim}{Claim}
\newsiamremark{fact}{Fact}
\crefname{fact}{Fact}{Facts}

\newsiamthm{property}{Property}

\headers{Persistent NMF}{J. Zhang and L. Li}

\title{Persistent Nonnegative Matrix Factorization via Multi-Scale Graph Regularization
\thanks{Submitted to the editors November~29, 2026.
\funding{This work was funded by National Natural Science Foundation of China projects under Grant No. 92470106 and 12222115.}}}

\author{JICHAO ZHANG\thanks{School of Mathematics and Statistics, Xi'an Jiaotong University, Xi'an, Shaanxi Province, 710049, China (\email{jichaozhang@stu.xjtu.edu.cn}).}
\and LIMIN LI\thanks{Corresponding author. School of Mathematics and Statistics, Xi'an Jiaotong University, Xi'an, Shaanxi Province, 710049, China (\email{liminli@mail.xjtu.edu.cn}).}}

\usepackage{amsopn}

%% file: pNMF-supplement-arxiv.tex
\begingroup
\renewcommand{\siamprelabel}{SM}
\renewcommand{\siampretitle}{Supplementary Materials: }
\setcounter{page}{1}
\setcounter{section}{0}
\setcounter{subsection}{0}
\setcounter{subsubsection}{0}
\setcounter{equation}{0}
\setcounter{theorem}{0}
\setcounter{figure}{0}
\setcounter{table}{0}
\setcounter{algorithm}{0}
\ifdefined\resetlinenumber\resetlinenumber[1]\fi
\headers{Persistent NMF}{J. Zhang and L. Li}

\begin{center}
  \bfseries\MakeUppercase{Supplementary Materials: Persistent Nonnegative Matrix Factorization via Multi-Scale Graph Regularization}\par
  \vspace{1em}
  JICHAO ZHANG AND LIMIN LI
\end{center}
\vspace{1em}

\section{Proofs and Scale Computation for the pNMF Model}
\label{sec:pnmf_modeling_proofs}

We first justify the canonical scale selection and its Kruskal construction, then establish the Laplacian bounds and the existence of a global pNMF minimizer.
All notation and assumptions follow the main-paper modeling section.

\subsection{Canonical Minimal Sufficient Scale Set}
\label{sec:thm_critical_scale_set}
\label{Supp-sec:thm_critical_scale_set}

\begin{proof}[Proof of Theorem~\ref{thm:critical_scale_set}]
The component count is nonincreasing because the VR complexes are nested, and it can change only at the finitely many pairwise distances.
Because the threshold includes all edges at their activation distances, the count is also right-continuous, and its positive jump locations are exactly $\varepsilon_1,\ldots,\varepsilon_m$.

For $0 \leq t < m$, the component count is constant on $[\varepsilon_t,\varepsilon_{t+1})$ and equals $\beta_0(\mathcal{X}_{\varepsilon_t})$, while it remains constant on $[\varepsilon_m,\infty)$.
These intervals cover all nonnegative scales, including the initial interval beginning at $0$, where coincident samples are already connected.
Because components can only merge, the partition into connected components is also constant on each interval.
Since the count strictly decreases at each selected positive scale, the intervals represent all attained counts exactly once.
Thus, $\Lambda_{\mathrm{PH}}$ is sufficient, and removing any element loses its associated count, so the set is minimal.

The left endpoint of each interval is the earliest scale attaining its count, which makes $\Lambda_{\mathrm{PH}}$ canonical.
Conversely, any minimal sufficient set must contain exactly one representative of each attained count, since an absent count violates sufficiency and a repeated representative can be removed.
Canonicality fixes each representative at the corresponding left endpoint, proving uniqueness.
This argument also covers $m=0$, when the count is constant on $[0,\infty)$ and the canonical set is $\{0\}$.

To identify the terminal count, observe that the threshold graph is eventually complete because all pairwise distances are finite.
Its final count is therefore one, and when $m \geq 1$, the last jump at $\varepsilon_m$ is the first scale attaining this count.
One degree-zero class persists indefinitely and contributes no finite death value.
Moreover, all vertices are present at scale $0$, so no new component is born at a positive scale.
The maps on $\mathcal{H}_0$ induced by the inclusions identify merged components, and each decrease equals the number of finite barcode intervals ending at that scale.
Hence the selected jump locations are precisely the distinct positive finite death values, and each jump decreases the count by at least one, giving $m \leq \beta_0(\mathcal{X}_0)-1 \leq n-1$.
If the columns are distinct and every positive death value has multiplicity one, the count decreases from $n$ to one in unit steps, so $m=n-1$.
\end{proof}

\subsection{Kruskal Computation and Tied Distances}
\label{sec:scale_computation}
\label{Supp-sec:scale_computation}

To compute the selected scales, apply Kruskal's algorithm to the complete graph on $[n]$ with edge weights $r_{ij}$.
The algorithm processes edges in nondecreasing order and accepts an edge when its endpoints belong to different union--find components, producing a minimum spanning tree.
First process the zero-distance edges to obtain the components present in $\mathcal{X}_0$.
Then process each group of edges with the same positive distance and record that distance if the component count decreases.

After a distance group $r$ has been processed, the union--find components coincide with those of the threshold graph containing all edges of weight at most $r$.
Indeed, every processed edge either joins two components or has endpoints already joined, so its endpoints lie in the same component afterward.
Conversely, every union uses an edge of the threshold graph, so it cannot join distinct components of that graph.
Thus, a positive group contains an accepted edge exactly when the VR component count decreases, and the distinct positive weights of the accepted tree edges are precisely $\Lambda_{\mathrm{PH}}^+$.
Each accepted edge reduces the count by one, so the number accepted in a group equals its total decrease.
Ties may change the tree edges selected, but they do not change this number or the resulting scale set.

For Euclidean data in $\mathbb{R}^p$, constructing all pairwise distances costs $O(pn^2)$ arithmetic operations, and sorting the $n(n-1)/2$ edges costs $O(n^2\log n)$.
The implementation uses path compression and links roots by vertex index.
Processing the edges stays within the sorting bound, giving a total cost of $O(pn^2+n^2\log n)$ arithmetic and comparison operations and $O(n^2)$ additional storage apart from the input matrix.

\subsection{Piecewise Lipschitz Continuity of the Laplacian}
\label{sec:pro_piecewise_lipschitz_continuity}
\label{Supp-sec:pro_piecewise_lipschitz_continuity}

\begin{proof}[Proof of Proposition~\ref{prop:piecewise_lipschitz_continuity}]
Let $e_i$ be the $i$th standard basis vector of $\mathbb{R}^n$ and set $v_{ij}:=e_i-e_j$.
The Laplacian has the edgewise representation
\[
L(\varepsilon)=\sum_{1 \leq i < j \leq n}A_{ij}(\varepsilon)v_{ij}v_{ij}^{\top}.
\]
Since $\|v_{ij}\|_2^2=2$, each rank-one matrix satisfies $\|v_{ij}v_{ij}^{\top}\|_F=2$.
To bound the variation of its coefficient, define the smooth weight function $w_{ij}(\varepsilon):=\exp(-r_{ij}^2/\varepsilon^\alpha)$ for $\varepsilon>0$.
Its derivative is
\[
w_{ij}'(\varepsilon)=\alpha r_{ij}^2\varepsilon^{-\alpha-1}\exp(-r_{ij}^2/\varepsilon^\alpha).
\]
For every $\varepsilon \geq \delta_1$,
\[
0 \leq w_{ij}'(\varepsilon) \leq \alpha r_{ij}^2\delta_1^{-\alpha-1}.
\]
Zero-distance weights are identically one and have zero derivative, so this estimate includes coincident samples.

If $\varepsilon,\varepsilon'$ belong to the same $I_k$ with $k \geq 1$, the active edge set is fixed and both scales are at least $\delta_1$.
The mean value theorem and the triangle inequality therefore give
\[
\|L(\varepsilon)-L(\varepsilon')\|_F \leq 2\sum_{1 \leq i < j \leq n}|A_{ij}(\varepsilon)-A_{ij}(\varepsilon')| \leq 2\alpha\delta_1^{-\alpha-1}\left(\sum_{1 \leq i < j \leq n}r_{ij}^2\right)|\varepsilon-\varepsilon'|.
\]
Thus, one admissible choice is
\begin{equation}
\label{Supp-eq:fixed_support_constant}
C_{\mathrm{fix}} := 2\alpha\delta_1^{-\alpha-1}\sum_{1 \leq i < j \leq n}r_{ij}^2.
\end{equation}
This constant depends only on $X$ and the fixed exponent $\alpha$, and the estimate applies to every $I_k$ with $k  \geq 1$, including the unbounded interval $I_M$.
On $I_0$, only zero-distance edges are active, so the Laplacian is constant and the same bound holds.
\end{proof}

\subsection{Adjacent-Scale Variation and Spectral Conclusions}
\label{sec:thm_lipschitz_bound}
\label{Supp-sec:thm_lipschitz_bound}

\begin{proof}[Proof of Theorem~\ref{thm:spectral_stability_monotonicity}]
\emph{(1) Laplacian and spectral variation.}
For every $1 \leq t < m$, set $a=\varepsilon_t$ and $b=\varepsilon_{t+1}$, so that $\delta_1 \leq a < b$.
Using the smooth weights from the preceding proof, separate the edges already present at $a$ from those admitted between $a$ and $b$:
\[
L(b)-L(a) = \sum_{\substack{1 \leq i < j \leq n \\ r_{ij} \leq a}}\bigl(w_{ij}(b)-w_{ij}(a)\bigr)v_{ij}v_{ij}^{\top} + \sum_{\substack{1 \leq i < j \leq n \\ a < r_{ij} \leq b}}w_{ij}(b)v_{ij}v_{ij}^{\top}.
\]
For the first sum, the derivative bound used in \Cref{Supp-eq:fixed_support_constant} limits the total Frobenius variation to $C_{\mathrm{fix}}(b-a)$.
For the second, each rank-one term has Frobenius norm $2w_{ij}(b)$, so the triangle inequality gives
\[
\|L(\varepsilon_{t+1})-L(\varepsilon_t)\|_F \leq C_{\mathrm{fix}}(\varepsilon_{t+1}-\varepsilon_t) + 2\sum_{\substack{1 \leq i < j \leq n \\ \varepsilon_t<r_{ij} \leq \varepsilon_{t+1}}}A_{ij}(\varepsilon_{t+1}).
\]
The first term accounts for changing weights on existing edges, while the second accounts for newly activated edges.

When $m \geq 2$, define
\begin{equation}
\label{Supp-eq:adjacent_scale_constant}
C_{\mathrm{adj}} := C_{\mathrm{fix}} + 2\max_{1 \leq t < m}\frac{\displaystyle\sum_{\substack{1 \leq i < j \leq n \\ \varepsilon_t<r_{ij} \leq \varepsilon_{t+1}}}A_{ij}(\varepsilon_{t+1})}{\varepsilon_{t+1}-\varepsilon_t}.
\end{equation}
The finitely many selected gaps are positive, so $C_{\mathrm{adj}}$ is finite and independent of $t$, proving the Laplacian bound.
For $m=1$, there is no adjacent pair and we may take $C_{\mathrm{adj}}=C_{\mathrm{fix}}$.
The eigenvalue bound follows from Weyl's inequality~\cite{horn1985matrix},
\[
|\gamma_\ell(b)-\gamma_\ell(a)| \leq \|L(b)-L(a)\|_2 \leq \|L(b)-L(a)\|_F,
\]
which completes part~(1).

\emph{(2) Spectral monotonicity.}
The smooth weights are nondecreasing because $w_{ij}' \geq 0$, and edges are only added as the scale increases.
Consequently, $A_{ij}(b)-A_{ij}(a) \geq 0$ for every edge, and the same edgewise representation gives
\[
L(b)-L(a) = \sum_{1 \leq i < j \leq n}\bigl(A_{ij}(b)-A_{ij}(a)\bigr)v_{ij}v_{ij}^{\top} \succeq 0.
\]
Here $\succeq$ denotes the Loewner order on symmetric matrices.
The Courant--Fischer principle~\cite{golub2013matrix} therefore yields $\gamma_\ell(a) \leq \gamma_\ell(b)$.

\end{proof}

\paragraph{Dependence on the selected scale gaps}
The jump term in \Cref{Supp-eq:adjacent_scale_constant} explains why the constant can grow when selected scales approach one another.
For example, consider the nonnegative samples $X_{\eta}=[0,1,2+\eta]$ with $0<\eta<1$, whose positive selected scales are $1$ and $1+\eta$.
At the latter scale, the edge between the second and third samples enters with weight $\exp(-(1+\eta)^{2-\alpha})$.
For any fixed $\alpha>0$, this weight tends to $e^{-1}$ as $\eta \downarrow 0$, while the change in the existing edge weight tends to zero, so
\[
L_{X_{\eta}}(1+\eta)-L_{X_{\eta}}(1) \longrightarrow e^{-1}(e_2-e_3)(e_2-e_3)^\top.
\]
The Frobenius norm therefore tends to $2/e$ while the scale gap is $\eta$, so any constant bounding this variation by the gap must diverge as $\eta \downarrow 0$.
The selected-sequence bound applies to each fixed data set but is not uniform as distinct transitions approach one another.

\subsection{Global Objective Coercivity}
\label{sec:prop_factor_coercivity}
\label{Supp-sec:prop_factor_coercivity}

\begin{proof}[Proof of Proposition~\ref{prop:factor_coercivity}]
The reconstruction and cross-scale terms are nonnegative, and $L_t \succeq 0$ makes each geometric term nonnegative.
Discarding these terms therefore gives
\[
\mathcal{O} \geq \lambda_3\sum_{t=1}^m\bigl(\|W_t\|_F^2+\|H_t\|_F^2\bigr).
\]
Since $\lambda_3>0$, the objective tends to infinity whenever the norm of the full factor tuple tends to infinity within the nonnegative feasible set, proving coercivity.

For any finite $\tau$, the sublevel set $\{\mathcal{O} \leq \tau,\ W_t,H_t \geq 0\}$ is bounded by this inequality.
It is also closed because the objective is continuous and the feasible orthant is closed, so it is compact.
In particular, the zero factor tuple is feasible with objective value $\tau_0=m\|X\|_F^2$, and continuity ensures that the corresponding nonempty compact sublevel set contains a minimizer.
Every feasible point outside that set has objective value greater than $\tau_0$, so this minimizer is global.
\end{proof}

\section{Quadratic Structure and Inverse Decay of the Coupled System}
\label{sec:block_operator_proofs}

To justify the operator analysis in Section~\ref{sec:block}, we first derive the conditional quadratic form and its optimality conditions.
We then combine polynomial approximation with block locality to prove inverse decay and extend the argument to a fixed active set.
All notation and parameter assumptions are those of Section~\ref{sec:block}.

\subsection{Conditional Quadratic Form and Optimality Conditions}
\label{sec:block_quadratic_form}
\label{Supp-sec:block_quadratic_form}

With the basis sequence fixed, each reconstruction term expands as
\[
\|X-W_tH_t\|_F^2 = \|X\|_F^2 - 2\operatorname{Tr}(H_t^\top W_t^\top X) + \operatorname{Tr}(H_t^\top W_t^\top W_tH_t).
\]
If $\mathbf{h}_t:=\operatorname{vec}(H_t)$ and $\mathbf{b}_t:=\operatorname{vec}(W_t^\top X)$, the linear and quadratic terms become $-2\mathbf{b}_t^\top\mathbf{h}_t$ and $\mathbf{h}_t^\top[I_n \otimes (W_t^\top W_t)]\mathbf{h}_t$.
For the graph penalty, the same vectorization and the symmetry of $L_t$ give
\[
\operatorname{Tr}(H_tL_tH_t^\top) = \mathbf{h}_t^\top(L_t \otimes I_d)\mathbf{h}_t.
\]
To identify the off-diagonal blocks, we expand the cross-scale term as
\[
\sum_{t=2}^m\|H_t-H_{t-1}\|_F^2 = \sum_{t=2}^m\|\mathbf{h}_t-\mathbf{h}_{t-1}\|_2^2 = \sum_{t=1}^m\nu_t\|\mathbf{h}_t\|_2^2-2\sum_{t=2}^m\mathbf{h}_t^\top\mathbf{h}_{t-1}.
\]
For $m=1$, the cross-scale sums vanish and $\nu_1=0$, so the expansion remains valid.
Multiplying the graph and cross-scale terms by $\lambda_1$ and $\lambda_2$, respectively, and adding the factor penalty gives Equation~\ref{eq:conditional_embedding_quadratic}, with the blocks in Equations~\ref{eq:diagonal_block_M} and~\ref{eq:global_block_operator}.
Therefore,
\[
\nabla\mathcal{O}_H(\mathbf{h})=2(\mathcal{A}\mathbf{h}-\mathbf{b}), \qquad
\nabla^2\mathcal{O}_H=2\mathcal{A}.
\]

To obtain a uniform lower bound for this Hessian, we prove Proposition~\ref{prop:diagonal_dominance}.

\begin{proof}[Proof of Proposition~\ref{prop:diagonal_dominance}]
Since $W_t^\top W_t \succeq 0$ and $L_t \succeq 0$, the diagonal blocks obey $\gamma_{\min}(\mathcal{M}_t) \geq \nu_t\lambda_2+\lambda_3$.
The off-diagonal block norms sum to $\nu_t\lambda_2$, which proves Equation~\ref{eq:block_coercivity}.
For the global bound, write
\[
\mathcal{B}_t := I_n \otimes (W_t^\top W_t) + \lambda_1L_t \otimes I_d \succeq 0.
\]
For any $\mathbf{z}=[\mathbf{z}_1^\top,\ldots,\mathbf{z}_m^\top]^\top$ with $\mathbf{z}_t \in \mathbb{R}^{dn}$, the quadratic form becomes
\[
\mathbf{z}^\top\mathcal{A}\mathbf{z} = \sum_{t=1}^m\mathbf{z}_t^\top(\mathcal{B}_t+\lambda_3I_{dn})\mathbf{z}_t + \lambda_2\sum_{t=2}^m\|\mathbf{z}_t-\mathbf{z}_{t-1}\|_2^2 \geq \lambda_3\sum_{t=1}^m\|\mathbf{z}_t\|_2^2 = \lambda_3\|\mathbf{z}\|_2^2.
\]
This proves Equation~\ref{eq:global_coercivity} and the inverse-norm bound.
\end{proof}

By Proposition~\ref{prop:diagonal_dominance}, the Hessian is bounded below by $2\lambda_3I_{dmn}$, so the objective is $2\lambda_3$-strongly convex.
Its quadratic form also satisfies
\[
\mathcal{O}_H(\mathbf{h}) \geq \lambda_3\|\mathbf{h}\|_2^2 - 2\|\mathbf{b}\|_2\|\mathbf{h}\|_2 + c,
\]
which shows that $\mathcal{O}_H(\mathbf{h}) \to \infty$ as $\|\mathbf{h}\|_2 \to \infty$.
Together with continuity on the nonempty closed orthant, this coercivity gives existence of a minimizer, while strong convexity gives uniqueness.
For this convex problem, optimality is characterized by the variational inequality
\[
(\mathcal{A}\mathbf{h}-\mathbf{b})^\top(\mathbf{v}-\mathbf{h}) \geq 0 \qquad
\text{for every }\mathbf{v} \geq 0.
\]
Writing $\mathbf{r}:=\mathcal{A}\mathbf{h}-\mathbf{b}$, variations in the positive coordinate directions show that $\mathbf{r} \geq 0$.
Whenever $h_i>0$, both signs of a sufficiently small coordinate variation are feasible, so $r_i=0$.
Together, these observations give $h_i r_i=0$ for every coordinate $i$, which is equivalent to $\mathbf{h}^\top\mathbf{r}=0$ under nonnegativity.
Conversely, Equation~\ref{eq:embedding_kkt} gives $\mathbf{r}^\top(\mathbf{v}-\mathbf{h})=\mathbf{r}^\top\mathbf{v} \geq 0$ for every feasible $\mathbf{v}$, so the KKT conditions also suffice for optimality.

\subsection{Block Locality and Exponential Inverse Decay}
\label{sec:prop_structural_bandwidth}
\label{Supp-sec:prop_structural_bandwidth}

\begin{proof}[Proof of Proposition~\ref{prop:structural_bandwidth}]
Set
\[
a:=\gamma_{\min}(\mathcal{A})>0, \qquad
b:=\gamma_{\max}(\mathcal{A}), \qquad
q:=\rho_{\mathcal{A}}.
\]
If $a=b$, symmetry gives $\mathcal{A}=aI_{dmn}$ and $q=0$.
Its off-diagonal blocks vanish, and the diagonal inverse blocks have norm $a^{-1}$, so Equation~\ref{eq:demko_bound} holds with the stated convention.
We may therefore assume $a<b$, in which case
\[
q=\frac{\sqrt{b}-\sqrt{a}}{\sqrt{b}+\sqrt{a}} \in (0,1).
\]

Because $\mathcal{A}$ has block bandwidth at most one, the product of $k$ copies of $\mathcal{A}$ has block bandwidth at most $k$.
Indeed, each nonzero term in its $(t,s)$ block can connect the scale indices through at most $k$ moves between neighboring scales.
It follows that any polynomial of degree at most $k$ evaluated at $\mathcal{A}$ has the same bandwidth bound.

To use this locality, fix $t \neq s$ and let $\ell:=|t-s| \geq 1$.
A polynomial of degree at most $\ell-1$ has a zero $(t,s)$ block when evaluated at $\mathcal{A}$, so its approximation error can bound the corresponding inverse block.
We construct such a polynomial by approximating $x^{-1}$ on $[a,b]$.
Let $T_\ell$ be the Chebyshev polynomial of the first kind, characterized by $T_\ell(\cos\theta)=\cos(\ell\theta)$, and define
\[
\xi:=\frac{a+b}{b-a}, \qquad
R_\ell(x):=\frac{T_\ell((a+b-2x)/(b-a))}{T_\ell(\xi)}, \qquad
p_{\ell-1}(x):=\frac{1-R_\ell(x)}{x}.
\]
Since $R_\ell(0)=1$, the numerator in the last expression is divisible by $x$, so $p_{\ell-1}$ is a polynomial of degree at most $\ell-1$.
For $x \in [a,b]$, the argument of the numerator Chebyshev polynomial lies in $[-1,1]$, which gives $|T_\ell((a+b-2x)/(b-a))| \leq 1$.
The definition of $q$ and the Chebyshev formula for $\xi>1$ give
\[
\xi+\sqrt{\xi^2-1}=q^{-1}, \qquad
T_\ell(\xi)=\frac{q^{-\ell}+q^\ell}{2}.
\]
Combining these identities gives
\[
\max_{x \in [a,b]}\left|\frac{1}{x}-p_{\ell-1}(x)\right| \leq \frac{1}{a}\frac{2q^\ell}{1+q^{2\ell}} \leq \frac{2}{a} q^\ell.
\]

The spectral theorem now yields
\[
\|\mathcal{A}^{-1}-p_{\ell-1}(\mathcal{A})\|_2 \leq \max_{x\in[a,b]}\left|\frac{1}{x}-p_{\ell-1}(x)\right| \leq \frac{2}{a} q^\ell.
\]
Since the $(t,s)$ block of $p_{\ell-1}(\mathcal{A})$ is zero, the corresponding inverse block can be extracted from this error.
Specifically, if $E_t:\mathbb{R}^{dn} \to \mathbb{R}^{dmn}$ inserts a vector into block $t$, then $\|E_t\|_2=1$ and
\[
(\mathcal{A}^{-1})_{t,s}=E_t^\top\bigl(\mathcal{A}^{-1}-p_{\ell-1}(\mathcal{A})\bigr)E_s.
\]
Taking spectral norms proves $\|(\mathcal{A}^{-1})_{t,s}\|_2 \leq (2/a)q^{|t-s|}$ directly, without any factor depending on $dn$.
For $t=s$, compression of the full inverse gives
\[
\|(\mathcal{A}^{-1})_{t,t}\|_2 \leq \|\mathcal{A}^{-1}\|_2=\frac{1}{a} \leq \frac{2}{a} q^0.
\]
Finally, $a \geq \lambda_3$ gives the second inequality in Equation~\ref{eq:demko_bound}.
\end{proof}

The estimate includes the single-scale case $m=1$, for which only the diagonal-block bound is needed.
It also includes the decoupled case $\lambda_2=0$, when both $\mathcal{A}$ and its inverse are block diagonal and every bound for distinct scales has a zero left-hand side.
The prefactor is bounded by $2/\lambda_3$, while a uniform condition-number bound keeps the rate in this estimate uniformly below one.

\subsection{Restriction to a Fixed Active Set}
\label{sec:block_active_set}
\label{Supp-sec:block_active_set}

To apply the operator estimates at a boundary minimizer, fix the active set of zero coordinates and let $F$ be its nonempty complement.
Let $P_F$ insert these free coordinates into the full vector.
Thus, $P_F^\top P_F=I$, and the reduced operator and right-hand side are
\[
\mathcal{A}_F:=P_F^\top\mathcal{A} P_F, \qquad
\mathbf{b}_F:=P_F^\top\mathbf{b}.
\]
At a minimizer with this active set, the coordinates outside $F$ vanish and the KKT conditions impose zero residual on $F$.
Writing $\mathbf{h}=P_F\mathbf{h}_F$ therefore gives the reduced stationarity equation
\begin{equation}
\label{Supp-eq:block_free_system}
\mathcal{A}_F\mathbf{h}_F=\mathbf{b}_F.
\end{equation}
A solution of this equation gives the full constrained minimizer if and only if it is nonnegative and the residual on the omitted coordinates is also nonnegative.
These checks supply the remaining conditions in Equation~\ref{eq:embedding_kkt}.
Among solutions satisfying these checks, the prescribed active set is realized exactly when $\mathbf{h}_F>0$, since a zero entry of $\mathbf{h}_F$ gives an additional active constraint.

For every vector $\mathbf{z}$ of free coordinates, the spectral bounds of the full operator give
\[
a\|\mathbf{z}\|_2^2 \leq (P_F\mathbf{z})^\top\mathcal{A}(P_F\mathbf{z}) = \mathbf{z}^\top\mathcal{A}_F\mathbf{z} \leq b\|\mathbf{z}\|_2^2,
\]
where $a=\gamma_{\min}(\mathcal{A})$ and $b=\gamma_{\max}(\mathcal{A})$.
Thus, $\mathcal{A}_F$ inherits the coercivity bound $\mathcal{A}_F \succeq aI \succeq \lambda_3I$, and its spectrum lies in $[a,b]$.
To retain the scale distances used in the full system, group the free coordinates by their original scale indices and retain empty groups when a scale has no free coordinates.
Although the block sizes may differ, blocks more than one scale apart still vanish, so the polynomial bandwidth argument remains valid.
Applying the argument of the preceding subsection gives, for nonempty scale groups $t$ and $s$,
\[
\| (\mathcal{A}_F^{-1})_{t,s}\|_2 \leq \frac{2}{a}\rho_{\mathcal{A}}^{|t-s|} \leq \frac{2}{\lambda_3}\rho_{\mathcal{A}}^{|t-s|}.
\]
If $F$ is empty, there are no free equations and no inverse blocks to estimate, and $\mathbf{h}=0$ is optimal exactly when $-\mathbf{b} \geq 0$.

Because the omitted coordinates are fixed at zero before solving the reduced equations, the relevant inverse is $\mathcal{A}_F^{-1}$, which generally differs from $P_F^\top\mathcal{A}^{-1}P_F$.
A perturbation calculation with an unchanged active set therefore uses the reduced system \Cref{Supp-eq:block_free_system}.
If the active set changes, the free-coordinate equations must be updated before applying this representation.

\section{Proofs for Embedding Evolution and Conditional Stability}
\label{sec:stability_proofs}

This section completes the analysis in Section~\ref{sec:evolution_stability} by proving the adjacent-scale bounds with active nonnegativity constraints, establishing the conditional embedding perturbation bound, and deriving local Laplacian stability from the ordering of pairwise distances.
We then extend input localization to a common active set and give sufficient conditions for uniform constants.
Throughout, $\Delta$ denotes an adjacent-scale difference and $\delta$ an input perturbation, as in the main paper.

\subsection{Adjacent-Scale Increment Bounds}
\label{sec:thm_embedding_increment_bounds}
\label{Supp-sec:thm_embedding_increment_bounds}

\begin{proof}[Proof of Lemma~\ref{lem:local_embedding_increment}]
Let $S_t$ be one half of the embedding gradient at the KKT point.
The embedding KKT conditions give
\[
H_t \geq 0, \qquad
S_t \geq 0, \qquad
\langle H_t,S_t \rangle_F=0,
\]
where $\langle U,V \rangle_F:=\operatorname{Tr}(U^\top V)$.
With the boundary convention $\Delta H_1=\Delta H_{m+1}=0$, the gradient equations can be written as
\[
(W_t^\top W_t+\lambda_3I_d)H_t+\lambda_1H_tL_t+\lambda_2(\Delta H_t-\Delta H_{t+1})-W_t^\top X=S_t, \qquad
1 \leq t \leq m.
\]
Subtracting the equation at $t-1$ from that at $t$ therefore gives
\begin{equation}
\label{Supp-eq:increment_kkt_difference}
\begin{aligned}
&\bigl(W_t^\top W_t+(2\lambda_2+\lambda_3)I_d\bigr)\Delta H_t+\lambda_1\Delta H_tL_t \\
&\qquad -\lambda_2(\Delta H_{t-1}+\Delta H_{t+1})=\mathcal{R}_t+S_t-S_{t-1}, \qquad
2 \leq t \leq m.
\end{aligned}
\end{equation}
To control the additional constraint terms, complementary slackness and nonnegativity give
\[
\langle \Delta H_t,S_t-S_{t-1} \rangle_F = -\langle H_t,S_{t-1} \rangle_F - \langle H_{t-1},S_t\rangle_F \leq 0.
\]
Thus, taking the Frobenius inner product of \Cref{Supp-eq:increment_kkt_difference} with $\Delta H_t$ produces a nonpositive constraint contribution on the right-hand side.
On the left, the Gram term is bounded below by $\gamma_{\min}(W_t^\top W_t)\|\Delta H_t\|_F^2$, and the graph term is nonnegative because $L_t \succeq 0$.
Cauchy--Schwarz applied to the remaining terms then yields
\[
\mu_t\|\Delta H_t\|_F^2 \leq \|\mathcal{R}_t\|_F\|\Delta H_t\|_F + \lambda_2\|\Delta H_t\|_F\bigl(\|\Delta H_{t-1}\|_F+\|\Delta H_{t+1}\|_F\bigr).
\]
If $\Delta H_t \neq 0$, division by its norm proves Equation~\ref{eq:recursive_inequality}, while the zero-increment case is immediate.
The boundary convention also covers $m=2$, when both neighboring increments in the sole inequality vanish.
\end{proof}

\begin{proof}[Proof of Corollary~\ref{cor:uniform_embedding_increment}]
Set $u_t:=\|\Delta H_t\|_F$ for $2 \leq t \leq m$, with $u_1=u_{m+1}=0$, and choose an index $t_*$ where $u_t$ attains its maximum.
Since neither neighboring increment exceeds $u_{t_*}$, Equation~\ref{eq:recursive_inequality} and $\mu \leq \mu_{t_*}$ imply
\[
\mu u_{t_*} \leq \|\mathcal{R}_{t_*}\|_F + 2\lambda_2u_{t_*}.
\]
Absorbing the last term into the left-hand side and using $\mu-2\lambda_2 \geq \lambda_3$ proves Equation~\ref{eq:uniform_embedding_increment}.
\end{proof}

To relate these bounds to changes in the basis and graph sequence, the Gram difference in Equation~\ref{eq:local_driving_term} satisfies
\[
W_t^\top W_t-W_{t-1}^\top W_{t-1} = W_t^\top\Delta W_t+(\Delta W_t)^\top W_{t-1}.
\]
Submultiplicativity of the matrix norms therefore gives
\[
\|\mathcal{R}_t\|_F \leq \bigl[\|X\|_2 + (\|W_t\|_2+\|W_{t-1}\|_2)\|H_{t-1}\|_F\bigr]\|\Delta W_t\|_F + \lambda_1\|H_{t-1}\|_F\|\Delta L_t\|_F.
\]
On an objective sublevel set $\{\mathcal{O} \leq \tau\}$, Proposition~\ref{prop:factor_coercivity} bounds the coefficient of $\|\Delta W_t\|_F$ by $\|X\|_2+2\tau/\lambda_3$.
The Laplacian estimate in Theorem~\ref{thm:spectral_stability_monotonicity} controls the remaining graph contribution, which includes both weight changes and newly admitted edges.

\subsection{Conditional Embedding Perturbation Bound}
\label{sec:global_embedding_perturbation}
\label{Supp-sec:global_embedding_perturbation}

\begin{proof}[Proof of Proposition~\ref{prop:global_embedding_perturbation}]
Applying the first-order optimality conditions to the conditional quadratic objectives in Equation~\ref{eq:conditional_embedding_quadratic}, with each minimizer as a feasible comparison point for the other, gives
\[
\langle \mathcal{A}\mathbf{h}-\mathbf{b},\delta\mathbf{h} \rangle \geq 0, \qquad
\langle \widetilde{\mathcal{A}} \ \ \widetilde{\mathbf{h}}-\widetilde{\mathbf{b}},\delta\mathbf{h} \rangle \leq 0.
\]
Subtracting yields
\[
\langle \widetilde{\mathcal{A}}\delta\mathbf{h},\delta\mathbf{h} \rangle \leq \langle \delta\mathbf{b}-\delta\mathcal{A}\,\mathbf{h},\delta\mathbf{h} \rangle.
\]
Using $\widetilde{\mathcal{A}} \succeq \lambda_3I$ from Proposition~\ref{prop:diagonal_dominance} and the Cauchy--Schwarz inequality, we obtain
\[
\lambda_3\|\delta\mathbf{h}\|_2^2 \leq \bigl(\|\delta\mathbf{b}\|_2 + \|\delta\mathcal{A}\|_2\|\mathbf{h}\|_2\bigr)\|\delta\mathbf{h}\|_2,
\]
which proves Equation~\ref{eq:global_operator_perturbation_bound}.

If both minimizers are interior, then $\mathcal{A}\mathbf{h}=\mathbf{b}$ and $\widetilde{\mathcal{A}}\widetilde{\mathbf{h}}=\widetilde{\mathbf{b}}$.
Subtracting gives
\[
\widetilde{\mathcal{A}}\delta\mathbf{h}=\delta\mathbf{b}-\delta\mathcal{A}\,\mathbf{h}.
\]
Invertibility of $\widetilde{\mathcal{A}}$ yields Equation~\ref{eq:perturbation_identity}.
\end{proof}

\subsection{Laplacian Stability Under Support Preservation}
\label{sec:lem_input_laplacian_stability}
\label{Supp-sec:lem_input_laplacian_stability}

\begin{proof}[Proof of Lemma~\ref{lem:input_laplacian_stability}]
For two inputs $X$ and $Y$, the reverse triangle inequality gives
\begin{equation}
\label{Supp-eq:distance_input_bound}
|\|y_i-y_j\|_2-\|x_i-x_j\|_2| \leq \|(y_i-x_i)-(y_j-x_j)\|_2 \leq \sqrt{2}\|Y-X\|_F.
\end{equation}
Let $\mathcal{P}:=\{\{i,j\}:1 \leq i < j \leq n\}$ and write $r_e:=\|x_i-x_j\|_2$ for $e=\{i,j\}$.
Because these distances are positive and distinct, the finite set
\[
\{r_e:e\in\mathcal{P}\} \cup \{|r_e-r_f|:e,f \in \mathcal{P},\ e \neq f\}
\]
has a positive minimum $\eta$.
Choose $r_0:=\eta/(4\sqrt{2})$.
For every nonnegative $Y$ with $\|Y-X\|_F<r_0$, each distance changes by less than $\eta/4$, so all distances remain positive and their strict ordering is preserved.
Kruskal's algorithm consequently accepts the same edges in the same order, giving $m=n-1$ positive selected scales with the same representing pairs $\{i_t,j_t\}$.
The ordering of all edges relative to each accepted edge is also unchanged, so
\[
\varepsilon_t(Y)=\|y_{i_t}-y_{j_t}\|_2, \qquad
\mathcal{E}_t(Y)=\mathcal{E}_t(X), \qquad
1 \leq t \leq m.
\]
In particular, \Cref{Supp-eq:distance_input_bound} gives $|\varepsilon_t(Y)-\varepsilon_t(X)| \leq \sqrt{2}\|Y-X\|_F$.

Once these supports are fixed, only the smooth edge weights vary.
Set
\[
a:=\frac{\varepsilon_1(X)}{2}>0, \qquad
R:=\max_{e \in \mathcal{P}}r_e+\sqrt{2}r_0.
\]
The choice of $r_0$ ensures that all selected scales are at least $a$, while all pairwise distances are at most $R$.
For $w(r,s):=\exp(-r^2/s^\alpha)$, the derivatives on $0 \leq r \leq R$ and $s \geq a$ satisfy
\[
|\partial_r w(r,s)| \leq 2Ra^{-\alpha}, \qquad
|\partial_s w(r,s)| \leq \alpha R^2a^{-\alpha-1}.
\]
Applying these bounds to the endpoint distance-scale pairs yields
\[
|A_{ij,t}(Y)-A_{ij,t}(X)| \leq c_w\|Y-X\|_F, \qquad
c_w:=\sqrt{2}\bigl(2Ra^{-\alpha}+\alpha R^2a^{-\alpha-1}\bigr)
\]
for every edge in the common support, where $A_{ij,t}(Y)$ denotes the $(i,j)$ entry of $A_t(Y)$.
Entries outside that support vanish for both inputs.

To pass from weights to Laplacians, let $e_i$ be the $i$-th standard basis vector in $\mathbb{R}^n$ and put $v_{ij}:=e_i-e_j$.
Then
\[
L_t(Y)-L_t(X) = \sum_{\{i,j\} \in \mathcal{E}_t(X)}\bigl(A_{ij,t}(Y)-A_{ij,t}(X)\bigr)v_{ij}v_{ij}^\top.
\]
Since $\|v_{ij}v_{ij}^\top\|_2=2$, the triangle inequality gives
\[
\|L_t(Y)-L_t(X)\|_2 \leq 2|\mathcal{E}_t(X)|c_w\|Y-X\|_F.
\]
Taking $r=r_0$ and
\[
C_L:=2\max_{1 \leq t \leq m}|\mathcal{E}_t(X)|c_w
\]
therefore proves Equation~\ref{eq:input_laplacian_stability} uniformly over the stated neighborhood.
\end{proof}

\subsection{Perturbation Identities on Common Free Coordinates}
\label{sec:common_free_input_bounds}
\label{Supp-sec:common_free_input_bounds}

To extend the interior localization estimate to constrained solutions, suppose the two conditional minimizers have the same active set and let
\[
F:=\{i:h_i>0\}=\{i:\widetilde{h}_i>0\}
\]
denote their common free coordinates.
If $F$ is empty, both minimizers are zero and every embedding perturbation bound is immediate.
Otherwise, let $P_F$ insert the free coordinates into the full stacked vector and define
\[
\mathcal{A}_F:=P_F^\top\mathcal{A} P_F, \qquad
\widetilde{\mathcal{A}}_F:=P_F^\top\widetilde{\mathcal{A}}P_F.
\]
Both solutions vanish outside $F$, while their KKT residuals vanish on $F$, so
\[
\mathcal{A}_F\mathbf{h}_F=P_F^\top\mathbf{b}, \qquad
\widetilde{\mathcal{A}}_F\widetilde{\mathbf{h}}_F=P_F^\top\widetilde{\mathbf{b}}.
\]
Here $\mathbf{h}_F:=P_F^\top\mathbf{h}$ and $\widetilde{\mathbf{h}}_F:=P_F^\top\widetilde{\mathbf{h}}$.
Subtracting these equations and reinserting the free coordinates gives
\begin{equation}
\label{Supp-eq:common_free_perturbation}
\delta\mathbf{h} = P_F\widetilde{\mathcal{A}}_F^{-1}P_F^\top\mathbf{g}, \qquad
\mathbf{g} = \delta\mathbf{b}-\delta\mathcal{A}\,\mathbf{h}.
\end{equation}
The reduced operator satisfies $\widetilde{\mathcal{A}}_F \succeq \lambda_3I$, so this identity also implies the global norm bound.
The variational-inequality proof of Proposition~\ref{prop:global_embedding_perturbation} establishes that bound even when the active sets differ.

For localization, group the free coordinates by their original scale indices and retain empty groups when necessary.
The reduced operator has block bandwidth at most one, and its eigenvalues lie in $[\gamma_{\min}(\widetilde{\mathcal{A}}),\gamma_{\max}(\widetilde{\mathcal{A}})]$.
The block polynomial argument in \Cref{Supp-sec:block_active_set} therefore gives
\[
\bigl\|(\widetilde{\mathcal{A}}_F^{-1})_{t,k}\bigr\|_2 \leq \frac{2}{\lambda_3}\overline{\rho}^{|t-k|}
\]
for nonempty free groups whenever $\kappa_2(\widetilde{\mathcal{A}}) \leq \overline{\kappa}$.
Coordinate projection cannot increase $\|\mathbf{g}_k\|_2$, and reinsertion preserves the norm of each free block.
Consequently, \Cref{Supp-eq:common_free_perturbation} yields
\[
\|(\delta\mathbf{h})_t\|_2 \leq \frac{2}{\lambda_3}\sum_{k=1}^m
\overline{\rho}^{|t-k|}\|\mathbf{g}_k\|_2.
\]
This also holds at a scale with no free coordinates because its embedding change is zero.
Splitting the sum over $\mathcal{I}_{\mathrm{chg}}$ and its complement, and applying the source bound in Theorem~\ref{thm:support_changing_localization}, proves Equation~\ref{eq:support_changing_localization} under the common-active-set assumption.

\subsection{Uniform Constants for Support-Changing Localization}
\label{sec:thm_support_changing_localization}
\label{Supp-sec:thm_support_changing_localization}

The localization estimate requires a uniform condition-number bound and a uniform source bound at scales whose graph supports are unchanged.
The former follows from the fixed basis sequence and the bounded graph weights.
Since every edge weight lies in $[0,1]$, each Laplacian satisfies
\[
0 \preceq L_t(Y) \preceq \sum_{i<j}(e_i-e_j)(e_i-e_j)^\top = nI_n-\mathbf{1}\mathbf{1}^\top \preceq nI_n,
\]
where $\mathbf{1} \in \mathbb{R}^n$ is the all-ones vector.
Also, the path Laplacian has norm at most four because $\sum_{t=2}^m\|z_t-z_{t-1}\|_2^2 \leq 4\sum_{t=1}^m\|z_t\|_2^2$.
Combining these estimates with the block structure of $\mathcal{A}(Y)$ gives
\[
\gamma_{\max}(\mathcal{A}(Y)) \leq \max_t\|W_t\|_2^2+\lambda_1n+4\lambda_2+\lambda_3.
\]
Together with the lower bound $\lambda_3$, this permits the uniform choice
\[
\overline{\kappa}:=\frac{\max_t\|W_t\|_2^2+\lambda_1n+4\lambda_2+\lambda_3}{\lambda_3}>1
\]
over the admissible inputs with the fixed basis sequence.

To control the background source, suppose the scales outside $\mathcal{I}_{\mathrm{chg}}$ satisfy $\|\delta L_k\|_2 \leq C_L\|\delta X\|_F$ uniformly over the perturbation class.
The fixed-basis expression for $\mathbf{g}_k$ then gives
\[
\|\mathbf{g}_k\|_2 \leq \bigl(\|W_k\|_2+\lambda_1C_L\|H_k\|_F\bigr)\|\delta X\|_F, \qquad
k \notin \mathcal{I}_{\mathrm{chg}}.
\]
Hence $C_{\mathrm{g}}=\max_k\|W_k\|_2+\lambda_1C_L\max_k\|H_k\|_F$
is admissible in that class.
The derivative argument in the proof of Lemma~\ref{lem:input_laplacian_stability} applies separately at each unchanged-support scale.
It gives a uniform $C_L$ when pairwise distances are uniformly bounded, the selected scales stay bounded away from zero, and their changes are Lipschitz in the input.
Fixed representing pairs provide the last property even when graph supports change at other scales.
These conditions, in addition to unchanged support, control the background term.
The source contributions at changed supports remain explicit in Equation~\ref{eq:support_changing_localization}.

\section{Block Solvers and Convergence Proofs for pNMF}
\label{sec:optimization_details}

The two-block formulation in Section~\ref{sec:optimization} separates the basis rows but retains the coupling between embeddings.
This section describes the common quadratic solver and its stopping rules, then proves the accumulation-point result for exact outer updates and derives the computational costs.
Throughout, the data, graph matrices, and regularization parameters are fixed, with the assumptions of Section~\ref{sec:optimization}.

\subsection{Projected Newton Solution of the Block Subproblems}
\label{sec:projected_newton_details}
\label{Supp-sec:projected_newton_details}

\paragraph{Quadratic form and optimality residual}
Using the notation of Equation~\ref{eq:commonsolver}, either block subproblem has the form
\[
\min_{z \geq 0}q(z):=z^\top Cz-2c^\top z+q_0, \qquad
C=C^\top \succeq \lambda_3I,
\]
where $q_0$ is independent of $z$.
For a row of $W_t$, $C=H_tH_t^\top+\lambda_3I_d$ and $c=H_tx$, where $x \in \mathbb{R}^n$ is the transpose of the corresponding row of $X$.
The implementation groups all feature rows from a batch of scales into one solve, whose quadratic matrix is the block diagonal collection of the row matrices.
For the joint embedding problem, $C=\mathcal{A}$ and $c=\mathbf{b}$.
In both cases, write $B:=2C$ for the Hessian, so $B \succeq 2\lambda_3I$ and every nonempty principal submatrix is positive definite.
The objective is strongly convex and coercive, which gives a unique minimizer on the closed nonnegative orthant.

To measure constrained optimality, let $g(z):=Bz-2c=2(Cz-c)$ and define
\[
G_q(z):=z-[z-g(z)]_+.
\]
For $z \geq 0$, $G_q(z)=0$ is equivalent to $g_i(z)=0$ whenever $z_i>0$ and $g_i(z) \geq 0$ whenever $z_i=0$.
These are precisely the KKT conditions.
The code stores the half-gradient $Cz-c$ and cancels the common factor two in the Newton systems, but evaluates $G_q$ using the full gradient.

\paragraph{Free-coordinate direction and PCG}
The main paper gives the exact-zero working-set description.
In the numerical implementation, let $\epsilon_{\mathrm{mach}}$ denote the machine precision of the working dtype and define
\begin{equation}
\label{Supp-eq:binding_free_sets}
\begin{aligned}
\mathcal{J}_r &:= \{i:z_i^{(r)} \leq a\max\{1,\|z^{(r)}\|_\infty\},\ g_i(z^{(r)})>0\}, \\
\mathcal{F}_r &:= \{1,\ldots,N_z\} \setminus \mathcal{J}_r,
\end{aligned}
\end{equation}
where $N_z$ is the number of variables in the current solve.
The default is $a=\max\{100\epsilon_{\mathrm{mach}},10^{-12}\}$ unless an active-set tolerance is supplied.
The sets are recomputed at each inner iteration, and coordinates in $\mathcal{J}_r$ retain their current values during the step.
A zero coordinate with negative gradient remains free, and the strict gradient test also leaves zero-gradient coordinates free.

Writing $\mathcal{J}=\mathcal{J}_r$ and $\mathcal{F}=\mathcal{F}_r$, the exact reduced Newton direction solves
\begin{equation}
\label{Supp-eq:projected_Newton_direction}
p_{\mathcal{J}}=0, \qquad
B_{\mathcal{F},\mathcal{F}}p_{\mathcal{F}}=-g_{\mathcal{F}}(z^{(r)}).
\end{equation}
If $g_{\mathcal{F}} \neq 0$, positive definiteness gives a unique nonzero solution and
\[
g(z^{(r)})^\top p = -p_{\mathcal{F}}^\top B_{\mathcal{F},\mathcal{F}}p_{\mathcal{F}} < 0.
\]
If the free right-hand side vanishes, the reduced direction is zero.
For the exact-zero rule this implies the complete KKT conditions, but under \Cref{Supp-eq:binding_free_sets} small positive binding coordinates can still have positive gradient.
The numerical solver therefore uses the full projected residual, rather than the reduced right-hand side, to declare inner convergence.

PCG uses a symmetric positive definite preconditioner to approximate \Cref{Supp-eq:projected_Newton_direction}.
For $b_{\mathcal{F}}:=-g_{\mathcal{F}}(z^{(r)}) \neq 0$, start from $u^{(0)}=0$ and use the unpreconditioned residual
\[
r_{\mathrm{PCG}}^{(j)}:=b_{\mathcal{F}}-B_{\mathcal{F},\mathcal{F}}u^{(j)}.
\]
The stopping rule is
\begin{equation}
\label{Supp-eq:PCG_stopping_rule}
\frac{\|r_{\mathrm{PCG}}^{(j)}\|_2}{\|r_{\mathrm{PCG}}^{(0)}\|_2} \leq \mathrm{tol}_{\mathrm{PCG}} \quad
\text{or} \quad
j \geq J_{\max},
\end{equation}
where $0<\mathrm{tol}_{\mathrm{PCG}}<1$ and $J_{\max} \geq 1$.
The residual is checked at the configured interval and at the iteration limit.
Its norm equals the Frobenius norm of the corresponding masked matrix or tensor, including all rows in a basis batch.
Cancelling the factor two in the system leaves this relative test unchanged.
No PCG solve is needed for a zero right-hand side.
The shared solver defaults to $\mathrm{tol}_{\mathrm{PCG}}=10^{-3}$ and checks every two PCG iterations, with iteration limits specified by the experimental configuration.

Extend the computed direction by zeros on $\mathcal{J}$.
If it is nonfinite or fails $g^\top p<0$, use the inverse-preconditioned gradient direction
\[
p_{\mathcal{J}}=0, \qquad
p_{\mathcal{F}}=-\mathcal{Z}_{\mathcal{F}}g_{\mathcal{F}},
\]
where $\mathcal{Z}_{\mathcal{F}}$ is the positive definite inverse-preconditioner action defined below.
For a nonzero free gradient, this gives
$g^\top p=-g_{\mathcal{F}}^\top\mathcal{Z}_{\mathcal{F}}g_{\mathcal{F}}<0$ in exact arithmetic.
The same fallback is used for one retry if the first line search fails.

\paragraph{Projected line search}
For $\beta_{\mathrm{ls}},c_{\mathrm{ls}} \in (0,1)$, form the trial points
\[
z(\xi):=[z^{(r)}+\xi p]_+, \qquad
\xi=\beta_{\mathrm{ls}}^\ell, \qquad
\ell=0,1,\ldots.
\]
The exact-arithmetic Armijo condition is
\begin{equation}
\label{Supp-eq:projected_Armijo_condition}
q(z(\xi)) \leq q(z^{(r)}) + c_{\mathrm{ls}}g(z^{(r)})^\top(z(\xi)-z^{(r)}).
\end{equation}

To verify that this test succeeds if backtracking continues with $g^\top p<0$, abbreviate $z=z^{(r)}$ and $g=g(z)$.
The initial direction of the projected path is
\[
v_i:=\begin{cases}
p_i,&z_i>0, \\
\max\{p_i,0\},&z_i=0.
\end{cases}
\]
For sufficiently small $\xi>0$, $z(\xi)=z+\xi v$.
At every zero coordinate where projection removes a negative $p_i$, one has $g_i \leq 0$, because zero coordinates with positive gradient are binding and have $p_i=0$.
Hence
\[
g^\top v = g^\top p - \sum_{\{i:z_i=0,\ p_i<0\}}g_ip_i \leq g^\top p < 0.
\]
This argument applies to both the safeguarded PCG direction and the fallback.
Since $q$ is quadratic,
\[
q(z+\xi v)-q(z) = \xi g^\top v+\tfrac12\xi^2v^\top Bv,
\]
which proves \Cref{Supp-eq:projected_Armijo_condition} for sufficiently small positive $\xi$.
Moreover, for an accepted nonzero displacement $d_z:=z(\xi)-z$, strong convexity gives
\[
(1-c_{\mathrm{ls}})g^\top d_z + \lambda_3\|d_z\|_2^2 \leq 0.
\]
Thus $g^\top d_z<0$ and the exact-arithmetic objective decreases.

The floating-point implementation evaluates the quadratic difference $\Delta q=g^\top d_z+\tfrac12d_z^\top Bd_z$ directly and accepts a trial when
\[
\Delta q \leq c_{\mathrm{ls}}g^\top d_z + 64\epsilon_{\mathrm{mach}}\max\{1,|c_{\mathrm{ls}}g^\top d_z|\}.
\]
This allowance can admit a small numerical increase, so the exact decrease argument above is not a strict monotonicity guarantee for the implementation.
The shared defaults are $c_{\mathrm{ls}}=10^{-4}$ and $\beta_{\mathrm{ls}}=1/2$, with $\ell=0,\ldots,20$, giving at most 21 trials per line-search attempt.
After a failed first attempt, the solver retries with the fallback direction.
If both attempts fail, or an accepted trial has zero displacement, it retains the current iterate and returns the status \texttt{line search stalled}.
This status is distinct from reaching the Newton limit or satisfying the inner residual target.

\paragraph{Reusable preconditioners}
For a row-wise basis problem, write $G_t=H_tH_t^\top+\lambda_3I_d$ and use
\[
\mathcal{P}_{W,t}:=2G_t.
\]
The implementation computes $G_t^{-1}$ once per scale at each outer iteration and reuses it for all feature rows and changing free-coordinate masks.
For the $B=2C$ convention, the inverse action is $\mathcal{P}_{W,t}^{-1}=\tfrac12G_t^{-1}$.
The preconditioner for a basis batch is the corresponding block diagonal collection of row preconditioners.
For the embedding problem, use the sample-block approximation
\[
\mathcal{P}_H := 2\operatorname{blkdiag}_{t=1}^m\left[I_n \otimes (W_t^\top W_t) + \lambda_1D_t \otimes I_d + (\nu_t\lambda_2+\lambda_3)I_{dn}\right],
\]
where $D_t$ is the graph degree matrix.
This matrix retains the Gram coupling within each sample and the degree terms, while omitting off-diagonal graph and cross-scale couplings.
It is positive definite because every diagonal shift includes $\lambda_3>0$.

The sample dependence enters only through a scalar shift, so one eigendecomposition per scale suffices:
\[
W_t^\top W_t = U_t\operatorname{diag}(\sigma_{t,1},\ldots,\sigma_{t,d})U_t^\top, \qquad
U_t^\top U_t=I_d.
\]
The inverse sample-block action on $v \in \mathbb{R}^d$ is
\[
\tfrac12U_t\operatorname{diag}\left(\frac1{\sigma_{t,1}+a_{t,i}},\ldots,\frac1{\sigma_{t,d}+a_{t,i}}\right)U_t^\top v, \qquad
a_{t,i}:=\lambda_1(D_t)_{ii}+\nu_t\lambda_2+\lambda_3.
\]
Each decomposition costs $O(d^3)$, and each sample application costs $O(d^2)$.
In floating-point computation, negative computed Gram eigenvalues are clipped to zero before adding the positive shifts.

To reuse a preconditioner under changing masks, let $E_{\mathcal{F}}$ insert free coordinates into the full variable space.
The reduced Hessian action is $u \mapsto E_{\mathcal{F}}^\top BE_{\mathcal{F}}u$, and the inverse-preconditioner action is
\[
\mathcal{Z}_{\mathcal{F}}:=E_{\mathcal{F}}^\top\mathcal{P}^{-1}E_{\mathcal{F}},
\]
with the appropriate row, batch, or embedding preconditioner $\mathcal{P}$.
For $u \neq 0$,
\[
u^\top\mathcal{Z}_{\mathcal{F}}u = (E_{\mathcal{F}}u)^\top\mathcal{P}^{-1}(E_{\mathcal{F}}u) > 0.
\]
Thus, zero extension, a full inverse-preconditioner application, and restriction give a positive definite PCG preconditioner without refactorization when the mask changes.
This compressed inverse generally differs from the inverse of the corresponding principal submatrix of $\mathcal{P}$.
The code applies the $C$-based inverse preconditioners to the half-gradient systems, which cancels the factors of two above.
The full embedding Hessian and its Kronecker representation are never assembled.

\subsection{Outer Iteration and Stopping Criteria}
\label{sec:outer_iteration_stopping}
\label{Supp-sec:outer_iteration_stopping}

Each inner solve starts from the current factors and returns its last accepted feasible iterate, or the starting point if no nonzero step is accepted.
At outer iteration $l$, each basis-batch solve and the joint embedding solve target
\begin{equation}
\label{Supp-eq:inner_projected_residual_test}
\frac{\|G_q(z)\|_2}{\max\{1,\|z\|_2\}} \leq \tau_{\mathrm{in}}^{(l)}.
\end{equation}
Here $z$ stacks every variable in the current basis batch or the complete embedding sequence.
For a basis batch, both norms aggregate all its scales and feature rows, so this test does not impose the same relative tolerance on every row separately.
It is distinct from the relative linear-system residual in \Cref{Supp-eq:PCG_stopping_rule}.
The shared solver uses
\[
\tau_{\mathrm{in}}^{(l)} = \max\{\tau_{\min},\tau_0\theta^l\}, \qquad
0 < \tau_{\min} \leq \tau_0, \quad 0 < \theta \leq 1,
\]
with the initial target, decay factor, floor, and scale-batch size specified in the experimental protocol.
Each inner call returns its attained residual and a status distinguishing tolerance satisfaction, the Newton iteration limit, and a stalled line search.
Reaching the PCG limit instead returns an approximate direction for the descent safeguard and line search.

Each inner residual measures optimality of one basis batch or of the joint embedding block, with the other variables fixed.
After the embedding update, we recompute the basis gradient and recover the embedding gradient by doubling the half-gradient returned by the embedding solver. We use these gradients to evaluate both projected residuals at the updated factor pair.
Applying $[\cdot]_+$ componentwise to factor tuples, define
\begin{align*}
\mathcal{G}_W(\boldsymbol{W},\boldsymbol{H}) &:= \boldsymbol{W}-[\boldsymbol{W}-\nabla_{\boldsymbol{W}}\mathcal{O}(\boldsymbol{W},\boldsymbol{H})]_+, \\
\mathcal{G}_H(\boldsymbol{W},\boldsymbol{H}) &:= \boldsymbol{H}-[\boldsymbol{H}-\nabla_{\boldsymbol{H}}\mathcal{O}(\boldsymbol{W},\boldsymbol{H})]_+.
\end{align*}
At a feasible pair, these mappings vanish simultaneously exactly when Equation~\ref{eq:pNMF_KKT} holds.
Using the product-space norms from the main paper, the joint normalized residual is
\begin{equation}
\label{Supp-eq:joint_projected_KKT_residual}
\mathrm{res}_{\mathrm{KKT}}(\boldsymbol{W},\boldsymbol{H}) := \frac{\bigl(\|\mathcal{G}_W(\boldsymbol{W},\boldsymbol{H})\|_F^2+\|\mathcal{G}_H(\boldsymbol{W},\boldsymbol{H})\|_F^2\bigr)^{1/2}}{\max\bigl\{1,(\|\boldsymbol{W}\|_F^2+\|\boldsymbol{H}\|_F^2)^{1/2}\bigr\}}.
\end{equation}
For $\mathcal{O}^{(l)}:=\mathcal{O}(\boldsymbol{W}^{(l)},\boldsymbol{H}^{(l)})$, also define
\[
\mathrm{rel}_{\mathcal{O}}^{(l+1)}  :=  \frac{|\mathcal{O}^{(l+1)}-\mathcal{O}^{(l)}|}{\max\{1,|\mathcal{O}^{(l)}|\}}.
\]
Write $\mathrm{res}_{\mathrm{KKT}}^{(l+1)}$ for \Cref{Supp-eq:joint_projected_KKT_residual} evaluated at $(\boldsymbol{W}^{(l+1)},\boldsymbol{H}^{(l+1)})$.

\begin{algorithm}[htbp]
\caption{Two-Block Alternating Optimization for pNMF}
\label{Supp-alg:two_block_alternating_optimizationSupp-alg:two_block_alternating_optimization}
\textbf{Input}: $X \in \mathbb{R}_+^{p \times n}$ and $\{L_t\}_{t=1}^m$, with $m \geq 1$. \\
\textbf{Parameters}: $\lambda_1,\lambda_3>0$, $\lambda_2 \geq 0$, inner target schedule, outer tolerances and budget, stopping mode, check frequency, and inner solver parameters and limits. \\
\textbf{Output}: factor sequences, outer termination status, and recorded inner and outer diagnostics.

\begin{algorithmic}[1]
\STATE Initialize $\boldsymbol{W}^{(0)} \geq 0$, $\boldsymbol{H}^{(0)} \geq 0$, and evaluate $\mathcal{O}^{(0)}$.
\FOR{$l=0,\ldots,l_{\max}-1$}
\STATE Set the inner target $\tau_{\mathrm{in}}^{(l)}$.
\STATE Approximately solve the independent basis-row problems in Equation~\ref{eq:W_row_quadratic} in scale batches, starting from $\boldsymbol{W}^{(l)}$ and targeting \Cref{Supp-eq:inner_projected_residual_test}, to obtain $\boldsymbol{W}^{(l+1)}$.
\STATE Approximately solve Equation~\ref{eq:H_joint_qp} jointly, starting from $\boldsymbol{H}^{(l)}$ and using the same target, to obtain $\boldsymbol{H}^{(l+1)}$.
\STATE Evaluate $\mathcal{O}^{(l+1)}$ and $\mathrm{rel}_{\mathcal{O}}^{(l+1)}$, check the joint residual when scheduled or needed for joint stopping, and record the diagnostics.
\IF{\Cref{Supp-eq:numerical_monotonicity_guard} holds}
\STATE \textbf{return} the updated pair with status ``numerical monotonicity violation''.
\ENDIF
\IF{joint-test mode is selected and both outer tolerances hold}
\STATE \textbf{return} the updated pair with status ``joint tests satisfied''.
\ENDIF
\IF{objective-only mode is selected and $\mathrm{rel}_{\mathcal{O}}^{(l+1)} \leq \mathrm{tol}_{\mathcal{O}}$}
\STATE \textbf{return} the updated pair with status ``objective tolerance reached''.
\ENDIF
\ENDFOR
\STATE \textbf{return} the final pair with status ``outer budget reached''.
\end{algorithmic}
\end{algorithm}

\Cref{Supp-alg:two_block_alternating_optimizationSupp-alg:two_block_alternating_optimization} uses positive tolerances $\mathrm{tol}_{\mathcal{O}}$ and $\mathrm{tol}_{\mathrm{KKT}}$ and an outer budget $l_{\max}$.
In joint-test mode, early stopping requires both tolerances at the same updated pair.
An optional objective-only mode stops when $\mathrm{rel}_{\mathcal{O}}^{(l+1)}  \leq  \mathrm{tol}_{\mathcal{O}}$ and reports that status without asserting joint KKT attainment.
Fixed-budget mode disables both early stopping tests, but retains the numerical safeguard below.
The joint residual is evaluated at configured checkpoints and whenever a candidate joint stopping test is triggered.
With periodic KKT checks enabled, it is also evaluated at the last budgeted iteration.
The coupling ablations, scalability tests, and single-cell runs use an outer KKT check interval of one iteration.

After both block updates and the objective evaluation, the numerical monotonicity safeguard triggers if
\begin{equation}
\label{Supp-eq:numerical_monotonicity_guard}
\mathcal{O}^{(l+1)}>\mathcal{O}^{(l)} + 256\epsilon_{\mathrm{mach}}\max\{1,|\mathcal{O}^{(l)}|\}.
\end{equation}
The implementation records this iteration and returns $\boldsymbol{W}^{(l+1)}$ and $\boldsymbol{H}^{(l+1)}$ with a numerical monotonicity violation status.
This check precedes the convergence tests and is distinct from inner line-search failure, objective-only stopping, joint-test satisfaction, and exhaustion of the outer budget.

In exact arithmetic, returning only steps accepted by \Cref{Supp-eq:projected_Armijo_condition} preserves objective nonincrease even when an inner budget is exhausted.
The floating-point allowances and safeguard specify the numerical stopping behavior, but do not make the updates exact or guarantee vanishing inner errors.
The convergence proof below therefore retains the exact-minimization hypothesis of Equation~\ref{eq:block_update}.

\subsection{Proof of KKT Stationarity of Accumulation Points}
\label{sec:thm_subproblem_kkt}
\label{Supp-sec:thm_subproblem_kkt}

\begin{proof}[Proof of Theorem~\ref{thm:subproblem_kkt}]
By Theorem~\ref{thm:subproblem_nonincreasing} and Proposition~\ref{prop:factor_coercivity}, every exact outer iterate satisfies
\[
\lambda_3\bigl(\|\boldsymbol{W}^{(l)}\|_F^2 + \|\boldsymbol{H}^{(l)}\|_F^2\bigr) \leq \mathcal{O}^{(l)} \leq \mathcal{O}^{(0)}.
\]
The sequence is therefore bounded and remains in the closed nonnegative orthant.
It has an accumulation point, and for any such point there is a subsequence with
\[
(\boldsymbol{W}^{(l_j)},\boldsymbol{H}^{(l_j)}) \longrightarrow (\boldsymbol{W}^*,\boldsymbol{H}^*).
\]
By Equation~\ref{eq:solution-limit}, both successive differences tend to zero, so the shifted subsequence also satisfies
\[
\boldsymbol{W}^{(l_j+1)} \longrightarrow \boldsymbol{W}^*, \qquad
\boldsymbol{H}^{(l_j+1)} \longrightarrow \boldsymbol{H}^*.
\]
In particular, the intermediate pairs $(\boldsymbol{W}^{(l_j+1)},\boldsymbol{H}^{(l_j)})$
converge to the same limit.

For every fixed nonnegative test tuple $\boldsymbol{V}$, exact minimization of the basis block gives
\[
\left\langle\nabla_{\boldsymbol{W}}\mathcal{O}(\boldsymbol{W}^{(l_j+1)},\boldsymbol{H}^{(l_j)}),\boldsymbol{V}-\boldsymbol{W}^{(l_j+1)}\right\rangle \geq 0.
\]
Similarly, for every fixed nonnegative test tuple $\boldsymbol{Y}$, exact minimization of the embedding block gives
\[
\left\langle\nabla_{\boldsymbol{H}}\mathcal{O}(\boldsymbol{W}^{(l_j+1)},\boldsymbol{H}^{(l_j+1)}),\boldsymbol{Y}-\boldsymbol{H}^{(l_j+1)}\right\rangle \geq 0.
\]
The inner products sum the Frobenius inner products over all scales.
Because the data and Laplacians are fixed, $\mathcal{O}$ is a polynomial in the factor entries and has continuous gradients.
Passing to the limit in these inequalities yields
\begin{align*}
\langle\nabla_{\boldsymbol{W}}\mathcal{O}(\boldsymbol{W}^*,\boldsymbol{H}^*),\boldsymbol{V}-\boldsymbol{W}^*\rangle &\geq 0, \\
\langle\nabla_{\boldsymbol{H}}\mathcal{O}(\boldsymbol{W}^*,\boldsymbol{H}^*),\boldsymbol{Y}-\boldsymbol{H}^*\rangle &\geq 0
\end{align*}
for every $\boldsymbol{V},\boldsymbol{Y} \geq 0$.

For either factor block, increasing a single coordinate in its variational inequality shows that the corresponding gradient entry is nonnegative.
At a positive coordinate, sufficiently small changes of either sign are feasible, so the gradient entry must be zero.
Together with feasibility of the limit, these statements are exactly the nonnegativity and complementarity conditions in Equation~\ref{eq:pNMF_KKT}.
Since the accumulation point was arbitrary, the conclusion holds for every accumulation point of the exact sequence.
\end{proof}

Together with Theorem~\ref{thm:subproblem_nonincreasing}, this result establishes convergence of the objective values and KKT stationarity of accumulation points, without asserting convergence of the entire factor sequence.

\subsection{Computational Complexity}
\label{sec:computational_complexity}
\label{Supp-sec:computational_complexity}

The following counts assume precomputed, sparsely stored Laplacians and use $m_L:=\sum_{t=1}^m\operatorname{nnz}(L_t)$.
They count dense arithmetic in the factor and Gram matrices and sparse arithmetic in the Laplacians.

At each outer iteration, forming $H_tH_t^\top$ and $XH_t^\top$ over all scales costs $O(md^2n+mpnd)$, while forming $W_t^\top W_t$ and $W_t^\top X$ costs $O(mpd^2+mpnd)$.
Thus, the shared product cost is
\[
O\bigl(mpnd+md^2(p+n)\bigr).
\]
Computing the basis Gram inverses and the embedding Gram eigendecompositions adds $O(md^3)$ work.
They are reused throughout the corresponding block solve.
The degree shifts require only scalar operations after those decompositions, so no $O(mnd^3)$ sample-wise factorization term is needed for this realization.

One Hessian or preconditioner application across the $mp$ basis row systems costs $O(mpd^2)$.
In Equation~\ref{eq:H_hessian_action}, the Gram products cost $O(md^2n)$, the Laplacian products $O(dm_L)$, and the adjacent-scale terms $O(mdn)$.
Since $d \geq 1$, the last term is absorbed, giving
\[
O(md^2n+dm_L)
\]
per joint Hessian application.
The preconditioner applications and coordinate masks cost $O(md^2n)$ and $O(mdn)$, respectively, and remain within the same order.

Let $R_W,R_H \geq 1$ bound the Newton iterations per basis-batch solve and joint embedding solve, respectively, and let $J_W,J_H \geq 1$ bound the corresponding PCG iterations per Newton step.
A basis batch containing $b$ scales costs $O(bpd^2)$ per Hessian or preconditioner application, and summing over batches gives $O(mpd^2)$.
With a fixed bound on line-search trials, the total work over at most $l_{\max}$ outer iterations is
\begin{equation}
\label{Supp-eq:total_algorithm_complexity}
\begin{aligned}
O\!\Bigl(l_{\max}\bigl[&mpnd+md^2(p+n)+md^3 \\
&+R_WJ_Wmpd^2+R_HJ_H(md^2n+dm_L)\bigr]\Bigr).
\end{aligned}
\end{equation}
Gradient evaluations, projected residuals, fallback directions, and a bounded number of trial quadratic differences are covered by these orders.
Recomputing the outer residual and objective requires the same types of products and does not change the total bound.
The counts are expressed in terms of iteration budgets and do not assert a condition-number-independent PCG iteration count or a rate to a prescribed outer accuracy.

For $N_{\mathrm{wrk}} \geq 1$ workers assigned to the independent basis-row arithmetic, balanced scheduling gives the idealized linear-solve time
\[
O\!\left(l_{\max}R_WJ_W\left\lceil \frac{mp}{N_{\mathrm{wrk}}}\right\rceil d^2\right)
\]
after the shared Gram matrices and inverse preconditioners are available.
This is the parallel time for the basis solve, while \Cref{Supp-eq:total_algorithm_complexity} is total arithmetic work.
Communication, batch reductions, synchronization, and the joint embedding solve are not included in this idealized basis estimate.

The stored arrays comprise $X$, the factor sequences, sparse Laplacians, small Gram matrices and their decompositions, and a constant number of work vectors.
Their total storage is
\[
O(pn+mpd+mdn+m_L+md^2).
\]
The sample-dependent diagonal shifts can be applied without storing their dense inverse blocks.
For dense Laplacians, $m_L=O(mn^2)$, whereas sparse storage uses the actual aggregate number of nonzeros.
The full $dmn \times dmn$ embedding Hessian is never stored.

\section{Numerical Experiments}
\label{Supp-sec:numerical_experiments}
The main paper uses three geometric constructions to illustrate embedding evolution, then studies coupling, conditional stability, and computational cost through controlled experiments on the noncoplanar-circle construction.
This section supplies the data-generation and optimization settings and extends the paired coupling comparison to the other two constructions.
It also gives the connectivity checks, additional embedding views, and quantitative diagnostics that explain the scope of the reported conclusions.

\subsection{Synthetic Data and Experimental Settings}
\label{Supp-sec:synthetic_experiment_settings}
The three examples differ in geometry and sampling density, while all use $100$ observed features and embedding dimension $d=2$.
The sample sizes are chosen to keep the relevant structures visible in the embedding plots.
Because the constructions also use different perturbation models, their side-by-side display compares forms of scale-dependent behavior rather than performance under a common noise distribution or a controlled increase in sample size.

\paragraph{Noncoplanar circles}
The reference construction contains four circles with radii $5$, $10$, $20$, and $50$, each sampled at $20$ equally spaced angles in a randomly chosen plane.
Plane normals are drawn from normalized three-dimensional Gaussian vectors, rejecting a candidate if its absolute inner product with an earlier normal exceeds $0.98$.
Independent $\mathcal{N}(0,0.5^2)$ noise is added to the three structured coordinates, after which each coordinate is translated so that its minimum is $5$.
We append $97$ independent nuisance coordinates drawn from $\mathcal{N}(5,1)$ and translate a nuisance coordinate only if necessary to ensure nonnegativity.
We use the seed-$0$ realization $X \in \mathbb{R}_+^{100\times80}$ for the circle example in Figure~\ref{fig:multi_scale_vis}(a) and the conditional diagnostics in \Cref{Supp-sec:detailed_theoretical_validation}.

The coupling comparison generates ten realizations with seeds $0$--$9$.
Each seed controls the plane orientations, spatial perturbations, and nuisance coordinates.

\paragraph{Hierarchical disks}
To separate nested grouping from the geometry of hollow curves, the second construction samples the interiors of nine planar disks.
Three group centers lie on a circle of radius $15$, at angles $\pi/2+2\pi a/3$ for $a=0,1,2$.
Within group $a$, three leaf-disk centers lie at distance $3$ from the group center, at angles $0.15+0.37a+2\pi b/3$ for $b=0,1,2$.
Each leaf disk has radius $1$ and contains $30$ points, giving $n=270$.
For each point, the radial distance is $\sqrt{U}$ with $U$ uniform on $[0,1]$, and the angle is uniform on $[0,2\pi)$, so sampling is uniform with respect to area.
No background points are placed between the leaf disks, and no additional Gaussian noise is added to their two structured coordinates.
We append $98$ independent $\mathcal{N}(0,0.02^2)$ nuisance coordinates and translate each of the $100$ features so that its minimum is $0.01$.
The reference realization uses seed $0$, and the fine and coarse generating labels are retained for display and connectivity checks only.

\paragraph{Continuous branching}
The third construction replaces separated groups with points sampled around seven connected planar line segments.
The trunk starts at $(0,0)$, points in direction $(0,1)$, and has length $2$.
At each of the three bifurcations, the two child directions are symmetric about the parent direction and separated by $60^\circ$.
In the order trunk, left main, right main, left outer, left inner, right inner, and right outer, the segment lengths are $2,4,3,2,1,1,3$.
We use $30$ points per unit segment length, giving respective sample counts $60,120,90,60,30,30,90$ and a total of $n=480$.
This construction varies sibling lengths while keeping their target linear sampling density equal.
On a segment with $q$ points, the longitudinal coordinate of point $i$, indexed from $0$, is $(i+0.5+u_i)/q$ with $u_i$ uniform on $[-0.35,0.35]$.
An independent transverse offset, uniform on $[-0.02,0.02]$, gives the sampled branches a small width without introducing prescribed gaps in the underlying construction.
As for the disks, we append $98$ independent $\mathcal{N}(0,0.02^2)$ nuisance coordinates and translate each feature to have minimum $0.01$.
The reference realization uses seed $0$.
Branch annotations are used only to follow the same points across panels and are not treated as connected-component labels or supplied to pNMF.

\paragraph{Shared model and optimization settings}
Fitting uses all $100$ observed coordinates without a logarithmic transformation or column normalization.
For the representative fits and coupling ablation, we use $\alpha=1.5$ and $(\lambda_1,\lambda_2,\lambda_3)=(10^4,10^4,1)$.
The graph sequence follows the non-strict threshold rule $r_{ij} \leq \varepsilon_t$ at the distinct positive finite $\mathcal{H}_0$ death values.
These pNMF fits use double precision and NNDSVDA initialization computed from an exact singular value decomposition, and each paired Decoupled fit changes only $\lambda_2$ to zero.
For the second column of Figure~\ref{fig:multi_scale_vis} in the main paper, we apply two-dimensional t-SNE~\cite{van2008visualizing} separately to each dataset using all $100$ input features and Euclidean distances.
We use PCA initialization, perplexity $30$, an automatic learning rate, a Barnes--Hut angle of $0.5$, at most $1000$ iterations, and random seed $42$.
The five embedding panels in each row show the saved pNMF coordinates without alignment or rescaling.

The configured budget for representative fits and coupling comparisons is $1500$ outer iterations with convergence-based early stopping disabled.
A numerical monotonicity safeguard can nevertheless terminate a run, and the recorded stopping reason is retained when this occurs.
The circle and branching references completed $1500$ iterations, with final joint relative projected-KKT residuals $3.43 \times 10^{-4}$ and $7.87 \times 10^{-7}$, respectively.
The hierarchical-disk reference stopped under the safeguard after $1103$ iterations, with residual $3.48 \times 10^{-6}$.
None of these records attains the configured $10^{-8}$ KKT tolerance.
These references support the qualitative embedding comparisons and are not used to compare runtimes at a common stationarity accuracy.

\begin{table}[htbp]
\centering
\caption{Configured pNMF iteration limits. Actual termination is reported separately when a numerical safeguard stops a representative fit.}
\label{Supp-tab:synthetic_solver_settings}
\small
\begin{tabular}{ccc}
\hline
Setting & Reference/ablation & Computational \\
\hline
Maximum outer iterations & $1500$ & $2000$ \\
Newton steps per basis-batch solve & $50$ & $50$ \\
Newton steps per joint embedding solve & $50$ & $5$ \\
PCG iterations per basis Newton step & $40$ & $40$ \\
PCG iterations per embedding Newton step & $80$ & $20$ \\
\hline
\end{tabular}
\end{table}

The inner tolerance is applied to the normalized projected residual in \Cref{Supp-eq:inner_projected_residual_test}, with basis rows processed in batches of eight scales.
It starts at $10^{-2}$ and is halved after each outer iteration down to $10^{-6}$.
The PCG relative residual tolerance is $10^{-3}$ and is checked every two iterations and at the iteration limit.
The shared settings in \Cref{Supp-sec:projected_newton_details} are retained, including the Armijo constant $10^{-4}$, step reduction factor $1/2$, at most $21$ line-search trials per attempt, and working-set tolerance $a=10^{-12}$.
The numerical monotonicity safeguard and its returned-iterate convention are specified in \Cref{Supp-sec:outer_iteration_stopping}.
\Cref{Supp-tab:synthetic_solver_settings} lists the iteration limits.
The computational experiment uses a different joint stopping protocol, described in \Cref{Supp-sec:detailed_computational_performance}, and the conditional perturbation tests refine the embedding subproblem to the accuracy given in \Cref{Supp-sec:detailed_theoretical_validation}.

\subsection{Scale Selection and Additional Embedding Views}
\label{Supp-sec:persistence_scale_selection}
The reference datasets have $m=79$, $269$, and $479$ positive scales, respectively.
Connected-component counts are checked on the actual off-diagonal supports of the Laplacians, so the audit uses the same graphs as the embeddings rather than recomputing boundary-edge decisions from rounded thresholds.
For the retained realizations, each positive scale produces one component merger, and the last selected graph is connected.

The main-paper circle panels retain indices $1,40,60,78,79$.
For the hierarchical disks, the indices are $1,261,267,268,269$, where the latter four are the first graph scales with $9,3,2,1$ components.
The nine-component and three-component partitions agree with the generating leaf disks and groups, respectively.
For the branching data, the fixed index rule $1,\lceil m/2\rceil,\lceil3m/4\rceil,m-1,m$ gives $1,240,360,478,479$.
The circle indices are retained from the original experiment, while the two new constructions use the stated graph- or index-based rules, which were fixed before fitting.
All three rows are displayed in increasing scale order, so they can also be read from right to left to examine finer structure.

\begin{figure}[htbp]
\centering
\includegraphics[width=\textwidth]{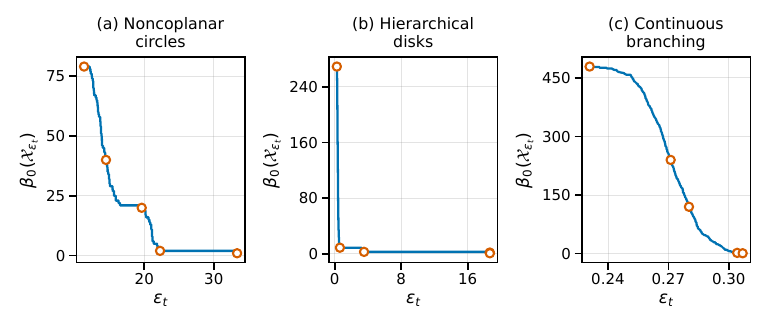}
\caption{Connected-component counts along the selected graph sequences for the three reference datasets. Each panel uses that dataset's actual thresholds, and the marked scales are those displayed in the main-paper embedding figure.}
\label{Supp-fig:three_dataset_connectivity}
\end{figure}

\begin{figure}[htbp]
\centering
\includegraphics[width=\textwidth]{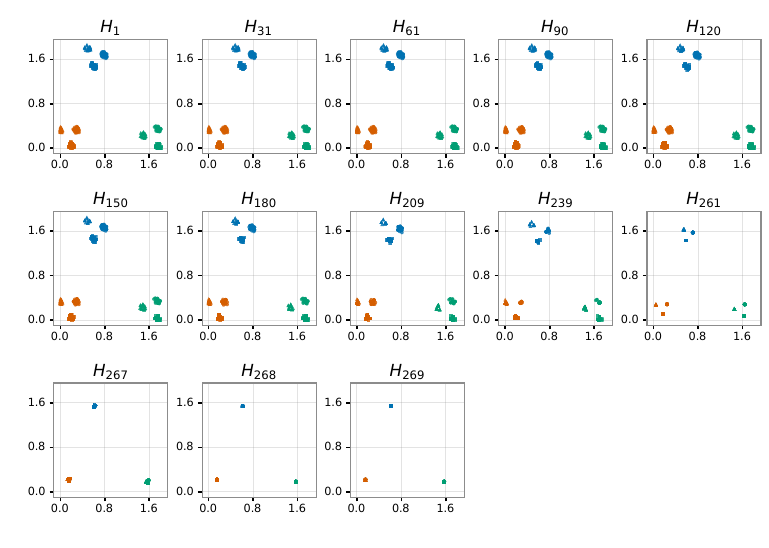}
\caption{Additional embeddings for the hierarchical disks, ordered by increasing scale and plotted with common axis limits. Colors identify the three generating groups, and circles, triangles, and squares distinguish the three disks within each group. Every panel contains all $270$ points.}
\label{Supp-fig:hierarchical_embedding_overview}
\end{figure}

\begin{figure}[htbp]
\centering
\includegraphics[width=\textwidth]{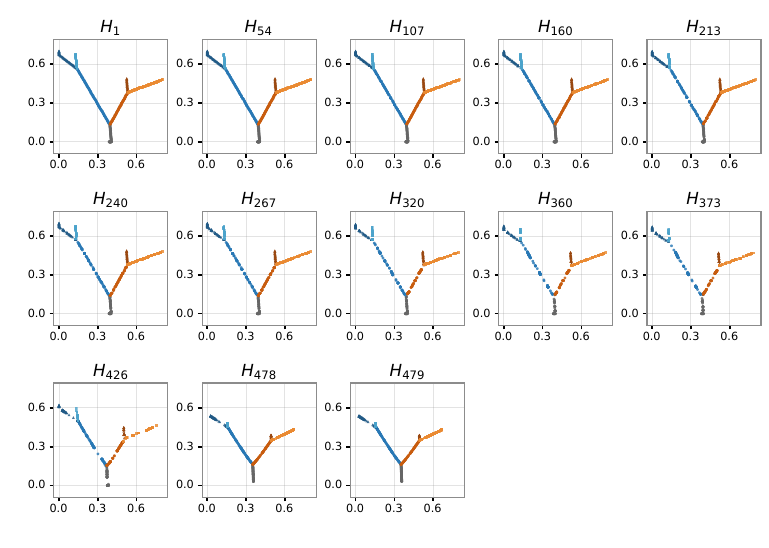}
\caption{Additional embeddings for the branching construction, ordered by increasing scale and plotted with common axis limits. Colors follow the generating branch segments, and every panel contains all $480$ points. Short terminal segments become less distinct at the largest selected scales, and intermediate embeddings contain local aggregates and gaps.}
\label{Supp-fig:branching_embedding_overview}
\end{figure}

\Cref{Supp-fig:three_dataset_connectivity} records the graph connectivity along the three selected sequences.
The detailed embedding views in \Cref{Supp-fig:hierarchical_embedding_overview,Supp-fig:branching_embedding_overview} retain the raw embedding coordinates and all sample points, with common axis limits within each dataset.
They show how the representative panels relate to additional scales without interpolation between embeddings.

\paragraph{Comparison of scale-selection rules}
To determine whether a scale grid captures the connectivity transitions underlying the embedding sequence, we compare four rules using the same number of positive scales on the standard $80$-sample realization.
PH uses the $79$ distinct finite $\mathcal{H}_0$ death values, while UDS places $79$ values uniformly between the first and last deaths.
RDS retains these endpoints and draws the $77$ interior values uniformly before sorting them.
AMS instead uses the PH deaths to draw one representative from each interval $[\varepsilon_t,\varepsilon_{t+1})$, $1 \leq t < 79$, and retains $\varepsilon_{79}$.
RDS and AMS each use $30$ selection seeds, numbered $0$--$29$.

\begin{table}[htbp]
\centering
\caption{Coverage of the $80$ canonical $\mathcal{H}_0$ connectivity states.
For random selection rules, entries give mean $\pm$ sample standard deviation over $30$ scale-selection paths.}
\label{Supp-tab:scale_selection_metrics}
\begin{tabular}{cccc}
\hline
Rule & Paths & Represented states & Coverage \\
\hline
PH  & 1  & $80$ & $1.0000$ \\
UDS & 1  & $29$ & $0.3625$ \\
RDS & 30 & $22.80\pm2.25$ & $0.2850\pm0.0281$ \\
AMS & 30 & $80$ & $1.0000$ \\
\hline
\end{tabular}
\end{table}

\begin{figure}[htbp]
\centering
\includegraphics[width=\textwidth]{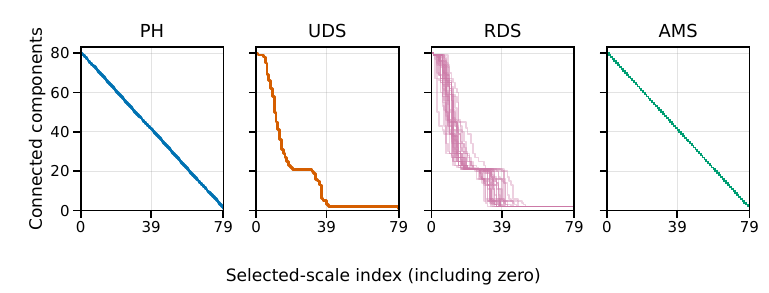}
\caption{Connected-component counts for PH, UDS, RDS, and AMS, in that order.
The initial zero scale is included for coverage evaluation, and every saved integer-valued path is plotted.
The RDS and AMS panels contain $30$ paths each, with the AMS paths coinciding because they cover the same ordered connectivity states.}
\label{Supp-fig:persistence_scale_selection}
\end{figure}

Coverage is the fraction of the $80$ canonical connectivity states represented by a path after adding scale zero.
Fitting uses the $79$ positive scales.
\Cref{Supp-fig:persistence_scale_selection,Supp-tab:scale_selection_metrics} compare the paths and their coverage without fitting additional factorizations.

Although all rules use the same number of positive scales, PH and AMS represent every connectivity state, whereas the uniform and global random grids omit many transitions.
AMS attains full coverage by using the PH intervals, so its agreement with PH reflects that shared persistence information.
PH additionally provides canonical boundary representatives.
Within an interval of constant connected-component partition, however, graph edges and heat-kernel weights can still change.
Equal connectivity coverage therefore does not imply equal Laplacians or embeddings, and the conclusion of this comparison concerns $\mathcal{H}_0$ connectivity.

\subsection{Additional Diagnostics for Cross-Scale Coupling}
\label{Supp-sec:additional_coupling_diagnostics}
AEC in the main paper measures adjacent geometric changes after removing overall scale, so we also examine embedding concentration and reconstruction error.
The ten paired realizations use data seeds $0$--$9$, which vary the circle orientations and noise.
Within each pair, Full and Decoupled pNMF share the data, positive scales, and NNDSVDA initialization, with only $\lambda_2$ changed from $10^4$ to $0$.
Both use the $1500$-iteration protocol in \Cref{Supp-sec:synthetic_experiment_settings}.
Their final joint relative KKT residuals range from $3.39 \times 10^{-4}$ to $1.04 \times 10^{-3}$ for Full and from $8.63 \times 10^{-4}$ to $2.18 \times 10^{-3}$ for Decoupled pNMF.
These fits are compared at the same budget.

\begin{figure}[htbp]
\centering
\includegraphics[width=\textwidth]{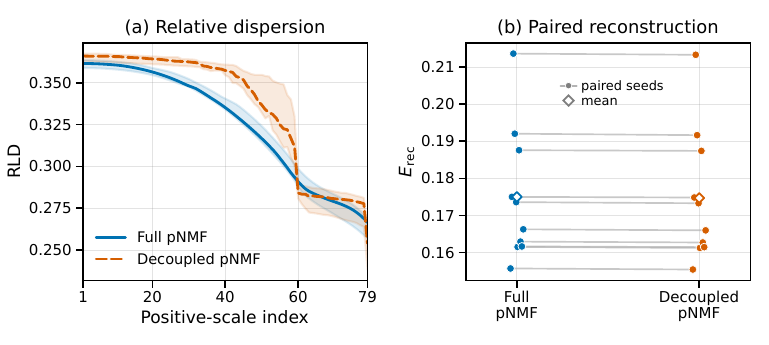}
\caption{Additional diagnostics for the ten paired realizations.
(a) Pointwise median RLD with interquartile bands over the $79$ positive scales.
(b) Scale-averaged relative reconstruction error for individual runs, with gray segments joining matched data realizations and open diamonds marking sample means.}
\label{Supp-fig:cross_scale_coupling_detailed}
\end{figure}

To complement AEC, \Cref{Supp-fig:cross_scale_coupling_detailed}(a) reports the relative latent dispersion
\[
\operatorname{RLD}_t := \frac{\|H_t-\bar{H}_t\|_F}{\|H_t\|_F+10^{-12}}, \qquad
\bar{H}_t:=\frac{1}{n}H_t\mathbf{1}\mathbf{1}^\top.
\]
This quantity describes concentration relative to the uncentered embedding norm and is interpreted together with AEC and reconstruction error, rather than minimized as a score.
The recorded values remain positive, indicating nonzero sample dispersion.

The ten-realization means of scale-averaged AEC are $0.00553065$ for Full and $0.01133109$ for Decoupled pNMF, whose ratio gives a reduction of $51.19\%$, rounded to $51.2\%$ in the main paper.
Averaging the ten within-pair percentage reductions instead gives $50.04\%$, with individual reductions from $44.44\%$ to $64.63\%$.
The first summary compares the two overall means, whereas the second gives equal weight to each realization's relative reduction.
The reconstruction comparison in panel~(b) shows that error increases in every pair, with mean increase $0.00025914$, or about $0.148\%$ relative to the Decoupled mean.
Thus, the reduction in adjacent geometric changes is consistent across realizations and comes with a small, nonzero fitting cost.

\subsubsection{Paired Comparisons on the Additional Constructions}
\label{Supp-sec:additional_geometry_coupling}
To determine whether the reduction in adjacent geometric changes extends beyond the circle construction, we repeat the comparison on the hierarchical disks and continuous branching data.
Each construction uses ten data seeds, numbered $0$--$9$, and Full and Decoupled pNMF share the input, selected graphs, and initialization within each pair.
The geometric parameters and sampling rules remain those in \Cref{Supp-sec:synthetic_experiment_settings}, while the seeds vary the sampled points and nuisance coordinates.
Both variants use the same configured limit of $1500$ outer iterations, although the numerical safeguard can end a run earlier.

\begin{figure}[htbp]
\centering
\includegraphics[width=\textwidth]{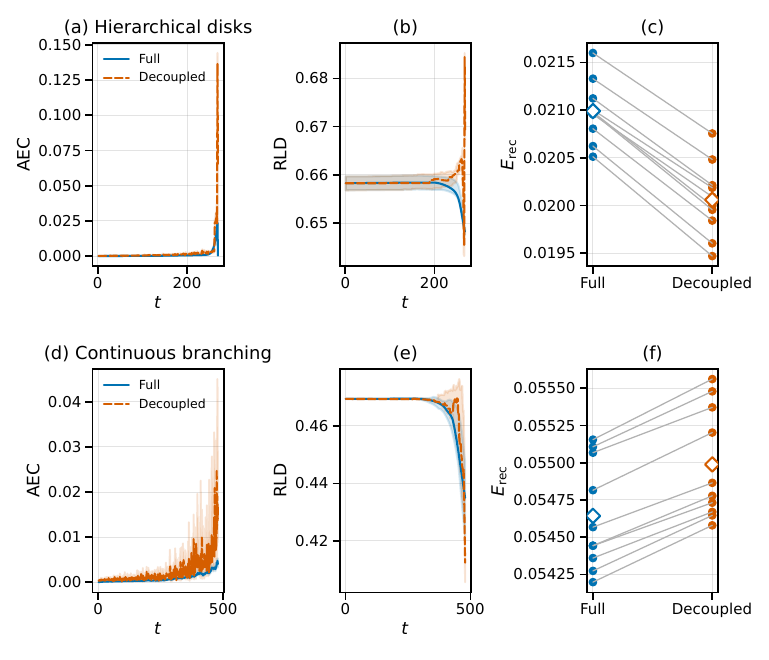}
\caption{Paired coupling comparisons for hierarchical disks (top row) and continuous branching (bottom row).
The first two columns show pointwise median AEC and RLD with interquartile bands over ten data seeds.
The last column shows reconstruction errors for individual fits, with gray lines joining paired realizations and open diamonds marking means.}
\label{Supp-fig:additional_geometry_coupling}
\end{figure}

\Cref{Supp-fig:additional_geometry_coupling} compares AEC, relative latent dispersion, and reconstruction error for these runs.
Full pNMF has smaller scale-averaged AEC in all ten pairs for each construction.
The ratios of the two method means give reductions of $63.37\%$ for the hierarchical disks and $74.53\%$ for continuous branching.
The pointwise median is also lower at $262$ of $268$ transitions for the disks and $472$ of $478$ transitions for branching.

\begin{table}[htbp]
\centering
\caption{Coupling comparisons on the additional constructions. Entries give mean $\pm$ sample standard deviation over ten realizations. AEC is averaged across adjacent scales within each run.}
\label{Supp-tab:additional_coupling_metrics}
\small
\begin{tabular}{cccc}
\hline
Construction & Variant & Mean AEC & $E_{\mathrm{rec}}$ \\
\hline
Hierarchical disks & Full & $(8.450\pm0.192)\times10^{-4}$ & $0.020990\pm0.000317$ \\
 & Decoupled & $(2.307\pm0.143)\times10^{-3}$ & $0.020059\pm0.000383$ \\
Continuous branching & Full & $(8.556\pm0.401)\times10^{-4}$ & $0.054643\pm0.000363$ \\
 & Decoupled & $(3.360\pm0.121)\times10^{-3}$ & $0.054988\pm0.000376$ \\
\hline
\end{tabular}
\end{table}

The associated reconstruction changes depend on the construction, as \Cref{Supp-tab:additional_coupling_metrics} shows.
For the disks, the error increases in every pair, with mean increase $9.30\times10^{-4}$.
For branching, it decreases in every pair, with mean change $-3.45\times10^{-4}$.
RLD remains positive for both variants, which confirms that the small AEC values do not arise from complete point collapse.

\begin{table}[htbp]
\centering
\caption{Actual stopping records for the additional paired comparisons. Ranges are over ten runs, and the last column counts numerical monotonicity safeguard stops.}
\label{Supp-tab:additional_coupling_stopping}
\small
\setlength{\tabcolsep}{3pt}
\begin{tabular}{ccccc}
\hline
Construction & Variant & Outer iterations & Final joint KKT & Safeguard stops \\
\hline
Hierarchical disks & Full & $964$--$1093$ & $(3.55\text{--}6.64)\times10^{-6}$ & $10$ \\
 & Decoupled & $1055$--$1500$ & $(2.75\text{--}5.02)\times10^{-6}$ & $9$ \\
Continuous branching & Full & $1500$ & $(6.76\text{--}8.05)\times10^{-7}$ & $0$ \\
 & Decoupled & $1500$ & $(9.59\text{--}9.84)\times10^{-7}$ & $0$ \\
\hline
\end{tabular}
\end{table}

The stopping records in \Cref{Supp-tab:additional_coupling_stopping} qualify this comparison.
For the hierarchical disks, all ten Full runs and nine Decoupled runs stop under the numerical monotonicity safeguard, while the remaining Decoupled run reaches the iteration limit.
All branching runs complete $1500$ iterations.
None reaches the configured $10^{-8}$ joint KKT tolerance, so these results compare the recorded fits under a common iteration limit and safeguard, without asserting equal iteration counts or attained stationarity accuracy.
The hierarchical-disk reference figure and the seed-$0$ paired comparison use separate executions on the same input matrix and selected thresholds.
Their Full fits stop after $1103$ and $1077$ iterations, with final joint residuals $3.48 \times 10^{-6}$ and $4.63 \times 10^{-6}$, respectively.
The reference diagnostics and paired summaries therefore retain their own stopping records.

\subsection{Detailed Numerical Validation of Theoretical Properties}
\label{Supp-sec:detailed_theoretical_validation}
To relate the observed evolution to the analysis, we examine graph increments, the associated block operator, and complete responses to input perturbations on the standard noncoplanar-circle realization.
The graph and embedding checks use the saved $1500$-iteration fit, while the perturbation tests fix its bases and refine the conditional embedding solution before applying the perturbations.
This allows the conditional bounds to be checked at the required numerical accuracy while retaining the basis sequence used in the main-paper experiments.

\begin{figure}[htbp]
\centering
\includegraphics[width=\textwidth]{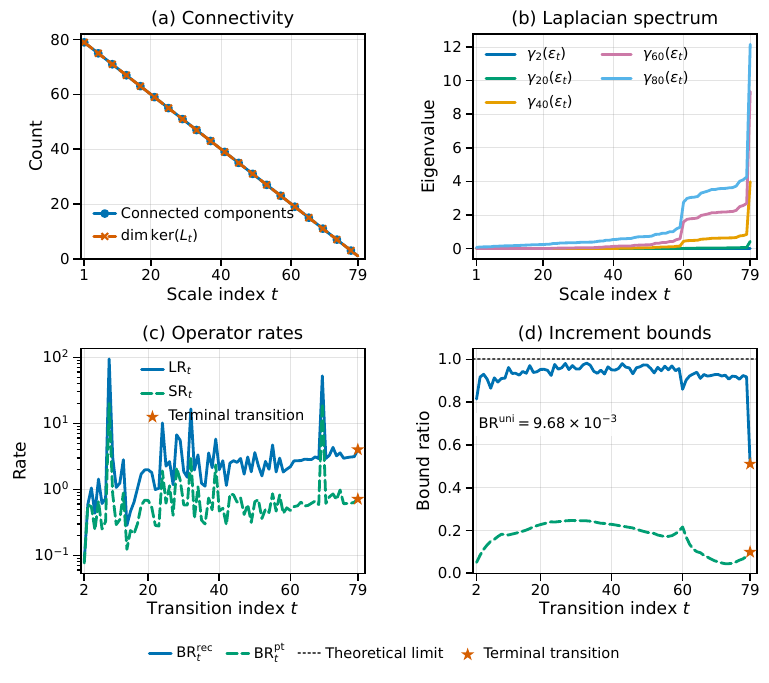}
\caption{Laplacian, spectral, and embedding evolution.
(a) Component count and Laplacian nullity.
(b) Selected ordered eigenvalues.
(c) $\mathrm{LR}_t=\|L_t-L_{t-1}\|_F/(\varepsilon_t-\varepsilon_{t-1})$ and
$\mathrm{SR}_t=\max_\ell|\gamma_\ell(\varepsilon_t)-\gamma_\ell(\varepsilon_{t-1})|/(\varepsilon_t-\varepsilon_{t-1})$.
(d) Realized-to-bound ratios $\mathrm{BR}_t^{\mathrm{rec}}$, $\mathrm{BR}_t^{\mathrm{pt}}$, and $\mathrm{BR}^{\mathrm{uni}}$ for the recursive, pointwise, and uniform increment inequalities.
Stars mark the transition to the last positive finite death scale.}
\label{Supp-fig:laplacian_spectral_embedding_evolution}
\end{figure}

\paragraph{Graph and embedding increments}
We first evaluate the structural checks and the embedding-increment inequalities in Equations~\ref{eq:recursive_inequality}, \ref{eq:uniform_embedding_increment}, and~\ref{eq:exponential_localization} on the saved sequence.
There are $79$ positive finite death scales and $78$ adjacent transitions.
In \Cref{Supp-fig:laplacian_spectral_embedding_evolution}, the component count agrees with Laplacian nullity at every scale, and support nesting, Loewner ordering, ordered-spectrum monotonicity, and the Weyl inequalities hold at every checked transition within numerical tolerance.
Negative eigenvalues of order $10^{-15}$ are consistent with floating-point roundoff.
For each embedding-increment inequality, we divide its realized left-hand side by its right-hand-side bound, so a ratio at most one indicates that the inequality is satisfied.
The largest ratios for the recursive, pointwise, and uniform bounds are $0.982259$, $0.246428$, and $0.009678$, respectively, showing that all three are satisfied but with different degrees of conservatism.
The conditional $H$-block relative projected KKT residual is $5.65\times10^{-7}$, below the $10^{-6}$ diagnostic tolerance, while the joint $W/H$ residual is $3.43\times10^{-4}$.
The ratios are numerical checks on an approximate conditional embedding solution at the saved bases, rather than a certification of exact KKT stationarity.

\begin{figure}[htbp]
\centering
\includegraphics[width=\textwidth]{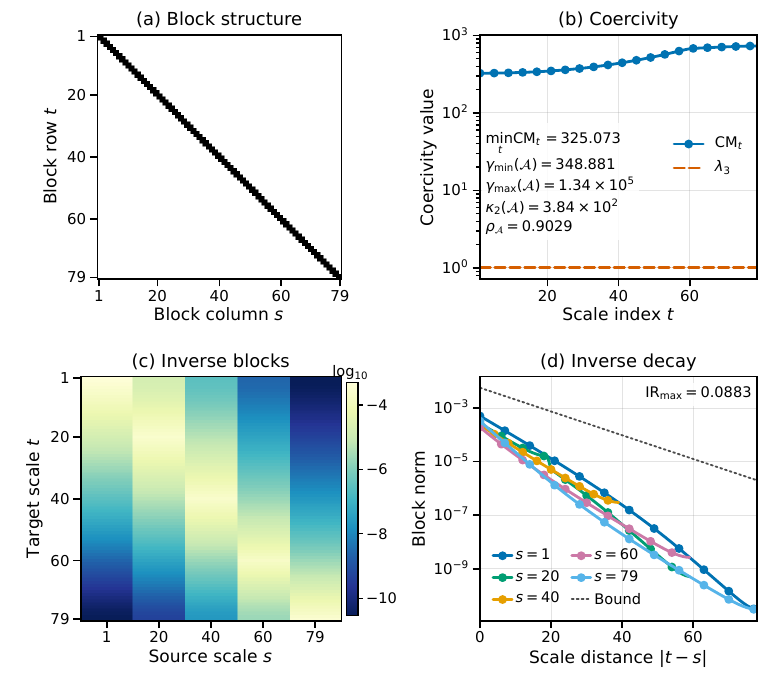}
\caption{Coercivity and inverse-block diagnostics.
(a) Block pattern of $\mathcal{A}$.
(b) Coercivity margins and spectral summaries.
(c) Base-$10$ logarithms of inverse-block norms for the five prescribed source scales.
(d) Maximum block norms at each scale distance and the exponential bound.
$\mathrm{IR}_{\max}$ is the largest realized-to-bound ratio among the $395$ inspected blocks.}
\label{Supp-fig:block_operator_coercivity_inverse_decay}
\end{figure}

\paragraph{Coercivity and inverse blocks}
Because the inverse-decay estimate relies on a positive spectral lower bound, we next inspect the coercivity margins
\[
\mathrm{CM}_t=\gamma_{\min}(\mathcal M_t)-\sum_{s\ne t}\|\mathcal{A}_{t,s}\|_2
\]
and five prescribed inverse block columns in \Cref{Supp-fig:block_operator_coercivity_inverse_decay}.
The minimum margin is $325.073>\lambda_3=1$, and $\gamma_{\min}(\mathcal{A})=348.881$.
Using source scales $1,20,40,60,79$, all $395$ inspected blocks satisfy the inverse bound, with maximum realized-to-bound ratio $0.088310$.
Across these $395$ blocks, the ratio of the largest to the smallest norm is approximately $10^{7.23}$.
The largest recorded linear-solve residual is $4.29\times10^{-15}$.
These checks support the predicted attenuation for the five inspected block columns.

\paragraph{Conditional input perturbations}
We then examine whether the inverse localization is reflected in responses to perturbed inputs.
The basis sequence remains fixed while the reference nonnegative quadratic subproblem is refined.
The refinement first uses the normalized full-gradient projected residual with tolerance $10^{-10}$ and at most $50$ projected-Newton iterations, with PCG tolerance $10^{-12}$ and a maximum of $20000$ PCG iterations.
This diagnostic solve uses a diagonal preconditioner and allows $31$ backtracking trials per line-search attempt.
It then solves the identified free-coordinate system directly, updating the active set until primal and dual checks pass, with a limit of $100$ such refinements.
Coordinates at most $10^{-10}$ are classified as active during this polishing stage.
The polished reference has $37$ active and $12603$ free coefficients, and its saved residual is
\[
\frac{\|\mathbf{h}-[\mathbf{h}-2\mathbf{r}]_+\|_2}{\max\{1,\|\mathbf{h}\|_2\}}=1.77\times10^{-11}, \qquad
\mathbf{r}=\mathcal{A}\mathbf{h}-\mathbf{b}.
\]
Let $F$ denote the reference free-coordinate set.
For each perturbed principal solve, with $\widetilde{\mathbf{r}}=\widetilde{\mathcal{A}}\widetilde{\mathbf{h}}-\widetilde{\mathbf{b}}$, the common-active-set checks are
\[
\min_{i\in F}\widetilde h_i>10^{-10},
\qquad \min_{i\notin F}\widetilde r_i\geq-10^{-8},
\qquad
\frac{\|\widetilde{\mathbf{r}}_F\|_2}
{\max\{1,\|\widetilde{\mathbf{b}}_F\|_2\}}\leq10^{-10}.
\]

\begin{figure}[htbp]
\centering
\includegraphics[width=\textwidth,height=0.78\textheight,keepaspectratio]{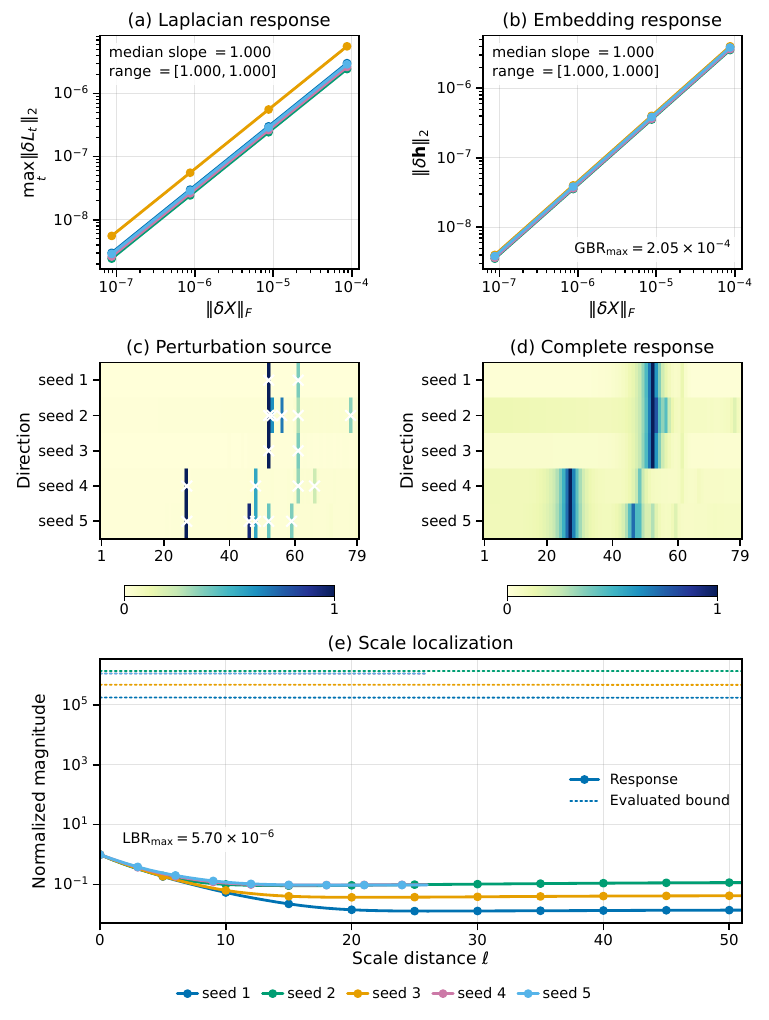}
\caption{Conditional perturbation and localization diagnostics with fixed bases and checked common active sets.
(a)--(b) Support-preserving Laplacian and embedding responses, with $\mathrm{GBR}_{\max}$ denoting the maximum realized global-bound ratio.
(c)--(d) Source and complete-response block norms, normalized separately by each row's maximum. Crosses in (c) mark changed-support scales.
(e) Normalized complete responses and the evaluated bounds in \Cref{Supp-eq:localization_observed_bound}, including the background term.
Response curves take the maximum over scales at the same distance, and each perturbation uses only its attained distance domain.
$\mathrm{LBR}_{\max}$ denotes the maximum response-to-bound ratio over the retained distance-group comparisons.}
\label{Supp-fig:input_perturbation_localization_coupling}
\end{figure}

Five Gaussian perturbation directions use seeds $1$--$5$ and are each rescaled to Frobenius norm $\|X\|_F$, so multiplying a direction by a scalar gives that relative input perturbation magnitude.
The support-preserving tests use the four magnitudes $10^{-10},10^{-9},10^{-8},10^{-7}$ for each direction, giving $20$ cases.
The Laplacian and embedding responses in \Cref{Supp-fig:input_perturbation_localization_coupling}(a)--(b) have log--log slopes close to one, consistent with the local Lipschitz comparison when the checked supports and scale correspondence are preserved.
For the global-bound ratio in panel~(b), the numerator is the observed response norm and the denominator is
\[
\frac{\|\delta\mathbf{b}_F\|_2+
\lambda_1\max_t\|\delta L_t\|_2\|\mathbf{h}\|_2}
{\gamma_{\min}(\mathcal{A})}.
\]
Here $\delta\mathbf{b}_F=P_F^\top\delta\mathbf{b}$ uses the common free coordinates.
The fixed-basis reconstruction, coupling, and factor terms have lower spectral bound $\gamma_{\min}(\mathcal{A})\approx348.881$, attained on the common constant-sample nullspace of the graph terms.
This remains a lower bound after perturbing the positive semidefinite graph terms and restricting to $F$, so it sharpens the $\lambda_3$-based bound for these tests.
The largest observed ratio is $2.05\times10^{-4}$.

For the support-changing tests, candidate magnitudes are taken from
\[
10^{-6}\{1,2,3,5,7,10,15,20,25,30,40,50,60,80,100\}.
\]
We search from largest to smallest and retain one case per direction that preserves input nonnegativity, death-edge correspondence and ordering, and the coefficient active set, while changing graph supports at between $1$ and $10$ scales.
The search examines $19$ direction--magnitude candidates, of which nine fail the death-edge ordering check and five fail the common-active-set check.
The five retained relative magnitudes lie between $10^{-5}$ and $10^{-4}$.
Because this screening enforces the assumptions of the conditional comparison, its acceptance counts do not estimate a general perturbation success probability.

The active constraints in the reference solution lead us to use the common-free-coordinate extension in \Cref{Supp-sec:common_free_input_bounds}.
For retained perturbation $j$, write $\mathcal{I}_j$ for its support-changing scale set and $\mathbf{g}_{j,t}$ for its effective source block on the common free coordinates, grouped by scale.
The computed response includes source contributions from every scale, and we normalize it by
\[
a_j:=\max_t\|(\delta\mathbf{h}_j)_t\|_2.
\]
Panel~(c) of Figure~\ref{fig:theoretical_validation_scale_evolution_stability} in the main paper plots the normalized embedding response $\|(\delta\mathbf{h}_j)_t\|_2/a_j$ against the scale index $t$ for each perturbation $j$.
Triangles mark the support-changing scale indices in $\mathcal{I}_j$.
For panel~(e) of \Cref{Supp-fig:input_perturbation_localization_coupling}, we group scales by their distance $\ell=\min_{k \in \mathcal{I}_j}|t-k|$ to the nearest support-changing scale and take the maximum normalized response in each group.

For each attained scale distance $\ell$, Equation~\ref{eq:support_changing_localization} implies the upper envelope shown in \Cref{Supp-fig:input_perturbation_localization_coupling}(e):
\begin{equation}
\label{Supp-eq:localization_observed_bound}
\max_{\substack{1\leq t\leq m\\
\min_{k\in\mathcal I_j}|t-k|=\ell}}
\frac{\|(\delta\mathbf{h}_j)_t\|_2}{a_j}
\leq \frac{2}{\lambda_3a_j}
\left[
C_{\mathrm{g}}^{\mathrm{obs}}
\frac{1+\bar\rho}{1-\bar\rho}\|\delta X_j\|_F
+\bar\rho^{\ell}\sum_{k\in\mathcal{I}_j}\|\mathbf{g}_{j,k}\|_2
\right].
\end{equation}
Indeed, $|t-k|\geq\ell$ for every $k\in\mathcal I_j$, so each weight $\bar\rho^{|t-k|}$ is bounded above by $\bar\rho^\ell$ before taking the maximum over target scales at that distance.
Here the background constant is evaluated over the five retained perturbations:
\[
C_{\mathrm{g}}^{\mathrm{obs}}
:=\max_j\max_{t\notin\mathcal{I}_j}
\frac{\|\mathbf{g}_{j,t}\|_2}{\|\delta X_j\|_F}
=35.937871,
\qquad \bar\rho=0.997821631.
\]
The common decay factor uses the recorded condition-number bound $\bar\kappa=841105.358$ and the lower spectral bound $\lambda_3=1$.
Since $C_{\mathrm{g}}^{\mathrm{obs}}$ is obtained from the same source sequences, the comparison checks these retained cases and does not provide a uniform constant for arbitrary perturbations.
Panel~(e) of \Cref{Supp-fig:input_perturbation_localization_coupling} compares these response envelopes with \Cref{Supp-eq:localization_observed_bound} over each perturbation's attained distance range.

Between $98.876\%$ and $99.977\%$ of the source squared Euclidean norm lies on changed-support scales, which motivates examining the distance from those scales.
The complete responses concentrate nearby but remain approximately $0.014$--$0.117$ at the largest observed distances after normalization.
All $210$ distance-group comparisons satisfy the envelope bound, with normalized upper bounds between $1.72\times10^5$ and $1.35\times10^6$ and a largest response-to-bound ratio of $5.70\times10^{-6}$.
The bounds are therefore highly conservative, while the nonzero distant responses are compatible with their background term.

\subsection{Detailed Computational Performance Results}
\label{Supp-sec:detailed_computational_performance}
To assess the cost of increasing both sample count and sequence length, we vary $n$ over $80$, $400$, $800$, $1200$, and $1600$.
Each circle contains $n/4$ points, while the ambient dimension remains $100$ and the embedding dimension is $d=2$.
Each size uses data seeds $0,1,2$, and the $45$ recorded fits comprise Full pNMF, Decoupled pNMF, and single-scale GNMF.
The paired pNMF variants share the data, scales, graphs, initialization, and stopping protocol, with parameters and inner limits as specified in \Cref{Supp-sec:synthetic_experiment_settings,Supp-tab:synthetic_solver_settings}.
Computations use float64 CuPy~13.6.0 on one NVIDIA GeForce RTX~4090, with Python~3.11.7, NumPy~1.26.3, SciPy~1.13.1, and CPU BLAS restricted to one thread.
For pNMF, stopping requires relative objective change at most $10^{-5}$ and joint relative projected KKT residual at most $10^{-4}$, subject to $2000$ outer iterations.
GNMF uses the symmetric union of five-nearest-neighbor edges, with retained edge weights $\exp[-r_{ij}^2/(\max_{a,b}r_{ab})^{1.5}]$ and graph weight $10^4$.
Its objective contains the reconstruction and graph terms, and both factors use NNDSVDA initialization and multiplicative updates without factor renormalization.
The same objective-change and joint-KKT tolerances are checked at every iteration, with a limit of $2000$ iterations.
Its different iteration unit, objective, and single-scale workload make it a cost reference for interpreting the expense of a full sequence.

\begin{figure}[htbp]
\centering
\includegraphics[width=\textwidth,height=0.78\textheight,keepaspectratio]{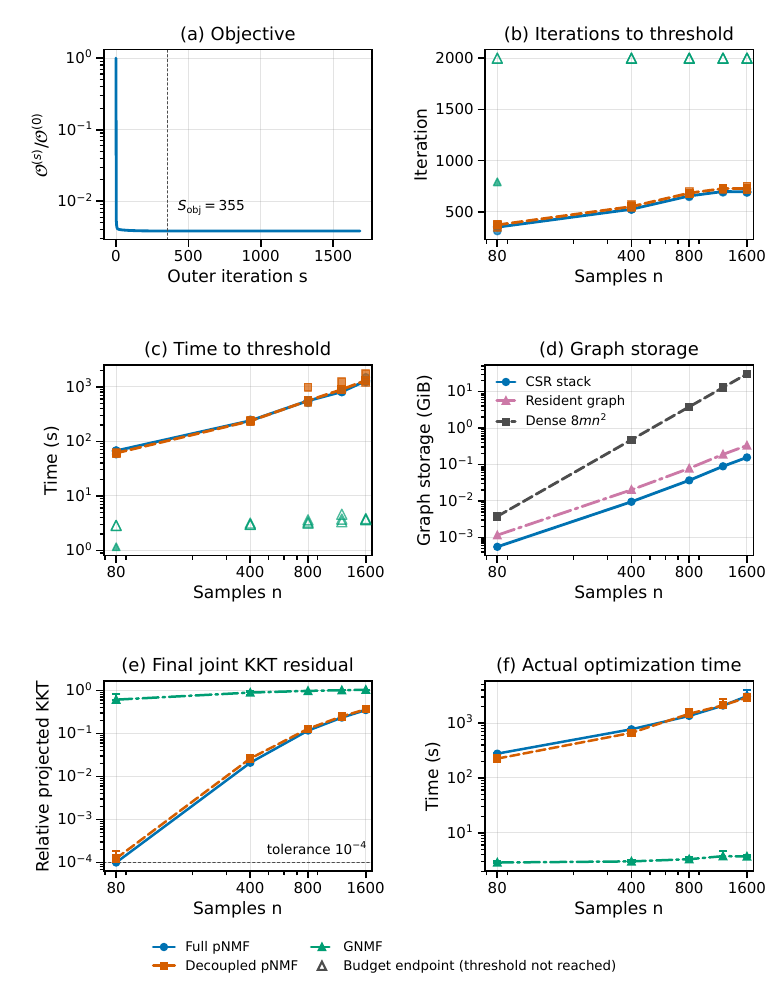}
\caption{Recorded optimization progress, cost, and graph storage.
(a) Representative Full objective for $n=80$, seed $0$.
(b)--(c) Iterations and elapsed times to the objective-change threshold. Filled symbols indicate runs that reach the threshold, lines connect three-seed medians only when all three runs reach it, and open triangles mark budget endpoints for the remaining runs.
(d) Median total storage of the per-scale CSR buffers, all resident graph structures, and the dense $8mn^2$ reference.
(e)--(f) Final joint relative KKT residual and actual optimization time, with medians and minimum--maximum bars over three seeds.
The dotted line in (e) is $10^{-4}$, and GNMF provides a single-scale cost reference.}
\label{Supp-fig:computational_performance_detailed}
\end{figure}

We distinguish objective stabilization from the full recorded optimization cost because a small objective change can occur while the joint residual remains substantial.
At the first iteration satisfying the objective-change threshold of $10^{-5}$, we record the iteration count $S_{\mathrm{obj}}$ and elapsed optimization time $T_{\mathrm{obj}}$. The full optimization time includes all iterations through joint stopping or budget exhaustion.
Both timings exclude graph construction, initialization, result writing, and plotting.
\Cref{Supp-fig:computational_performance_detailed} therefore reports both quantities together with the final residuals and full optimization cost.

\begin{table}[htbp]
\centering
\caption{Actual optimization time (median, seconds), final joint relative KKT ranges, and joint attainment counts.
Each method has three runs at each size, and F and D denote Full and Decoupled pNMF.}
\label{Supp-tab:computational_budget}
\small
\setlength{\tabcolsep}{3.5pt}
\begin{tabular}{cccccc}
\hline
$n$ & Time F & Time D & KKT F & KKT D & Joint (F,D)\\
\hline
80 & 277.4 & 225.4 & $(9.97\text{--}9.99)\times10^{-5}$ & $(1.04\text{--}1.83)\times10^{-4}$ & $(3,0)$\\
400 & 770.7 & 663.0 & $0.0200\text{--}0.0224$ & $0.0248\text{--}0.0281$ & $(0,0)$\\
800 & 1349.8 & 1476.8 & $0.1135\text{--}0.1173$ & $0.1234\text{--}0.1285$ & $(0,0)$\\
1200 & 2081.2 & 2104.9 & $0.2312\text{--}0.2458$ & $0.2454\text{--}0.2648$ & $(0,0)$\\
1600 & 3094.3 & 2912.4 & $0.3306\text{--}0.3686$ & $0.3458\text{--}0.3865$ & $(0,0)$\\
\hline
\end{tabular}
\end{table}

Across the five sample sizes, Full pNMF reaches the objective-change threshold after medians of $349$, $524$, $654$, $697$, and $694$ iterations.
The corresponding medians for Decoupled pNMF are $371$, $550$, $682$, $726$, and $727$.
Although Full uses slightly fewer median outer iterations at every size, the two variants have comparable median times to the threshold, as shown in panel~(c).
All $30$ pNMF runs reach the objective-change threshold, whereas GNMF does so in one of its $15$ runs.
Only three runs, all Full pNMF at $n=80$, meet the joint stopping condition, so objective stabilization alone does not describe the attained stationarity accuracy.
\Cref{Supp-tab:computational_budget} reports actual cost alongside the final residuals and attainment counts, allowing budget endpoints to be interpreted with their recorded accuracy.

The separate block residuals normalize the Frobenius norm of the projected-gradient mapping by $\max\{1,\|\boldsymbol{W}\|_F\}$ for the basis block and by $\max\{1,\|\boldsymbol{H}\|_F\}$ for the embedding block, using the stacked-factor norms from the main paper.
At $n=1600$, the Full $W$-block relative residuals are $0.459$--$0.516$, while the $H$-block residuals are $7.56\times10^{-8}$--$2.68\times10^{-7}$.
Because these values are evaluated after both block updates and the $H$ update changes the subsequent conditional gradient with respect to $W$, they do not identify the stopping accuracy of the preceding $W$ subproblem or establish the cause of slow joint convergence.

The storage comparison addresses the additional cost of retaining the growing graph sequence, whose actual scale count is $m=n-1$ in these runs.
The per-scale CSR total includes the data, column-index, and row-pointer buffers of all scale matrices, while resident graph storage also includes the retained block CSR and graph metadata.
At $n=1600$, the median per-scale CSR total is $0.1560$ GiB and the resident graph total is $0.3311$ GiB, compared with the $30.4985$ GiB dense reference.
The graph totals exclude factors, solver workspaces, and allocator caches.

\section{Application to scRNA-seq Data Clustering}
\label{Supp-sec:scrna_clustering}

We give the evaluation protocol, show first-scale projections for the other four datasets, and compare Full and Decoupled pNMF across all fitted scales on all five datasets.

\subsection{Evaluation Protocol}
\label{Supp-sec:scrna_evaluation_protocol}

\begin{table}[htbp]
\centering
\caption{Data dimensions and pNMF sequence lengths. Here $n$ is the number of cells, $p$ the number of retained genes, $n_{\mathrm{cl}}$ the annotated cluster count, $d=\lceil\sqrt n\rceil$ the latent dimension used by all methods, and $m$ the number of fitted positive scales.}
\label{Supp-tab:scrna_dataset_dimensions}
\begin{tabular}{cccccc}
\hline
Dataset & $n$ & $p$ & $n_{\mathrm{cl}}$ & $d$ & $m$\\
\hline
GSE75748time & 758 & 19,189 & 6 & 28 & 757\\
GSE94820 & 1,140 & 26,593 & 5 & 34 & 1,139\\
GSE75140 & 734 & 18,927 & 9 & 28 & 733\\
Dendritic\_batch1 & 384 & 15,426 & 4 & 20 & 383\\
Dendritic\_batch2 & 384 & 15,426 & 4 & 20 & 383\\
\hline
\end{tabular}
\end{table}

To place the methods on the same input, we transform each gene-by-cell expression matrix by $x\mapsto\log_{10}(1+x)$ and then divide each cell column by its $\ell_2$ norm.
All genes present in the processed matrices are retained, and every method uses the resulting matrix $X$.
The dendritic-batch matrices are transposed from their stored cell-by-gene layout before preprocessing.
The expression and label files have matching cell-identifier order, and none of the five matrices has a zero-norm cell column after the logarithmic transform.
\Cref{Supp-tab:scrna_dataset_dimensions} records the matrix dimensions, prescribed cluster counts, latent dimensions, and numbers of fitted positive scales.
The dendritic-cell datasets use population labels, whereas GSE75748time and GSE75140 use the temporal and developmental categories described in the main text.

\begin{table}[htbp]
\centering
\caption{Final normalized joint KKT residuals after $20$ outer iterations, evaluated at the final factor pair as in \Cref{Supp-eq:joint_projected_KKT_residual}. All ten runs terminate at the outer iteration budget with convergence-based early stopping disabled. The configured residual tolerance is $10^{-6}$.}
\label{Supp-tab:scrna_final_kkt}
\begin{tabular}{ccc}
\hline
Dataset & Full pNMF & Decoupled pNMF\\
\hline
GSE75748time & $3.8042\times10^{-2}$ & $3.8253\times10^{-2}$\\
GSE94820 & $6.2714\times10^{-2}$ & $6.2631\times10^{-2}$\\
GSE75140 & $5.2698\times10^{-2}$ & $5.2472\times10^{-2}$\\
Dendritic\_batch1 & $3.5233\times10^{-2}$ & $3.5221\times10^{-2}$\\
Dendritic\_batch2 & $3.5888\times10^{-2}$ & $3.5864\times10^{-2}$\\
\hline
\end{tabular}
\end{table}

All five Full pNMF runs use $\alpha=1.5$ and the same manually selected regularization parameters $(\lambda_1,\lambda_2,\lambda_3)=(1,1,10^{-3})$, with scales constructed from the distinct positive finite $0$-PH death values of the preprocessed data.
We use double precision, NNDSVDA initialization~\cite{boutsidis2008svd} based on an exact SVD, and training seed $0$.
The projected Newton--PCG solver uses the normalized inner residual in \Cref{Supp-eq:inner_projected_residual_test}, with initial tolerance $10^{-2}$, a reduction factor of $1/2$ after each outer iteration, and a lower limit of $10^{-6}$.
For the basis updates, the residual aggregates all feature rows in a batch of at most eight scales, whereas the embedding residual covers the complete joint embedding block.
Both the $W$ and $H$ updates allow at most $50$ Newton iterations, with at most $40$ and $80$ PCG iterations per Newton step, respectively, and relative PCG tolerance $10^{-3}$ checked every two iterations and at the iteration limit.
The solver uses Armijo constant $10^{-4}$, backtracking factor $1/2$, and \texttt{max\_line\_search\_steps}$=20$, following the acceptance and budget conventions in \Cref{Supp-sec:projected_newton_details}.
All five Full runs reach the budget of $20$ outer iterations with convergence-based early stopping disabled.
Their final normalized joint KKT residuals range from $0.0352$ to $0.0627$, above the configured tolerance $10^{-6}$, so these results are fixed-budget comparisons rather than evaluations at certified joint stationarity.
\Cref{Supp-tab:scrna_final_kkt} records the residuals together with those of the paired Decoupled runs.

\begin{table}[htbp]
\centering
\caption{Baseline parameters. Entries are $(\lambda_{\mathrm{base}},k)$, where $\lambda_{\mathrm{base}}$ is the graph-regularization coefficient. For GNMF and TNMF, $k$ is the neighborhood size. For $k$-TNMF, $k=8$ is the largest neighborhood size and the number of nested graph levels. NMF has no graph parameter.}
\label{Supp-tab:baseline_parameters}
\begin{tabular}{cccc}
\hline
Dataset & GNMF & TNMF & $k$-TNMF\\
\hline
GSE75748time & $(0.1,5)$ & $(0.1,20)$ & $(1,8)$\\
GSE94820 & $(0.1,15)$ & $(0.1,1)$ & $(0.1,8)$\\
GSE75140 & $(0.1,1)$ & $(0.1,1)$ & $(0.1,8)$\\
Dendritic\_batch1 & $(0.1,15)$ & $(0.1,10)$ & $(1,8)$\\
Dendritic\_batch2 & $(1,15)$ & $(0.1,5)$ & $(0.1,8)$\\
\hline
\end{tabular}
\end{table}

The baselines use NNDSVDA initialization and the same latent dimension, with multiplicative-update parameter \texttt{max\_iter}$=1000$ and relative objective-change tolerance $10^{-3}$.
Their graph parameters and the filtration weights below were selected manually. The recorded parameter values are listed in \Cref{Supp-tab:baseline_parameters}.
For TNMF and $k$-TNMF, the supplied filtration weights are $(1,0,1,0,1,0,1,1)$ and are normalized by their sum when constructing the graph.

To account for reciprocal factor scaling, each nonzero basis column can be divided by its norm while the corresponding coefficient row is multiplied by that norm, leaving the product unchanged.
The baseline implementations apply this scaling after each pair of multiplicative updates and then evaluate the objective used in their relative-change stopping test.
The scaling preserves $WH$ but generally changes the graph-regularization term, so it is an implementation step rather than an invariance of the full regularized objective.
pNMF uses the post-optimization coordinates $Z_t$ defined in the main text for evaluation and retains its raw factors for path diagnostics.
The main-table pNMF result always uses $Z_1$, so the reported scale is fixed independently of the labels.
Lloyd's $k$-means~\cite{lloyd1982least} then uses the annotated cluster count, $10$ initialization candidates, at most $300$ iterations, and random seed $1$.
These candidates constitute one clustering evaluation per embedding and do not represent independent model-training replicates.
The original evaluations used scikit-learn $1.2.2$, NumPy $1.26.3$, and SciPy $1.13.1$.

The four agreement measures use the same contingency counts.
Let $\widehat{\mathcal{C}}=\{\widehat C_i\}_{i=1}^{n_{\mathrm{cl}}}$ and $\mathcal Y=\{Y_j\}_{j=1}^{n_{\mathrm{cl}}}$ denote the predicted and reference partitions, where $\mathcal Y$ is the annotation partition used for evaluation.
Write $n_{ij}=|\widehat C_i\cap Y_j|$, $a_i=\sum_j n_{ij}$, $b_j=\sum_i n_{ij}$, and $n=\sum_{ij}n_{ij}$.
With
\[
Q=\sum_{ij}\binom{n_{ij}}2,\qquad
A=\sum_i\binom{a_i}2,\qquad
B=\sum_j\binom{b_j}2,
\]
the Adjusted Rand Index~\cite{hubert1985comparing} is
\[
\operatorname{ARI}=\frac{Q-AB/\binom n2}{(A+B)/2-AB/\binom n2}.
\]
Normalized Mutual Information uses the arithmetic mean of the two entropies~\cite{vinh2010information}:
\[
\operatorname{NMI}(\widehat{\mathcal{C}},\mathcal Y)
=\frac{2\operatorname{MI}(\widehat{\mathcal{C}};\mathcal Y)}
{\operatorname{Ent}(\widehat{\mathcal{C}})+\operatorname{Ent}(\mathcal Y)}.
\]
The remaining measures are
\[
\operatorname{Purity}=\frac1n\sum_i\max_j n_{ij},
\qquad
\operatorname{Accuracy}
=\max_{\pi\in\mathfrak{S}_{n_{\mathrm{cl}}}}\frac1n\sum_i n_{i,\pi(i)},
\]
where $\mathfrak{S}_{n_{\mathrm{cl}}}$ is the set of label permutations.
Accuracy is computed by optimal one-to-one assignment~\cite{crouse2016implementing}, whereas Purity permits several clusters to select the same reference class.

\subsection{Visualization Settings and Interpretation}
\label{Supp-sec:embedding_diagnostics}
To inspect how the representations differ in two dimensions, we use t-SNE and UMAP projections obtained under a common protocol with random seed $42$.
For t-SNE, perplexity is $30$, initialization is PCA, the learning rate is automatic, the iteration count is $1000$, and the Barnes--Hut angle is $0.5$.
UMAP uses $15$ neighbors, minimum distance $0.1$, and Euclidean distance, with umap-learn $0.5.4$.
Both methods operate on the input or evaluation embeddings without preliminary PCA dimensionality reduction, although t-SNE uses PCA to initialize its two-dimensional coordinates.
Factorized representations undergo the common evaluation scaling before projection.

To highlight the 12hr and 24hr populations in main-text Figure~\ref{fig:scrna_embedding_visualization}, we pool their samples and form the coordinate-wise $2$nd--$98$th percentile bounding box, padded on each side by $10\%$ of the corresponding percentile span.
If the top edge reaches or crosses the panel frame, the rectangle is shifted downward until the gap is $3\%$ of the vertical axis span.
Colors and rectangles are added after projection, so they do not affect the embeddings or clustering assignments.

\subsection{Additional First-Scale Embedding Visualizations}
\label{Supp-sec:additional_h1_visualizations}

\begin{figure}[htbp]
\centering
\includegraphics[width=\textwidth]{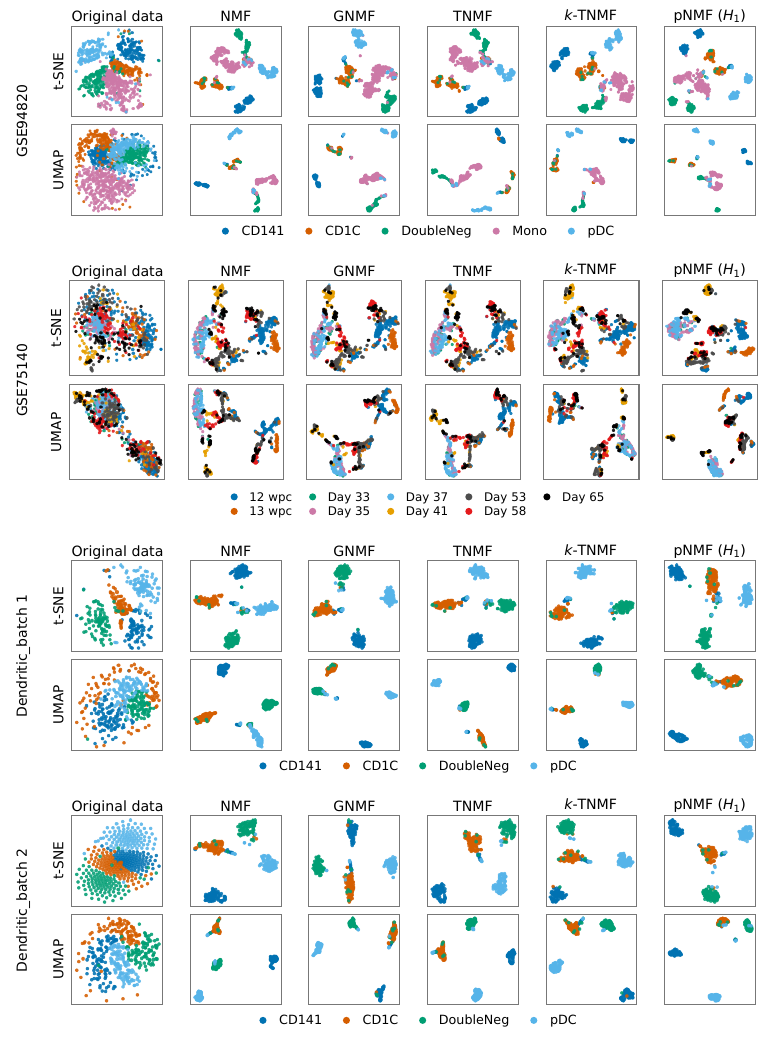}
\caption{t-SNE (upper row of each block) and UMAP (lower row) projections for the remaining datasets. Columns show the preprocessed input, NMF, GNMF, TNMF, $k$-TNMF, and pNMF using $Z_1$ (labeled $H_1$). Shared dendritic-cell types use consistent colors: CD141, CD1C, DoubleNeg, and pDC, with Mono additionally included in GSE94820. GSE75140 labels are 12 and 13 weeks post-conception (wpc) and Days 33, 35, 37, 41, 53, 58, and 65. Labels color the projections after computation.}
\label{Supp-fig:remaining_h1_visualization}
\end{figure}

To examine the representations under the other annotations, \Cref{Supp-fig:remaining_h1_visualization} extends the comparison of the input and embeddings to the remaining four datasets using the same projection protocol and evaluation scaling, with pNMF represented by $Z_1$.

For GSE94820 and the dendritic batches, separation varies among the factorized representations.
Substantial overlap remains among the GSE75140 developmental categories.

\subsection{Effect of Cross-Scale Coupling on scRNA-seq Embedding Sequences}
\label{Supp-sec:scrna_cross_scale_coupling}
To assess the effect of coupling, we compare Full pNMF, with $(\lambda_1,\lambda_2,\lambda_3)=(1,1,10^{-3})$, against Decoupled pNMF, which changes only $\lambda_2$ to zero.
On each dataset, the two variants use the same scale sequence, rank, initialization settings, and training seed.
Both variants run for $20$ outer iterations, and their final joint KKT residuals exceed the configured tolerance on all five datasets (\Cref{Supp-tab:scrna_final_kkt}).

Since the coupling penalty acts on the raw $H_t$ but clustering uses $Z_t$, we examine both the geometric evolution of the raw factors and agreement with annotations after evaluation scaling.

For the columns $h_{t,i}$ of $H_t$, define the distance matrix and its normalized form by
\begin{align*}
[\operatorname{Dist}(H_t)]_{ij} &=\|h_{t,i}-h_{t,j}\|_2, \\
\widehat{\operatorname{Dist}}(H_t) &=
\begin{cases}
\operatorname{Dist}(H_t)/\|\operatorname{Dist}(H_t)\|_F, &\|\operatorname{Dist}(H_t)\|_F>0, \\
0, &\text{otherwise}.
\end{cases}
\end{align*}
The adjacent embedding change is then $\mathrm{AEC}_t=\|\widehat{\operatorname{Dist}}(H_t)-\widehat{\operatorname{Dist}}(H_{t-1})\|_F$ for $t=2,\ldots,m$.
To assess local smoothness alongside overall path variation and sample dispersion, we use the following summaries, with $\xi=10^{-12}$ and RLD denoting relative latent dispersion:
\begin{align*}
\overline{\mathrm{AEC}}^{\mathrm{fin}}
&=\frac{1}{m-1}\sum_{t=2}^{m}\|\widehat{\operatorname{Dist}}(H_t)-\widehat{\operatorname{Dist}}(H_{t-1})\|_F, \\
\mathcal{R}_{\mathrm{emb}}^{\mathrm{fin}}
&=\frac{1}{m-2}\sum_{t=2}^{m-1}\|\widehat{\operatorname{Dist}}(H_{t+1})-2\widehat{\operatorname{Dist}}(H_t)+\widehat{\operatorname{Dist}}(H_{t-1})\|_F, \\
\mathrm{Span}_{\mathrm{emb}}^{\mathrm{fin}}
&=\|\widehat{\operatorname{Dist}}(H_m)-\widehat{\operatorname{Dist}}(H_1)\|_F, \\
\mathrm{RLD}_{\min}^{\mathrm{fin}}
&=\min_{1\leq t\leq m}
\frac{\|H_t(I_n-\mathbf{1}_n\mathbf{1}_n^\top/n)\|_F}{\|H_t\|_F+\xi}.
\end{align*}
Mean AEC measures adjacent change, and roughness measures variation in successive increments.
Endpoint span and minimum RLD complement these local quantities by checking whether the path still changes overall and whether the samples remain dispersed.
The superscript ``fin'' denotes the complete fitted sequence of positive scales, including the transition to the last scale.
Because AEC and roughness use scale-index differences without division by the threshold spacings $\Delta\varepsilon$, they describe the discrete sequence and do not estimate a continuous-threshold derivative or stability constant.
The paired variants share this sequence within each dataset, while comparisons of absolute diagnostic values across datasets also reflect differences in the data and scale sampling.

\begin{figure}[htbp]
\centering
\includegraphics[width=\textwidth]{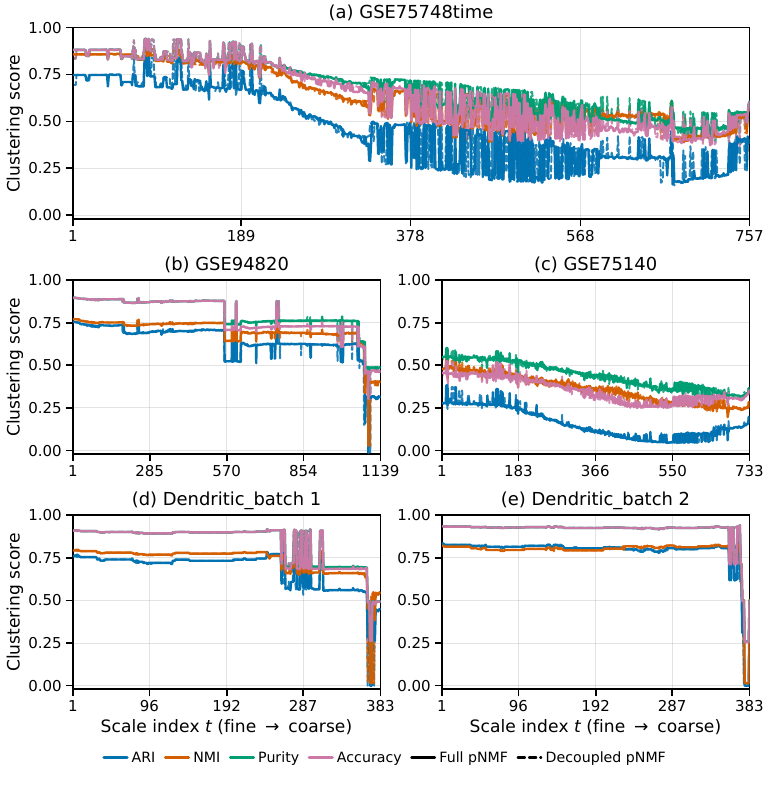}
\caption{Scale-dependent scores for Full and Decoupled pNMF. Panels (a)--(e) follow the dataset order of the main table. Colors identify the four metrics, while solid and dashed curves identify Full and Decoupled pNMF. Both variants use the same evaluation scaling, annotated cluster counts, and $k$-means settings. All saved scales are shown, and the vertical range retains negative ARI values.}
\label{Supp-fig:scrna_coupling_clustering}
\end{figure}

\Cref{Supp-fig:scrna_coupling_clustering} shows why both complete-sequence averages and individual-scale comparisons are needed.
The largest absolute mean score difference across the $20$ dataset--metric pairs is approximately $0.003381$, but this is the absolute value of a signed mean, in which positive and negative scale-wise differences can cancel.
For example, on GSE75748time, Full minus Decoupled ARI has mean $-0.003381$ but mean absolute difference $0.046415$.
Its absolute difference exceeds $0.05$ at $204$ of the $757$ scales and reaches $0.265108$ at $t=671$.
Using $Z_1$, Full attains an ARI of $0.747205$ on GSE75748time, compared with $0.693798$ for Decoupled. 
At the first positive scale, Decoupled scores slightly higher than Full on all four metrics for GSE75140 and Dendritic\_batch1.
The two variants have equal scores on GSE94820 and Dendritic\_batch2.

\begin{figure}[htbp]
\centering
\includegraphics[width=\textwidth]{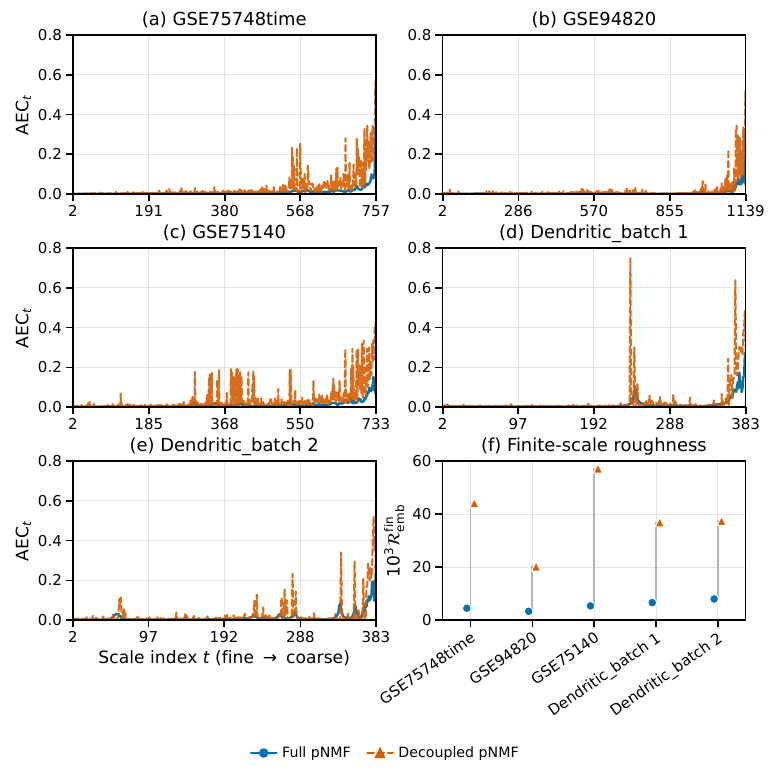}
\caption{Geometric diagnostics on the raw embedding sequences. Panels (a)--(e) show AEC over every saved adjacent transition for the five datasets in main-table order. Panel (f) compares roughness multiplied by $10^3$, with gray segments joining paired values. The AEC range $[0,0.8]$ and roughness display range $[0,60]$ include all plotted values.}
\label{Supp-fig:scrna_coupling_diagnostics}
\end{figure}

\begin{table}[ht]
\centering
\caption{Paired path summaries. AEC red. and rough. red. are percentage reductions of Full relative to Decoupled pNMF. The remaining columns give endpoint span and minimum relative latent dispersion for Full (F) and Decoupled (D), with all saved positive scales included.}
\label{Supp-tab:scrna_coupling_summary}
\begingroup\small\setlength{\tabcolsep}{3.4pt}
\begin{tabular}{ccccccc}
\hline
Dataset & AEC red. (\%) & Rough. red. (\%) & Span F & Span D & RLD F & RLD D \\
\hline
GSE75748time & 68.2 & 89.8 & 1.117 & 1.026 & 0.104 & 0.070\\
GSE94820 & 61.1 & 83.6 & 1.091 & 1.006 & 0.131 & 0.155\\
GSE75140 & 69.6 & 90.6 & 0.983 & 0.943 & 0.180 & 0.325\\
Dendritic\_batch1 & 50.9 & 82.1 & 1.045 & 0.938 & 0.127 & 0.290\\
Dendritic\_batch2 & 45.1 & 78.7 & 1.018 & 0.918 & 0.120 & 0.101\\
\hline
\end{tabular}\endgroup
\end{table}

At the recorded training budget, Full pNMF reduces mean AEC by $45.1\%$--$69.6\%$ and roughness by $78.7\%$--$90.6\%$ across the five datasets (\Cref{Supp-fig:scrna_coupling_diagnostics,Supp-tab:scrna_coupling_summary}).
In main-table order, Full pNMF has smaller AEC at $97.4\%$, $85.9\%$, $94.0\%$, $80.6\%$, and $79.1\%$ of transitions.
Its endpoint spans range from $0.983$ to $1.117$, so the normalized distance geometry still varies across scales. Its minimum RLD ranges from $0.104$ to $0.180$, indicating that the samples retain some dispersion at every saved scale.

RLD measures dispersion within the raw embeddings $H_t$, and the variant with the larger minimum RLD varies by dataset.

\section*{Code Availability}
The code for pNMF and representative experiments is available at \url{https://github.com/LiminLi-xjtu/pNMF}.

\clearpage
\endgroup

%% file: pNMF-main-arxiv.bbl
\begin{thebibliography}{10}

\bibitem{armijo1966minimization}
{\sc L.~Armijo}, {\em {Minimization of functions having Lipschitz continuous
  first partial derivatives}}, Pacific Journal of Mathematics, 16 (1966),
  pp.~1--3, \url{https://doi.org/10.2140/pjm.1966.16.1}.

\bibitem{bertsekas1982projected}
{\sc D.~P. Bertsekas}, {\em {Projected {Newton} Methods for Optimization
  Problems with Simple Constraints}}, SIAM Journal on Control and Optimization,
  20 (1982), pp.~221--246, \url{https://doi.org/10.1137/0320018}.

\bibitem{boutsidis2008svd}
{\sc C.~Boutsidis and E.~Gallopoulos}, {\em {{SVD} based initialization: A head
  start for nonnegative matrix factorization}}, Pattern Recognition, 41 (2008),
  pp.~1350--1362, \url{https://doi.org/10.1016/j.patcog.2007.09.010}.

\bibitem{cai2010graph}
{\sc D.~Cai, X.~He, J.~Han, and T.~S. Huang}, {\em {Graph Regularized
  Nonnegative Matrix Factorization for Data Representation}}, IEEE Transactions
  on Pattern Analysis and Machine Intelligence, 33 (2011), pp.~1548--1560,
  \url{https://doi.org/10.1109/TPAMI.2010.231}.

\bibitem{camp2015human}
{\sc J.~G. Camp, F.~Badsha, M.~Florio, et~al.}, {\em {Human cerebral organoids
  recapitulate gene expression programs of fetal neocortex development}},
  Proceedings of the National Academy of Sciences, 112 (2015),
  pp.~15672--15677, \url{https://doi.org/10.1073/pnas.1520760112}.

\bibitem{chu2016single}
{\sc L.-F. Chu, N.~Leng, J.~Zhang, et~al.}, {\em {Single-cell {RNA}-seq reveals
  novel regulators of human embryonic stem cell differentiation to definitive
  endoderm}}, Genome Biology, 17 (2016), p.~173,
  \url{https://doi.org/10.1186/s13059-016-1033-x}.

\bibitem{chung1997spectral}
{\sc F.~R.~K. Chung}, {\em {Spectral graph theory}}, vol.~92, American
  Mathematical Society, 1997, \url{https://doi.org/10.1090/cbms/092}.

\bibitem{crawley2015decomposition}
{\sc W.~Crawley-Boevey}, {\em {Decomposition of pointwise finite-dimensional
  persistence modules}}, Journal of Algebra and Its Applications, 14 (2015),
  p.~1550066, \url{https://doi.org/10.1142/S0219498815500668}.

\bibitem{crouse2016implementing}
{\sc D.~F. Crouse}, {\em {On implementing 2D rectangular assignment
  algorithms}}, IEEE Transactions on Aerospace and Electronic Systems, 52
  (2016), pp.~1679--1696, \url{https://doi.org/10.1109/TAES.2016.140952}.

\bibitem{demko1984decay}
{\sc S.~Demko, W.~F. Moss, and P.~W. Smith}, {\em {Decay rates for inverses of
  band matrices}}, Mathematics of Computation, 43 (1984), pp.~491--499,
  \url{https://doi.org/10.2307/2008290}.

\bibitem{edelsbrunner2008persistent}
{\sc H.~Edelsbrunner and J.~Harer}, {\em {Persistent Homology---A Survey}}, in
  Surveys on Discrete and Computational Geometry: Twenty Years Later, vol.~453
  of Contemporary Mathematics, American Mathematical Society, 2008,
  pp.~257--282, \url{https://doi.org/10.1090/conm/453/08802}.

\bibitem{edelsbrunner2010computational}
{\sc H.~Edelsbrunner and J.~Harer}, {\em {Computational Topology: An
  Introduction}}, vol.~69, American Mathematical Society, 2010,
  \url{https://doi.org/10.1090/mbk/069}.

\bibitem{gillis2020nonnegative}
{\sc N.~Gillis}, {\em {Nonnegative Matrix Factorization}}, Society for
  Industrial and Applied Mathematics, 2020,
  \url{https://doi.org/10.1137/1.9781611976410}.

\bibitem{golub2013matrix}
{\sc G.~H. Golub and C.~F. Van~Loan}, {\em Matrix Computations}, Johns Hopkins
  University Press, 4th~ed., 2013,
  \url{https://doi.org/10.56021/9781421407944}.

\bibitem{hestenes1952methods}
{\sc M.~R. Hestenes and E.~Stiefel}, {\em {Methods of conjugate gradients for
  solving linear systems}}, Journal of Research of the National Bureau of
  Standards, 49 (1952), pp.~409--436,
  \url{https://doi.org/10.6028/jres.049.044}.

\bibitem{horn1985matrix}
{\sc R.~A. Horn and C.~R. Johnson}, {\em {Matrix Analysis}}, Cambridge
  University Press, 1985, \url{https://doi.org/10.1017/CBO9780511810817}.

\bibitem{hozumi2024analyzing}
{\sc Y.~Hozumi and G.-W. Wei}, {\em {Analyzing single cell RNA sequencing with
  topological nonnegative matrix factorization}}, Journal of Computational and
  Applied Mathematics, 445 (2024), p.~115842,
  \url{https://doi.org/10.1016/j.cam.2024.115842}.

\bibitem{hubert1985comparing}
{\sc L.~Hubert and P.~Arabie}, {\em {Comparing partitions}}, Journal of
  Classification, 2 (1985), pp.~193--218,
  \url{https://doi.org/10.1007/BF01908075}.

\bibitem{kim2008nonnegative}
{\sc H.~Kim and H.~Park}, {\em {Nonnegative Matrix Factorization Based on
  Alternating Nonnegativity Constrained Least Squares and Active Set Method}},
  SIAM Journal on Matrix Analysis and Applications, 30 (2008), pp.~713--730,
  \url{https://doi.org/10.1137/07069239X}.

\bibitem{kim2011fast}
{\sc J.~Kim and H.~Park}, {\em {Fast Nonnegative Matrix Factorization: An
  Active-Set-Like Method and Comparisons}}, SIAM Journal on Scientific
  Computing, 33 (2011), pp.~3261--3281,
  \url{https://doi.org/10.1137/110821172}.

\bibitem{lee1999learning}
{\sc D.~D. Lee and H.~S. Seung}, {\em {Learning the parts of objects by
  non-negative matrix factorization}}, Nature, 401 (1999), pp.~788--791,
  \url{https://doi.org/10.1038/44565}.

\bibitem{lee2000algorithms}
{\sc D.~D. Lee and H.~S. Seung}, {\em {Algorithms for non-negative matrix
  factorization}}, in Proceedings of the 14th International Conference on
  Neural Information Processing Systems, NIPS'00, MIT Press, 2000,
  pp.~535--541, \url{https://dl.acm.org/doi/10.5555/3008751.3008829}.

\bibitem{lloyd1982least}
{\sc S.~Lloyd}, {\em {Least squares quantization in PCM}}, IEEE Transactions on
  Information Theory, 28 (1982), pp.~129--137,
  \url{https://doi.org/10.1109/TIT.1982.1056489}.

\bibitem{mcinnes2020umap}
{\sc L.~McInnes, J.~Healy, and J.~Melville}, {\em {{UMAP}: Uniform Manifold
  Approximation and Projection for Dimension Reduction}}, 2020,
  \url{https://arxiv.org/abs/1802.03426}.

\bibitem{memoli2022persistent}
{\sc F.~M{\'e}moli, Z.~Wan, and Y.~Wang}, {\em {Persistent {Laplacians}:
  Properties, Algorithms and Implications}}, SIAM Journal on Mathematics of
  Data Science, 4 (2022), pp.~858--884,
  \url{https://doi.org/10.1137/21M1435471}.

\bibitem{saad2003iterative}
{\sc Y.~Saad}, {\em {Iterative Methods for Sparse Linear Systems}}, Society for
  Industrial and Applied Mathematics, 2nd~ed., 2003,
  \url{https://doi.org/10.1137/1.9780898718003}.

\bibitem{van2008visualizing}
{\sc L.~van~der Maaten and G.~Hinton}, {\em {Visualizing Data using t-{SNE}}},
  Journal of Machine Learning Research, 9 (2008), pp.~2579--2605,
  \url{http://jmlr.org/papers/v9/vandermaaten08a.html}.

\bibitem{villani2017single}
{\sc A.-C. Villani, R.~Satija, G.~Reynolds, et~al.}, {\em {Single-cell
  {RNA}-seq reveals new types of human blood dendritic cells, monocytes, and
  progenitors}}, Science, 356 (2017), p.~eaah4573,
  \url{https://doi.org/10.1126/science.aah4573}.

\bibitem{vinh2010information}
{\sc N.~X. Vinh, J.~Epps, and J.~Bailey}, {\em {Information Theoretic Measures
  for Clusterings Comparison: Variants, Properties, Normalization and
  Correction for Chance}}, Journal of Machine Learning Research, 11 (2010),
  pp.~2837--2854, \url{http://jmlr.org/papers/v11/vinh10a.html}.

\bibitem{wang2013multiple}
{\sc J.~J.-Y. Wang, H.~Bensmail, and X.~Gao}, {\em {Multiple graph regularized
  nonnegative matrix factorization}}, Pattern Recognition, 46 (2013),
  pp.~2840--2847, \url{https://doi.org/10.1016/j.patcog.2013.03.007}.

\bibitem{wang2020persistent}
{\sc R.~Wang, D.~D. Nguyen, and G.-W. Wei}, {\em {Persistent spectral graph}},
  International Journal for Numerical Methods in Biomedical Engineering, 36
  (2020), p.~e3376, \url{https://doi.org/10.1002/cnm.3376}.

\bibitem{wang2012nonnegative}
{\sc Y.-X. Wang and Y.-J. Zhang}, {\em {Nonnegative Matrix Factorization: A
  Comprehensive Review}}, IEEE Transactions on Knowledge and Data Engineering,
  25 (2013), pp.~1336--1353, \url{https://doi.org/10.1109/TKDE.2012.51}.

\bibitem{zomorodian2010fast}
{\sc A.~Zomorodian}, {\em {Fast construction of the {Vietoris-Rips} complex}},
  Computers \& Graphics, 34 (2010), pp.~263--271,
  \url{https://doi.org/10.1016/j.cag.2010.03.007}.

\bibitem{zomorodian2005computing}
{\sc A.~Zomorodian and G.~Carlsson}, {\em {Computing Persistent Homology}},
  Discrete \& Computational Geometry, 33 (2005), pp.~249--274,
  \url{https://doi.org/10.1007/s00454-004-1146-y}.

\end{thebibliography}
